\documentclass{article}

\usepackage[utf8]{inputenc}

\usepackage{graphicx}
\usepackage{amsmath}
\usepackage{amssymb}
\usepackage{mathtools}
\usepackage{siunitx}
\usepackage{longtable,tabularx}
\usepackage{bm}

\usepackage{mathrsfs}
\usepackage{amsfonts}
\usepackage{amsthm} 
\usepackage{verbatim}
\usepackage{xcolor}
\usepackage{mdwlist}

\usepackage{lineno,hyperref}
\modulolinenumbers[5]

\usepackage{bbm}
\usepackage{xy} 
\xyoption{all}

\title{Lunar Crater Identification in Digital Images}

\author{
  John A. Christian\thanks{Department of Mechanical, Aerospace, and Nuclear Engineering, Rensselaer Polytechnic Institute, Troy, NY 12180, USA, \texttt{chrisj9@rpi.edu}}, Harm Derksen\thanks{Department of Mathematics, Northeastern University, Boston, MA  02115, USA, \texttt{ha.derksen@northeastern.edu}}, and Ryan Watkins\thanks{Planetary Science Institute, Tuscon, AZ 85719, USA, \texttt{rclegg-watkins@psi.edu}}
}

\begin{document}
\maketitle

\begin{abstract}
It is often necessary to identify a pattern of observed craters in a single image of the lunar surface and without any prior knowledge of the camera's location. This so-called ``lost-in-space'' crater identification problem is common in both crater-based terrain relative navigation (TRN) and in automatic registration of scientific imagery. Past work on crater identification has largely been based on heuristic schemes, with poor performance outside of a narrowly defined operating regime (e.g., nadir pointing images, small search  areas). This work provides the first mathematically rigorous treatment of the general crater identification problem. It is shown when it is (and when it is not) possible to recognize a pattern of elliptical crater rims in an image formed by perspective projection. For the cases when it is possible to recognize a pattern, descriptors are developed using invariant theory that provably capture all of the viewpoint invariant information. These descriptors may be pre-computed for known crater patterns and placed in a searchable index for fast recognition. New techniques are also developed for computing pose from crater rim observations and for evaluating crater rim correspondences. These techniques are demonstrated on both synthetic and real images.
\end{abstract}


\newcommand{\bx}{\textbf{\emph{x}}}
\newcommand{\bPhi}{\boldsymbol{\Phi}}
\newcommand{\bGamma}{\boldsymbol{\Gamma}}
\newcommand{\bXi}{\boldsymbol{\Xi}}
\newcommand{\bpsi}{\boldsymbol{\psi}}
\newcommand{\beps}{\boldsymbol{\epsilon}}
\newcommand{\bpi}{\boldsymbol{\pi}}
\newcommand{\ba}{\textbf{\emph{a}}}
\newcommand{\bb}{\textbf{\emph{b}}}
\newcommand{\bc}{\textbf{\emph{c}}}
\newcommand{\bg}{\textbf{\emph{g}}}
\newcommand{\bl}{\boldsymbol{\ell}}
\newcommand{\bX}{\textbf{\emph{X}}}
\newcommand{\be}{\textbf{\emph{e}}}
\newcommand{\bt}{\textbf{\emph{t}}}
\newcommand{\bn}{\textbf{\emph{n}}}
\newcommand{\boldf}{\textbf{\emph{f}}}
\newcommand{\bw}{\textbf{\emph{w}}}
\newcommand{\bh}{\textbf{\emph{h}}}
\newcommand{\br}{\textbf{\emph{r}}}
\newcommand{\bs}{\textbf{\emph{s}}}
\newcommand{\bbm}{\textbf{\emph{m}}}
\newcommand{\bu}{\textbf{\emph{u}}}
\newcommand{\bv}{\textbf{\emph{v}}}
\newcommand{\by}{\textbf{\emph{y}}}
\newcommand{\bz}{\textbf{\emph{z}}}
\newcommand{\bp}{\textbf{\emph{p}}}
\newcommand{\bq}{\textbf{\emph{q}}}
\newcommand{\bA}{\textbf{\emph{A}}}
\newcommand{\bC}{\textbf{\emph{C}}}
\newcommand{\bD}{\textbf{\emph{D}}}
\newcommand{\bF}{\textbf{\emph{F}}}
\newcommand{\bG}{\textbf{\emph{G}}}
\newcommand{\bI}{\textbf{\emph{I}}}
\newcommand{\bP}{\textbf{\emph{P}}}
\newcommand{\bJ}{\textbf{\emph{J}}}
\newcommand{\bK}{\textbf{\emph{K}}}
\newcommand{\bH}{\textbf{\emph{H}}}
\newcommand{\bR}{\textbf{\emph{R}}}
\newcommand{\bQ}{\textbf{\emph{Q}}}
\newcommand{\bbN}{\textbf{\emph{N}}}
\newcommand{\bk}{\textbf{\emph{k}}}
\newcommand{\bB}{\textbf{\emph{B}}}
\newcommand{\bT}{\textbf{\emph{T}}}
\newcommand{\bL}{\textbf{\emph{L}}}
\newcommand{\bE}{\textbf{\emph{E}}}
\newcommand{\bM}{\textbf{\emph{M}}}
\newcommand{\bS}{\textbf{\emph{S}}}
\newcommand{\bU}{\textbf{\emph{U}}}
\newcommand{\bV}{\textbf{\emph{V}}}
\newcommand{\bY}{\textbf{\emph{Y}}}
\newcommand{\etal}{\textit{et al.}}
\newcommand{\bxi}{\boldsymbol{\xi}}
\newcommand{\bnu}{\boldsymbol{\nu}}
\newcommand{\bbeta}{\boldsymbol{\beta}}
\newcommand{\btheta}{\boldsymbol{\theta}}
\newcommand{\bomega}{\boldsymbol{\omega}}

\newcommand{\R}{{\mathbb R}}
\newcommand{\C}{{\mathbb C}}
\newcommand{\Z}{{\mathbb Z}}
\newcommand{\PP}{{\mathbb P}}
\newcommand{\trdeg}{\operatorname{trdeg}}
\newcommand{\eproof}{\ \hfill$\square$\\}
\newcommand{\bproof}{\noindent{\it Proof. }}
\newcommand{\varV}{{\mathscr V}}
\newcommand{\varS}{{\mathscr S}}
\newcommand{\varY}{{\mathscr Y}}
\newcommand{\varG}{{\mathscr G}}
\newcommand{\varU}{{\mathscr U}}
\newcommand{\varZ}{{\mathscr Z}}
\newcommand{\bbc}{\bar{\bc}}
\newcommand{\bbr}{\bar{\br}}
\newcommand{\bo}{\textbf{\emph{o}}}
\newcommand{\bbo}{\bar{\bo}}
\newcommand{\bba}{\bar{\ba}}
\newcommand{\bbb}{\bar{\bb}} 
\newcommand{\bbp}{\bar{\bp}}
\newcommand{\bbq}{\bar{\bq}}
\newcommand{\bbu}{\bar{\bu}} 
\newcommand{\varH}{{\mathscr H}}

\newtheorem{theorem}{Theorem}
\newtheorem{remark}[theorem]{Remark}
\newtheorem{lemma}[theorem]{Lemma}

\section{Introduction}
Future lunar exploration missions are expected to rely on optical measurements (e.g., images from a camera) to navigate independently of Earth-based operators. Although the use of images for spacecraft navigation---called optical navigation (OPNAV)---is a well-established practice for conventionally navigated spacecraft \cite{Owen:2008,Owen:2011}, autonomous on-board OPNAV remains an emerging technology. Recent decades have  witnessed great technological advancement and expanding acceptance of autonomous vision-based navigation for the exploration of other celestial bodies (e.g., Moon, Mars, asteroids). These advancements in real-time, on-board OPNAV are exemplified by the technological progression from simple estimation of lander velocity with the Mars Exploration Rover's DIMES system in 2004 \cite{Cheng:2004} to autonomous feature tracking that will soon be demonstrated on the OSIRIS-REx mission to asteroid Bennu \cite{Olds:2015} and during landing of the Mars Perseverance (Mars 2020) rover \cite{Johnson:2020}.

This work focuses on OPNAV for lunar missions using digital images captured by a conventional camera (images form a pushbroom camera \cite{Gupta:1997} are a different problem and are not discussed here). Images from a conventional camera follow the geometry of perspective projection. How to best use such images of the Moon for navigation depends on many factors---with distance and lighting often being the two most important considerations. It is usually best to use horizon-based OPNAV \cite{Owen:2011,Christian:2017} when far away from the Moon and when a large portion of the lunar disk is contained within the image. As the vehicle gets closer to the Moon (e.g., low lunar orbit, lunar lander descent), it becomes more appropriate to observe specific landmarks on the lunar surface. These landmarks could be specific types of morphological features (e.g., craters) or a patch of unique-looking terrain.

The principal difficulty with landmark-based OPNAV is the need to match points in an image to points in an onboard map. This has led to the continued use of horizon-based OPNAV at close distances where it would otherwise be better to use surface features. It has also led to navigation with unknown landmarks \cite{Levine:1966,Toda:1966,Bellantoni:1967,Christian:2019,Christian:2020}. Unfortunately, horizon-based OPNAV is not always practical (especially at close ranges) and unknown landmarks do not generally provide full state observability. Consequently, the ability to match surface features to a map is a required capability for autonomous vision-based navigation for  many missions.

Many landmark-based OPNAV algorithms rely on the real-time rendering of an onboard digital elevation map (DEM) \cite{Olds:2015,Adams:2008,Gaskell:2008,Johnson:2020}. The standard approach is to render the expected appearance of landmark patches and then compare this to regions of the navigation image, usually by means of a 2D cross-correlation. This, however, requires the vehicle to carry onboard 3D models for each landmark patch (and, sometimes, complete DEMs for certain portions of the body) and also have the ability to render these patches in real-time. Further, the rendering step requires  \emph{a priori} state knowledge, thus precluding this technique from supporting  lost-in-space landmark-based OPNAV. More troubling, however, is that the  rendering-based template matching approach feeds navigation  state data into the measurement generation process, creating scenarios where corrupted navigation states render otherwise good images ineffectual---potentially leading to unnecessary filter reinitializations, trajectory aborts (e.g., during lunar descent), or other undesirable events. 

Landmark-based OPNAV using craters \cite{Cheng:2005,Hanak:2010,Park:2019,Maass:2020} is an alternative strategy that, when properly implemented, does not necessarily require \emph{a priori} state data. Furthermore, OPNAV with craters does not require onboard DEMs or an onboard rendering capability. Instead, such systems require an image processing algorithm to detect craters and a pattern matching algorithm to match these observations to a catalog of known craters. This manuscript focuses on the second part (crater pattern matching) portion of this problem.

There are two general classes of crater matching algorithms: (1) tracking and (2) lost-in-space. In the case of crater tracking, it is possible to use \emph{a priori} state information to predict where catalog craters should appear in an image. These predictions are used to associate observed craters with model craters through a variety of schemes of varying sophistication. We generally find that explicit crater matching is undesirable and more robust tracking performance is achievable using probabilistic techniques (as exemplified by anonymous feature tracking (AFT) \cite{McCabe:2020}). In the case of lost-in-space crater matching, no \emph{a priori} state information is available. Lost-in-space algorithms, therefore, must be able to recognize a crater pattern from its appearance in an image regardless of camera pose (i.e., camera position and attitude). Such a capability is necessary for filter (re)initialization. It is also useful in providing state-independent measurements (if required) or a state-independent verification of crater tracking correspondences.  This work focuses in solving the lost-in-space crater identification problem.

We briefly note that, while autonomous spacecraft navigation was the original motivation for this work, crater pattern recognition is equally useful for the registration of scientific images \cite{Troglio:2012}.

In this manuscript, we provide a comprehensive treatment of lunar crater pattern recognition in digital images formed by a conventional camera (i.e., under perspective projection). We begin by outlining the  philosophical framework for such a system (Section~\ref{Sec:CraterIDArch}) and then review important background material (Section~\ref{Sec:Background}). This is followed by the detailed mathematical development in Sections \ref{Sec:SingleCraterGeom}--\ref{Sec:IndexMatchingTop}. Numerical results and examples are shown in Section~\ref{Sec:NumResults}. This work contains a number of key results, which are summarized here:
\begin{enumerate*}
\item Almost all lunar impact craters are elliptical, and many are nearly circular (Section~\ref{Sec:PlanetaryScience}). Rather than being a heuristic design choice for easy crater identification, we discuss why craters are necessarily elliptical in shape from a perspective of planetary science and impact mechanics. Recent advancements in our knowledge of the lunar crater population suggests that a circular crater assumption (common amongst crater identification algorithms) is not well-supported by the data for small craters (diameter less than about 30 km).

\item  Invariant theory is the proper mathematical framework for recognizing a pattern of crater rims in an image (Section~\ref{Sec:ExistenceOfInvariantsTopLevel}). Elliptical craters on  the lunar surface project to elliptical features in an image (Section~\ref{Sec:SingleCraterGeom}).  For some  patterns of craters, there exist algebraic quantities that remain  unchanged (i.e., invariant) regardless of the camera pose and that may be computed from just the projected crater rims in a single image. Such quantities we call \emph{projective invariants}. Invariant theory provides the means by which we may determine the existence (or not) of these projective invariants.

\item There are no projective invariants for arbitrarily placed conics (e.g., elliptical crater rims) on the surface of an arbitrarily shaped body. We provide the first known proof of this fact (Section~\ref{Sec:InvariantsProjArbitraryConicsP3}). This result is significant, as it produces a theoretical obstacle for solving the lost-in-space crater identification problem about very irregular bodies, such as some asteroids and comets. Fortunately, the Moon is not an arbitrarily shaped body.

\item Projective invariants exist for arbitrarily placed conics (e.g., elliptical crater rims) lying on a nondegenerate quadric surface (e.g., sphere, ellipsoid), which is an excellent approximation for the shape of the Moon. A different set of invariants exists for conics lying on a common plane. The existence of invariants for conics on a nondegenerate quadric surface is a novel result (Section~\ref{Sec:InvariantsProjQuadSurfConicsP3}), while invariants for conics on a plane is a known result (Section~\ref{Sec:InvariantsConicsP2}). Since the Moon is nearly spherical on the regional/global level and nearly planar on the local level, projective invariants generally exist for craters lying on the surface of the Moon. In both cases (global and local), we show how to efficiently compute these invariants from only the projected conics (i.e., the contour of the crater rim) that appear in a digital image (Section~\ref{Sec:ComputingInvariantsTop}).

\item There are no algebraically independent invariants for conics (e.g., elliptical crater rims) on a nondegenerate quadric surface or on a plane beyond the ones developed in this work. All other invariants a researcher may conceive for one of these two cases may be a written as an algebraic function of the invariants discussed herein. 
Thus, our invariants describe all the independent information about a crater rim pattern that is pose invariant and useful in constructing a descriptor that may be indexed. We provide the first known proof of this fact for the case of conics on a nondegenerate quadric surface (Section~\ref{Sec:NonCoplanarTriadInvariants}). Given these findings, we suggest future investigation of invariants for patterns of lunar crater rims be focused on ease of computation or numerical stability rather than attempting to extract additional independent information from the crater rim contours (since there is no additional independent information to be had that remains unchanged with camera viewpoint).

\item The projective invariants for a triad of lunar craters may be used to form a feature descriptor that is insensitive to camera viewpoint (Section~\ref{Sec:PatternDescriptor}). These descriptors may be computed for known crater patterns and stored in a searchable catalog (i.e., an index). Matches to an observed crater pattern in any image are found by simply finding the nearest neighbor for its descriptor in the index. In cases where it is desirable for the matching to be permutation sensitive, the descriptor is nothing more than the projective invariants concatenated into a vector. Alternatively, in cases where we wish the matching to be insensitive to crater ordering, we apply invariant theory a second time to construct \emph{projective and permutation} ($p^2$) \emph{invariants}. This is a novel result for crater pattern descriptors.

\item Lunar crater identification is a multi-scale pattern recognition problem. Local, regional, and global crater patterns are built from individual craters of different (increasing) diameters and occur over different (increasing) extent on the lunar surface. For example, a 1 km crater may contribute to a local crater pattern but not a global crater pattern, while a 100 km crater may contribute to  a global crater pattern but not a local crater pattern. To address this challenge, we present the first known use of Hierarchical Equal Area isoLatitude Pixelization (HEALPix) to construct a hierarchy of crater catalogs (Section~\ref{Sec:BuildCrateIndex}). 

\item A substantial amount of navigation information is contained within the shapes of projected crater rims in an image. Many existing crater identification algorithms, however, compute camera pose using only the center coordinates of the craters---thus ignoring a great deal of actionable information. We discuss how to compute the camera position using the contours of projected image conics, rather than just the conic center coordinates. Our solution is non-iterative and finds the camera position in the least squares sense when observing $d\geq2$ craters (Section~\ref{Sec:ConicPose}).

\item Crater match hypotheses are verified by comparing the observed crater rims to what is expected were the hypothesized match true. Most existing crater identification algorithms perform this verification by comparing only the coordinates of crater centers, which generally requires $\geq5$ craters locations to agree. Instead, we develop a novel distance metric as a measure of how dissimilar two ellipse contours are from one another (Section~\ref{Sec:EllipseComp}) and use this to compare the observed and expected crater rims. Our metrics satisfy the three classical  axioms for a metric (minimality, symmetry, triangle inequality) along with a fourth requirement of similarity invariance. Using the entire crater rim permits pattern verification with just the three craters used to perform the index match and does not require additional craters for verification. Thus, as compared to past methods, the new distance metric permits crater identification over more sparsely cratered regions of the Moon.

\item There is no single crater identification algorithm that is always best. The correct solution depends on orbital regime (or descent trajectory), camera specifications, on-board computational and memory resources, operational cadence, and other mission-specific parameters. There are critical crater identification design decisions to be made for the crater catalog (raw source of data), index scale (local, regional, global, or a combination), index organization, and invariant descriptor type (projective invariants or $p^2$ invariants). Therefore, rather than present a single crater identification algorithm, we suggest a \emph{framework} solving this problem and provide the reader with tools for each step. Once understood, the elements in this manuscript may be used to identify craters for a wide diversity of mission types. A few illustrative examples are provided to enhance understanding of how this might work.

\end{enumerate*}

Finally, we note that this work relies on mathematics and techniques that may be obscure to the average astrodynamicist and spacecraft navigator.  Therefore, we adopt a rather explanatory approach to facilitate understanding, dispel common misconceptions, and encourage adoption of the techniques herein.

\section{A Framework for Lost-in-Space Terrain Relative Navigation (TRN)}
\label{Sec:CraterIDArch}

The principal difficulty with recognizing 3D objects from their projection in a digital image is that the appearance of these objects changes with camera viewpoint and illumination geometry. The recognition of landmarks on the lunar surface inherits this generic computer vision difficulty. In many cases (including lunar landmark identification), we may circumvent this problem entirely by choosing to represent objects with a descriptor that remains the same (is \emph{invariant}) regardless of camera viewpoint. Such an object recognition philosophy was influentially advocated for in \cite{Mundy:1992} and \cite{Zisserman:1995}, and we adopt this philosophy here.

There are a variety of object attributes that may remain unchanged with camera viewpoint and illumination geometry, including color, texture, and geometric structure. Finding these invariant attributes, however, requires a great deal of care. Of note, we know that no such invariant exists for an arbitrary patch of terrain with constant albedo and with an approximately Lambertian reflectance \cite{Chen:2000}. This suggests that terrain patches (or maplets)---as is often used in correlation-based TRN \cite{Olds:2015,Adams:2008,Gaskell:2008,Johnson:2020}---may not be well-suited for lost-in-space TRN, though this is an important topic of future work. We also observe that popular image feature descriptors---such as SIFT \cite{Lowe:2004}, SURF \cite{Bay:2008}, BRIEF \cite{Leutenegger:2011}, ORB \cite{Rublee:2011}, and others---do not describe the type of invariance required for lost-in-space TRN. They are, by construction, not the appropriate tool for the problem considered in this manuscript. There are two primary reasons for this. First, most of these descriptors are only formally invariant for a similarity transformation (translation, rotation, scaling), though many have been shown to be robust for a modest amount of affine shear. While there do exist feature descriptors that are formally affine invariant \cite{Morel:2009}, it was convincingly argued by Lowe \cite{Lowe:2004} that such invariance is often undesirable within the context of feature descriptors. Regardless, none of these descriptors are especially robust to extreme projective transformations---making them more appropriate for matching features between two similar images than for matching a single arbitrary image to a prebuilt database of features on a 3D object. The second (and more significant) problem with classical feature descriptors for lost-in-space TRN is that none of them are illumination invariant, as discussed at length in \cite{Christian:2020}. In brief, all of these feature descriptors work by finding and describing regions of unique 2D patterns in an image. For TRN above an airless body (e.g., Moon, asteroid), the 2D pattern observed in an image is formed by the way that light is reflected off the terrain and towards the camera. When the lighting geometry changes, the 2D pattern in the image changes---thus the feature descriptor for the same surface patch will change. Hence, the feature descriptors are not illumination invariant.

An alternative to landmark patches and feature descriptors is to look for explicit geometric features. Crater rims are one such geometric feature that can be easily recognized and consistently localized, even when observed from different camera viewpoints and under different lighting conditions. We find the geometry describing the pattern of multiple crater rims to be one of the most accessible attributes from which viewpoint invariant information may be constructed.

For any given object geometry, such as a pattern of crater rims, there may be a variety of viewpoint-invariant properties---but not all of these necessarily provide equal perceptual salience \cite{Jacobs:2003}. Thus, at a minimum, we seek non-constant (i.e., non-trivial) invariants that can effectively differentiate between many different crater rim patterns. Such non-trivial invariants do not always exist for craters lying on the surface of arbitrary 3D celestial bodies, though they do exist for a body such as the Moon. Even when invariants do exist, finding a way to compute them is not always straightforward. However, as we show in this manuscript, the additional effort to develop these invariants produces powerful pattern recognition capabilities.

The identification of perceptually salient invariants under the action of perspective projection (e.g., a digital image of the Moon collected with a conventional camera) allows us to quickly distinguish between two different objects (e.g., two different crater patterns) possessing two different invariant descriptors. Therefore, if we compute the invariants for a particular lunar crater pattern, these invariants may be concatenated into a descriptor for that pattern and stored in a catalog (or \emph{index}) for future comparison. We repeat this procedure for each known crater pattern pre-flight and produce a very large index of patterns and their corresponding descriptors. Then, when an image is acquired in-flight, the invariants may be computed for an observed crater pattern and concatenated into the same type of pattern descriptor that we stored within the index.  Therefore, the index may be queried for known patterns with descriptor values similar to that of the observed crater pattern. Efficient data structures allow such queries to be performed very fast, usually in $\mathcal{O}(\log n)$ time.

Matches from the index may be used to construct crater match hypotheses, which must be verified before acceptance. The primary means for verification is to use a hypothesized crater correspondence to compute the camera's location. This hypothesized camera location is used to reproject expected crater rims into the image, which are compared to the observed crater rims. If the reprojected pattern matches the observations, the match hypothesis is accepted. If not, the match hypothesis is rejected, and we attempt a different hypothesis.

This framework is summarized in Fig.~\ref{fig:FrameworkFlowchart}. The remainder of this manuscript is dedicated to providing the details for each step of this process.

\begin{figure}[b!]
\centering
\includegraphics[width=0.8\columnwidth,trim=0in 0in 0in 0in,clip]{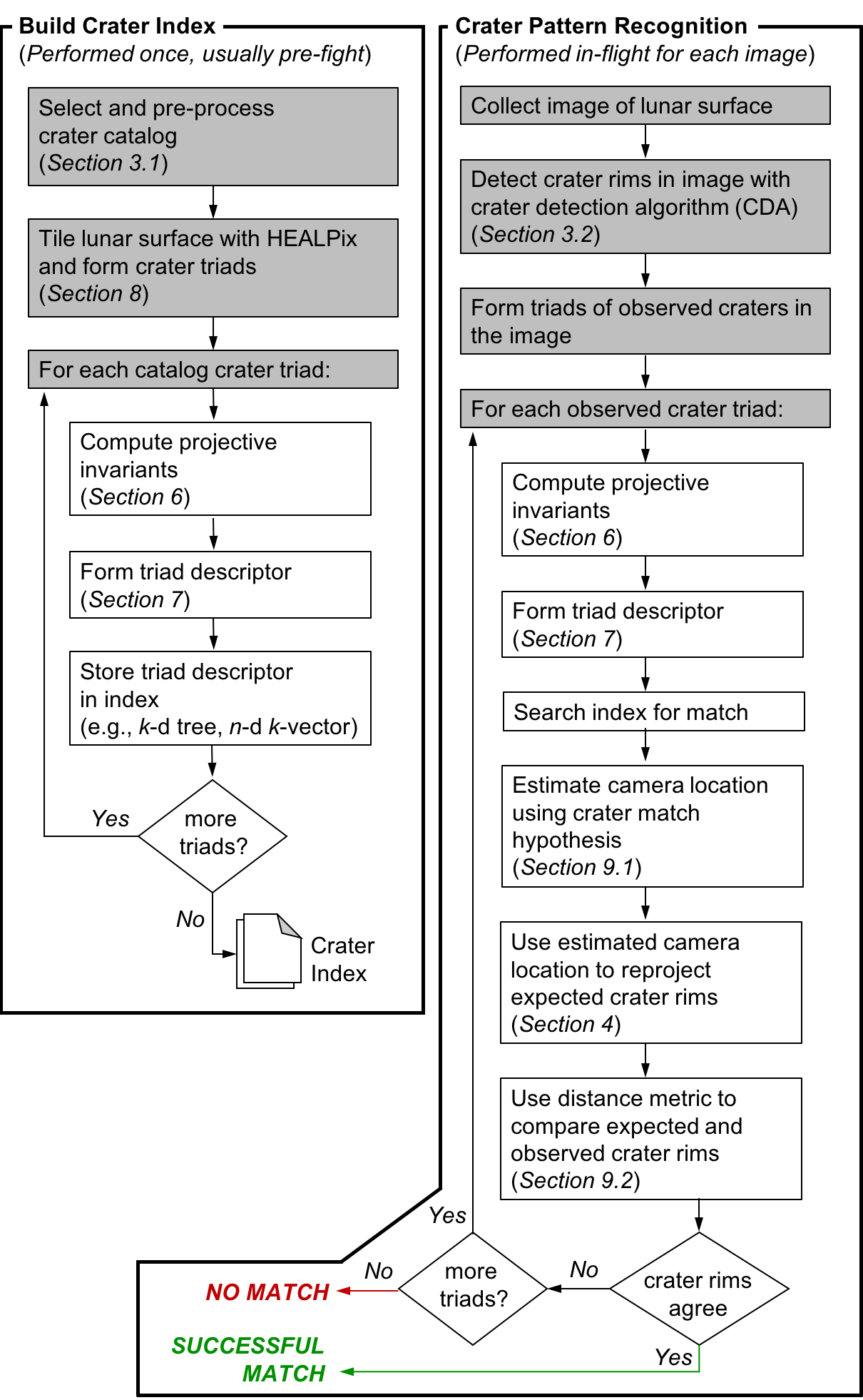}
	\caption{The crater identification framework has two major components: index construction and pattern recognition. The index construction step (left) is often time-intensive since we must loop through all possible triads and then store them in an efficient data structure. The crater patter recognition step (right) is fast and is performed for each image.}
	\label{fig:FrameworkFlowchart}
\end{figure}

\section{Background}
\label{Sec:Background}
\subsection{Observations on Lunar Crater Morphology and Size-Frequency Distribution}
\label{Sec:PlanetaryScience}
Lunar crater identification algorithms often begin with the assumption that the crater rims are either circular or elliptical in shape. It is also generally assumed that craters are well distributed about the Moon and that the crater catalog is static (new craters are not being added or removed). We find, however,  that many recent crater identification algorithms are built on assumptions not supported by a modern understanding of the lunar crater population---and, therefore, are not likely to perform well across the full diversity of orbital regimes. A thorough understanding of lunar crater morphology and size-frequency distributions must precede the development of a computational algorithm recognize patterns of these features.

\subsubsection{Crater Catalogs}

To recognize crater patterns in images of the lunar surface, we must first know where the craters are located and how they are shaped. Such data is contained in lunar crater catalogs. In most cases, however, these catalogs were built for planetary science purposes and not for image registration purposes---with the specific scientific question motivating the catalog construction often influencing the contents of the catalog. This must be considered before repurposing a catalog for spacecraft navigation, since not all lunar crater catalogs are appropriate for building an index for crater pattern identification. 

There are a variety of lunar crater catalogs available; e.g., \cite{Robbins:2018}, \cite{Head:2010}, \cite{Povilaitis:2018}, and others \cite{Salamuniccar:2012,Wang:2015}. These catalogs were created with differing scientific objectives and from different raw data sets---thus, there should be no expectation of a one-to-one correspondence between their entries. For example, the minimum crater diameter is different for each catalog (and some catalogs even focus on very specific size ranges). As another example, some catalogs include all crater-like objects (e.g., \cite{Robbins:2018}), while others only include craters the authors are confident originated from an impact event (e.g., \cite{Povilaitis:2018}). Moreover, some catalogs attempt to only count the craters produced by direct impacts, thus ignoring those craters thought to be produced by secondary impacts (i.e., ejecta from the primary impact). This highlights the importance of selecting the crater catalog that best matches a particular mission's crater identification needs.

Crater catalogs are a useful tool for studying the statistical properties of the entire lunar crater population. To obtain global and comprehensive coverage, we usually exchange the detailed understanding of individual craters for a general understanding of all craters. Clearly, a detailed study of any single crater would yield a more sophisticated understanding of that crater---and likely produce crater information (e.g., size, shape, depth) that differs slightly from its corresponding entry in a global crater catalog.

The most comprehensive crater catalog to date is that of \cite{Robbins:2018}, containing about 1.3 million lunar impact craters having a diameter larger than about 1--2 km. This is the database we choose to use in this work.

\subsubsection{The Elliptical Crater Assumption}
\label{Sec:EllipCraterAssumption}
The natural shape of an impact crater is an ellipse. Almost all lunar craters are nearly elliptical, and many craters are nearly circular \cite{Bottke:2000}. The same is true for crater populations on other large celestial bodies \cite{Herrick:2012}. Clearly, no crater rim is a perfect ellipse or perfect circle and this is only an approximation of the general shape.

The are a variety of factors that determine the size and shape of an impact crater. With regards to the shape of lunar crater rims, the most important factors are the impactor velocity, impactor size, and lunar surface properties. Lower impact angles, $\phi$ (see Fig.~\ref{fig:ImpactGeom}), tend to produce craters of higher ellipticity---where we define \emph{ellipticity}, $\epsilon=a/b\geq1$, as the ratio of crater rim semi-major to semi-minor axis. The sensitivity of the crater ellipticity to this impact angle is a function of the \emph{cratering efficiency} (the ratio of crater diameter to impactor diameter) \cite{Bottke:2000,Elbeshausen:2013}, which acts as a model surrogate for impactor size and lunar surface properties. An empirical fit to hydrocode simulations and laboratory experiments shows excellent agreement between the impact angle at which $\epsilon=1.1$ occurs ($\phi_{\epsilon=1.1}$) and a simple power law relation \cite{Elbeshausen:2013},
\begin{equation}
    \phi_{\epsilon=1.1} = (90 \text{ deg}) \left(  \frac{D_{90}}{L} \right)^{-0.82}
\end{equation}
\begin{figure}[t!]
\centering
\includegraphics[width=0.7\columnwidth,trim=0in 0in 0in 0in,clip]{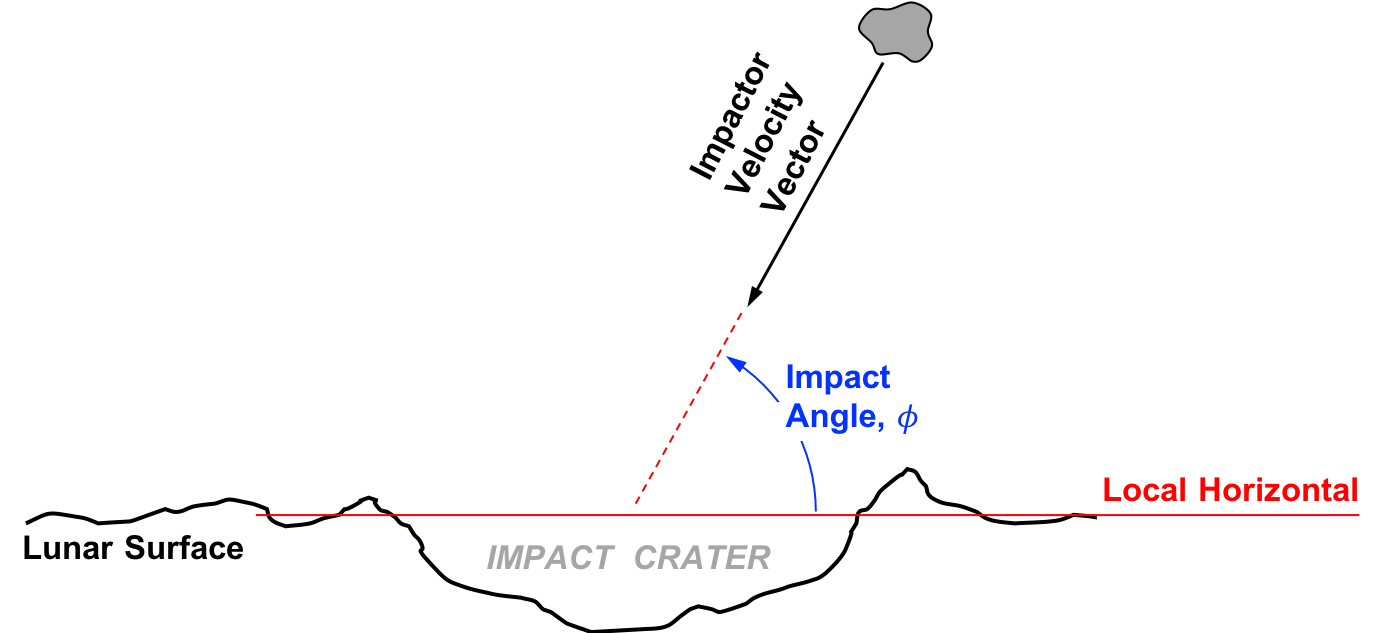}
	\caption{Geometry (side view) of impact crater creation, with the impact angle being defined as the angle above the local horizontal.}
	\label{fig:ImpactGeom}
\end{figure}
where $D_{90}$ is the crater diameter for a vertical impact and $L$ is the impactor diameter (i.e., $D_{90}/L$ is the cratering efficiency for a vertical impact). Thus, an impact with $\phi < \phi_{\epsilon=1.1}$ will produce a crater with $\epsilon > 1.1$.
The cutoff angle $\phi_{\epsilon=1.1}$ is generally between 3--5 deg (for sand, $D_{90}/L\sim60$) and 30 deg (for brittle rock, $D_{90}/L\sim4$) \cite{Elbeshausen:2013,Collins:2011,Michikami:2017}. 

Given the rather oblique impact angle required for an elliptical crater, it follows that a large percentage of lunar craters should be nearly circular. For an airless body such as the Moon, the cumulative distribution function (CDF) for the impact angle when the impactors are assumed to arrive from random directions is \cite{Shoemaker:1962}
\begin{equation}
    p(\phi \leq \phi_0) = \sin^2 \phi_0
\end{equation}
This model suggests that about 7.6\% of the crater population should have  an impact angle less than 5 deg and that about 25\% of the crater population should have an impact angle less than 30 deg. The database of \cite{Robbins:2018}, however, found a higher percentage of elliptical craters than might otherwise be expected and proposed a few hypotheses for this observation. Of particular note is that Robbins's database includes both primary and secondary impact craters, whereas many ellipticity studies (e.g., \cite{Bottke:2000}) only consider primary craters.

Regardless of the underlying physical cause, the newly available global distribution of crater ellipticity from \cite{Robbins:2018} makes clear that a circular crater assumption is not valid at the local level (i.e., when viewing craters of diameter smaller than about 30 km). This finding is significant, as many recent crater identification algorithms assume a circular crater model---a choice that is not well supported by the data. To see this, consider the ellipticity distribution as a function of crater size as shown in Fig.~\ref{fig:CraterEllipHist}. We immediately see that over half of the entire crater population has an ellipticity greater than 1.1 (the threshold planetary scientists usually use to define an ``elliptical'' crater \cite{Bottke:2000,Elbeshausen:2013}). Many of these craters, however, are not morphologically well preserved---and, therefore, are not good candidates for crater-based navigation since they will not produce good ellipse fits. If we restrict ourselves to well-defined craters whose catalog shape is supported by over 90\% of its circumference, we see a substantially larger percentage of the craters are nearly circular, especially when considering craters larger than 30 km.

\begin{figure}[t!]
\centering
\includegraphics[width=1\columnwidth,trim=0in 0in 0in 0in,clip]{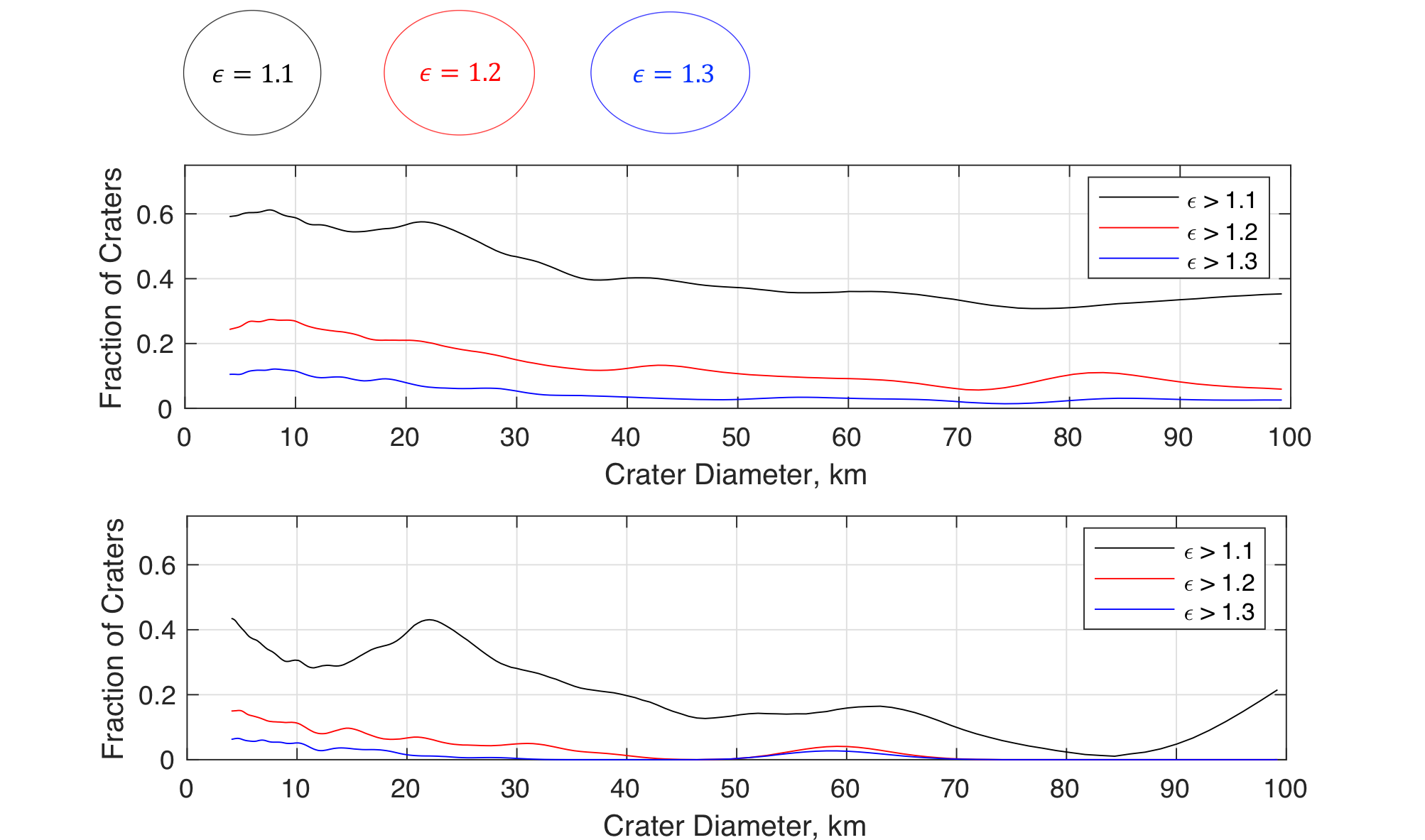}
	\caption{Lunar craters with well-defined rims tend to be less elliptical than the general population of craters. Plots show the fraction of craters above a specified ellipticity for every catalog crater (top) and for only those with well-defined crater rims (bottom). Visualization of crater ellipticity shown above plots for intuition and context. Results are based on post-processing of catalog data from \cite{Robbins:2018}.}
	\label{fig:CraterEllipHist}
\end{figure}

Topography or layered strength differences can create other irregular crater shapes \cite{Melosh:1989}.
On the Moon, simple craters ($\lessapprox$10--20 km in diameter) often have bowl-shaped morphologies, with small flat floors and few internal topographic features. Some may have deposits on the floor that originate from mass wasting or ponding of impact melt. The rims of simple craters are generally uplifted (unless significant degradation has occurred, as in the case of old simple craters), and the crater is surrounded by an ejecta blanket. More energetic impacts create larger, more complex craters. Complex craters ($\gtrapprox$10--20 km) have more varied morphologies, including central peaks, terraced walls, and flattened floors. Central peaks are formed by rebounding of material from deeply buried rocks, making these features especially interesting for studying materials excavated from lower in the crust. Larger complex craters can often have more than one central peak structure. Surrounding rocks in the wall of complex craters can also collapse, forming terrace-like structures. As impact features increase in size, they also increase in complexity. Craters larger than 300 km are usually classified as basins.

\begin{figure}[t!]
\centering
\includegraphics[width=1\columnwidth,trim=0in 0in 0in 0in,clip]{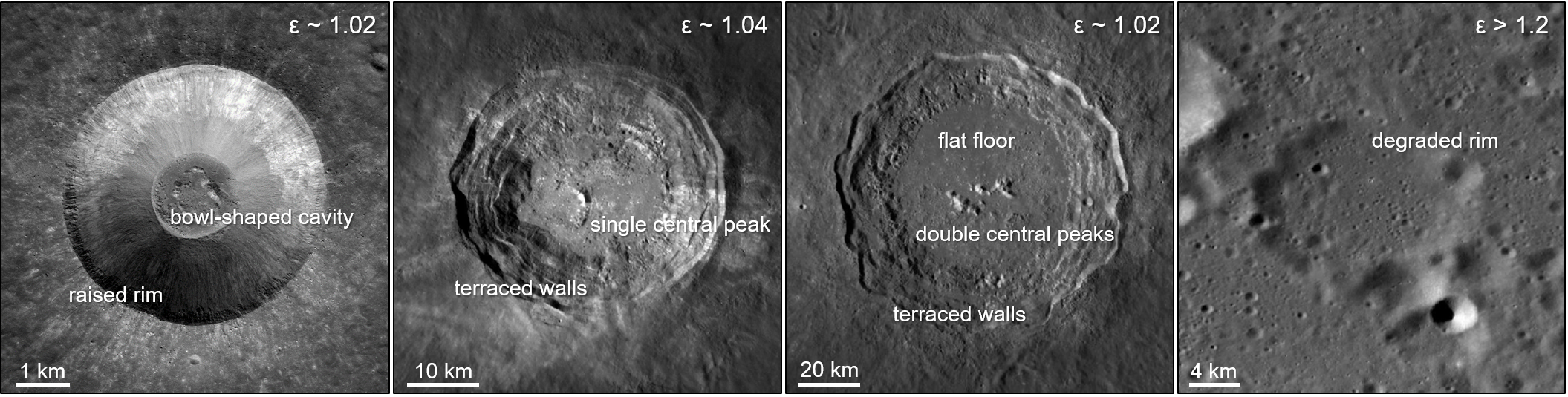}
	\caption{Craters of different sizes and ages exhibit a variety of morphologies. (A) Lichtenburg B (33.3 N, 61.5 W), a 5 km simple bowl-shaped crater with a raised rim. (B) Aristarchus (23.7 N, 47.4 W), a 40 km complex crater with terraced walls and a single central peak. (C) Copernicus (9.62 N, 20.08 W), a 93 km complex crater with double central peaks, terraced walls, and a flat floor. (D) Riccius W (38.9 S, 25. 2 E), a 19 km degraded crater with an infilled cavity and much of the rim no longer clearly visible.}
	\label{fig:CraterMorph}
\end{figure}

\subsubsection{Lunar Crater Size-Frequency Distribution}

Analyzing crater populations plays a major role in establishing planetary chronologies and in dating geologic units.  There is general consensus that the lunar cratering projectile flux was relatively constant within the last 3 Gyr \cite{Neukum:2001,Speyerer:2016}, with the period before 3 Gyr experiencing variable impact rates as a result of the hypothesized late heavy bombardment. 

Cumulative crater size-frequency distributions (SFDs), which describe the number of craters greater than a given diameter per measured area, are the usual data product used to estimate the age of planetary surfaces. Crater SFDs are conventionally determined by binning measured craters based on their diameters \cite{CraterWg:1979,Hiesinger:2000}, though modern practitioners are transitioning to more formal statistical approaches (e.g., using empirical density functions) \cite{Robbins:2018b}. These SFDs are often used to produce \emph{production functions} that describe  the rate at which new craters occur.

The most commonly used production function for estimating lunar surface ages is the Neukum Production Function (NPF) \cite{Neukum:1983,Neukum:2001}. The NPF models the formation rate of craters between 10 m and 300 km, and is generally used to estimate ages for geologic units on the Moon \cite{Neukum:2001}. The NPF is shown in Fig.~\ref{fig:NPF} and is given by
\begin{equation}
    \log_{10}( N ) = \sum_{k=0}^m a_k \left[ \log_{10}( D ) \right]^k
\end{equation}
where $D$ is the crater diameter, $N$ is the number of craters of size $D$ or smaller per unit area per unit time, and the coefficients $a_k$ (for $m=11$) are given in \cite{Neukum:2001}. The NPF suggest that that we should expect the production of new craters on the lunar surface to occur at a rate of $1.7 \times 10^{-5}$ new craters larger than 4 km per $\mathrm{km^2/Gyr}$. For the Moon, this results in approximately 635 new craters per Gy (or about $1.7\times 10^{-9}$ craters/day). Assuming impacts follow a Poisson distribution, we may straightforwardly compute the probability of an impact creating a crater larger than 4 km during a specified period of time. If a crater catalog must be valid for a period of days to a few years, the probability of a new crater appearing is essentially zero. Thus, it is quite reasonable to assume a static crater catalog.

\begin{figure}[t!]
\centering
\includegraphics[width=1\columnwidth,trim=0in 0in 0in 0in,clip]{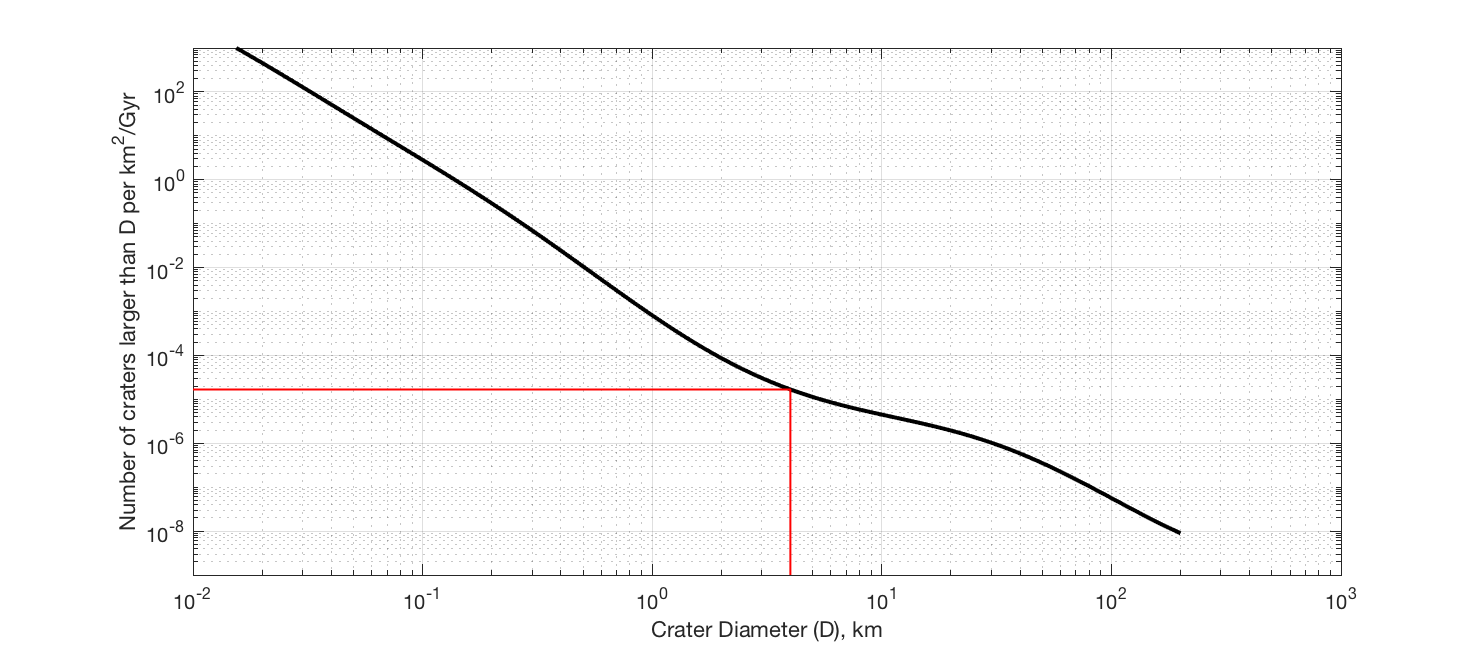}
	\caption{The Neukum Production Function  (NPF) describes the cumulative number of new lunar craters per km per Gyr below a specified diameter. The red line shows the NPF evaluated at a diameter of  4 km.}
	\label{fig:NPF}
\end{figure}

We note that recent work analyzing new impact craters as detected by the Lunar Reconnaissance Orbiter (LRO) Camera \cite{Speyerer:2016} found that the current cratering rate may potentially be higher than what the NPF predicts. Regardless, the expected discrepancies are small within the present context and the static crater catalog assumption remains valid.

\subsection{Crater Detection}
Successful crater detection is a prerequisite for crater identification. While not the focus of the present work, a brief review of crater detection algorithms (CDAs) is appropriate.

There have been a number of recent surveys of CDAs for both planetary science \cite{Salamuniccar:2008,DeLatte:2019} and spacecraft navigation \cite{Woicke:2018} applications. We observe here that the crater detection requirements for planetary science and navigation are often quite different---meaning algorithms developed for one application are not necessarily transferable to another application. There are two primary differences. First, from a scientific standpoint, the CDA objective is often to find all of the craters. For navigation, however, the objective is not to find every crater, but only enough to navigate. Any crater found that is not in the database becomes clutter that must be rejected. Thus, an algorithm that returns every crater is not necessarily better for navigation, and  we are more interested in algorithms that preferentially return only well-modeled craters that exist in our on-board catalog. The second primary difference between scientific and navigation CDAs is suitability for real-time use. CDAs for autonomous navigation must be fast enough to run in real-time on a space-qualified computing platform and robust enough to not require human supervision. No such requirements are generally placed on CDAs for scientific analyses. 


Of special note is the tremendous progress made in CDAs with machine learning over the last five years. Some notable algorithms include: DeepMoon\footnote{\url{https://github.com/silburt/DeepMoon}} \cite{Silburt:2019}, Python crater detection algorithm (PyCDA)\footnote{\url{https://github.com/AlliedToasters/PyCDA}} \cite{Klear:2018}, LunaNet \cite{Downes:2020}, CraterIDNet \cite{Wang:2018}, and others \cite{Emami:2015,Benedix:2018,DeLatte:2019}.

With well over 100 different algorithms contained in the CDA surveys \cite{Salamuniccar:2008,DeLatte:2019,Woicke:2018}, we find no need to discuss the topic further here. The interested reader is directed to these surveys and the references therein.

\subsection{Crater Identification}
\label{Sec:BackgroundCraterID}
Lost-in-space crater identification from a single image is a pattern recognition problem. When challenged with an image containing a set of observed craters, the task is to find the corresponding craters from within a large catalog of known craters. 

This problem is not new, and numerous attempts have been made in the past to recognize crater patterns. However, the lack of an existing invariant for non-coplanar ellipses has forced past work to make numerous simplifying assumptions or employ \emph{ad hoc} recognition schemes. These restrictive assumptions are relaxed by the developments introduced in this work. We now review some of the most notable prior work in crater identification and discuss their drawbacks using insights from invariant theory. As the reader will see, miscounting of independent invariants is a pervasive  problem in the crater identification literature---leading to descriptors that are incomplete,  redundant (i.e., not all elements are independent), or both. This recurring problem is one of the principal motivations of the present work.

Hanak \cite{Hanak:2009,Hanak:2010} uses an adaptation of star identification algorithms \cite{Mortari:2004,Samaan:2006} to recognize crater triads---an interpretation of the crater identification problem that has become influential within the spacecraft navigation community. Hanak chooses to index the crater pattern using the triangle's interior angles, which is not an invariant under general perspective projection. Therefore, the approach of \cite{Hanak:2009,Hanak:2010} operationally restricts usability to a nadir pointing camera with a modest field-of-view, which restricts the crater pattern to a similarity transformation where inter-crater angles remain invariant. If the camera is not nadir pointing, or if the spacecraft is far away from the Moon and craters are not nearly coplanar, the descriptor is no longer pose invariant and matching performance may be poor. Once potential matches are made, the authors of \cite{Hanak:2010} check these hypotheses against four additional metrics (three ratios of crater diameters to inter-crater distances, and a $\pm 1$ flag indicating if the triangle is clockwise or counterclockwise going form shortest to longest leg), which are also not projective invariants. We observe, therefore, that the authors of \cite{Hanak:2009,Hanak:2010} develop a six element descriptor, with none of the elements being a projective invariant. We contrast this to the seven known invariants for three coplanar conics. It is possible, therefore, to build a seven-element descriptor (see Section~\ref{Sec:TriadOfCoplanarConics}) that holds more descriptive power than \cite{Hanak:2010} and that is also a projective invariant.

Yu, et al., assume coplanar craters under affine transformation \cite{Yu:2014}, and develop an invariant descriptor based on the ratio of crater areas. The ratio of areas of two closed coplanar curves as invariant for affine transformations is discussed in \cite{Forsyth:1991}. Identification using only area ratios, however, neglects the spatial relationships between the craters and significantly reduces the descriptiveness of these invariants. Like the work of Hanak \cite{Hanak:2010}, the methods  from \cite{Yu:2014} are also limited to nearly nadir pointing images of nearly coplanar craters. 

Cheng, et al., develop a crater identification scheme by matching pairs of coplanar craters \cite{Cheng:2003,Cheng:2005} using invariants from \cite{Forsyth:1991} and \cite{Quan:1992}. A pair of coplanar conics possess two unique invariants, which may be computed preflight and indexed into a catalog for in-flight matching. The descriptive power of only two conics, however, is low---so multiple pairs must be compared, which is straightforward to accomplish. Additionally, the constraint of using coplanar conics limits the flight regimes where this technique works well. Despite these minor drawbacks, we consider the work of \cite{Cheng:2003,Cheng:2005} to be the most theoretically informed lost-in-space crater identification approach developed to date.

Park, et al., consider a pattern of three coplanar craters \cite{Park:2019} by looking at invariants of six coplanar points, where the six points are formed by the intersection of the triangle connecting the crater centers with the crater conics. These six coplanar points are split into six possible groupings of five points, from which the two unique invariants for each set of points may be computed using results from \cite{Forsyth:1991}.  Then, for each set of points, these two invariants are made insensitive to point ordering using the projective and permutation invariants (also called $p^2$ invariants) from \cite{Lenz:1994}. The pair of invariants for five coplanar points from \cite{Forsyth:1991} produce five $p^2$ invariants \cite{Lenz:1994}. The authors of \cite{Park:2019} concatenate the five $p^2$ invariants for each of the six sets of points to create a 30-element descriptor $(5\times6=30)$ for the crater triad. However, as discussed in \cite{Lenz:1994}, only two of the five $p^2$ invariants for a set of five coplanar points are independent (which should be expected, since making the two invariants from \cite{Forsyth:1991} insensitive to point ordering can not add new invariants). Therefore, the six sets of five points are really described by vector of $6 \times  2 = 12$ invariants, such that the 30-element descriptor from \cite{Park:2019} has 18 redundant elements that do not enhance discriminative power. Moreover, the six sets of five points are not constructed from unique points (they are the six possible five-point combinations of the same six points), thus the remaining 12 invariants are not all independent. We observe that six coplanar image points possess only four independent projective invariants. We also observe that $d\geq2$ coplanar conics possess only $5d-8$ independent projective invariants (seven invariants for a triad of craters), as discussed in Section~\ref{Sec:InvariantsConicsP2}. Consequently, the 30-element descriptor of \cite{Park:2019} has only four independent elements, 26 redundant (dependent) elements, and is missing three independent invariants.  We note here that the seven invariants for three coplanar craters may be computed directly from the observed image conics as discussed in Section~\ref{Sec:CoplanarInvariants}, without the need to compute sets of points. Likewise, invariants for clusters of four or more coplanar craters may also be computed directly \cite{Heisterkamp:1997}. The practical mechanics of computing such invariants are discussed in Section~\ref{Sec:TriadOfCoplanarConics}. If one wishes to build a permutation invariant descriptor (e.g., with $p^2$ invariants), we discuss how to do this for triads of craters in Section~\ref{Sec:PatternDescriptor}.

Maass, et al., propose an assortment of crater identification techniques \cite{Maass:2020}, though we only discuss their lost-in-space technique here.  They assume circular craters, thereby allowing explicit estimation of the crater normal in the camera frame (there are two possibilities \cite{Shiu:1989,Christian:2017AAS}). Given a set of three craters and their normals, the authors of \cite{Maass:2020} consider the $2^3=8$ possible configurations and choose the one where the normals are most similar (which presupposes a nearly planar configuration). More craters can be added to the pattern by iteratively solving a binary global optimization problem, with the cost function being the curvature energy of a cubic spline through a Delaunay triangulation of the crater centers. This resolves the 3D crater configuration (or reconstruction) to an unknown scale. Thus, Maass, et al., define a pattern descriptor using a pair of independent interior angles of a 3D crater triad, which we observe is a projective invariant since it is a reconstruction of the 3D geometry (and not the triangle in the image). Maass, et al., augment this check with the two independent ratios of crater radii. Pattern matches are verified by comparing the reprojected centers of additional craters. This usually requires two additional craters, thus requiring an image to possess five craters to yield a successful match. The method of \cite{Maass:2020} just described has three major drawbacks. First, we know from Section~\ref{Sec:EllipCraterAssumption} that the circular crater assumption is often not valid, with 25\%--50\% of small craters having an ellipticity $\epsilon > 1.1$. This effects the ability to accurately estimate the surface normal direction (e.g., a nadir-pointing image of an elliptical crater with $\epsilon = 1.1$ would have a normal direction error over 20 deg if it was incorrectly assumed circular). Second, the two parameter pattern descriptor (or four parameters if you count the radii ratios, though these do not appear to be explicitly used in the search) are not maximally descriptive---we know  there to be seven invariant parameters for a triad of nearly coplanar craters (see Sections~\ref{Sec:InvariantsConicsP2} and \ref{Sec:TriadOfCoplanarConics}). Thus, while the descriptor is indeed pose invariant (if the craters are circular), there are at least three independent pieces of information that remain for describing the crater pattern. Third, by verifying matches using only crater centers, more craters are required. Verifying using the crater rim reprojection (see Section~\ref{Sec:EllipseComp}) reduces the number of craters required.

There are a variety of other crater matching schemes \cite{Lu:2016,Kariya:2017,Cui:2018}, all of which make assumptions similar to the above methods.

This work is most similar to that of \cite{Cheng:2003,Cheng:2005}, though we relax the requirement that craters be coplanar---thus enabling a substantial expansion of the orbital regimes where crater matching may be performed. We also allow for the simultaneous consideration of more craters by using triads instead of pairs.

\section{Geometry of the Projection of a Single Crater}
\label{Sec:SingleCraterGeom}

\subsection{Mathematical Representation of a Crater}
Since impact craters are known to be elliptical in shape, many lunar crater databases store the ellipse parameters for the best-fit ellipse to the rim  of each cataloged crater. The most common database parameterization is the (1) latitude/longitude of the crater center $\{\varphi,\lambda\}$, (2) lengths of crater ellipse semimajor and semiminor axes $\{a,b\}$, and the orientation of the crater ellipse $\psi$ (often measured counterclockwise from East). This is the parameterization used in the Robbins database \cite{Robbins:2018}.

In this work, we assume that a crater is a planar feature defined by the 2D elliptical curve that best describes crater's rim. Circular craters are a special case ($a=b$) that is naturally handled without special treatment.

Let $\bxi_M$ be a three-dimensional (3D) point in the selenographic (i.e., Moon-centered, Moon-fixed) frame,
\begin{equation}
\bxi_{M} = 
\left[ \begin{array}{c c c}
          x_M \\
          y_M \\
          z_M
   \end{array} \right]
\end{equation}
which may be represented in homogeneous coordinates as
\begin{equation}
\bar{\bxi}_{M} = 
\left[ \begin{array}{c c c}
          x_M \\
          y_M \\
          z_M \\
          1
   \end{array} \right]
\end{equation}
A crater is a planar feature, so let the plane $\mathcal{P}_i$ containing crater $i$ be described by the $4 \times 1$ vector $\bpi_i$, such that
\begin{equation}
\label{eq:PointPlane3D}
    \mathcal{P}_i = \{ \bar{\bxi} \in \mathbb{P}^3 \; | \; \bpi^T_i \bar{\bxi} = 0 \}
\end{equation}

\subsubsection{Two-dimensional (2D) Crater Description}
Within the plane $\mathcal{P}_i$, define a 2D coordinate system with origin at the ellipse center and aligned with the ellipse principal axes (see Fig.~\ref{fig:EllipseDefinition}). In this case, we arrive at the canonical form of an ellipse, where the point $[x',y']$ lies on the ellipse if
\begin{align}
\label{eq:2DEllipseCanonical}
\frac{x'^2}{a^2} + \frac{y'^2}{b^2}  = 1
\end{align}
where $a$ is the ellipse semi-major axis and $b$ is the ellipse semi-minor axis. Now, define another 2D coordinate system within $\mathcal{P}_i$ with a different origin and different coordinate axis orientation (see Fig.~\ref{fig:EllipseDefinition}). Let the center of the ellipse be located at $[x_c,y_c]$ and the ellipse principal axes be rotated by an angle $\psi$. A point $[x,y]$ in this coordinate system is related to a point $[x',y']$ by a Euclidian transformation,
\begin{equation}
\left[ \begin{array}{c c c}
          x' \\
          y'
   \end{array} \right] = 
   \left[ \begin{array}{c c c}
          \cos \psi & \sin \psi \\
          -\sin \psi & \cos \psi
   \end{array} \right]
  \left[ \begin{array}{c c c}
          x-x_c \\
          y-y_c
   \end{array} \right] =
   \left[ \begin{array}{c c c}
          (x-x_c) \cos \psi + (y-y_c)\sin \psi \\
          -(x-x_c) \sin \psi + (y-y_c)  \cos \psi
   \end{array} \right]
\end{equation}
\begin{figure}[t!]
\centering
\includegraphics[width=0.65\columnwidth,trim=0in 0in 0in 0in,clip]{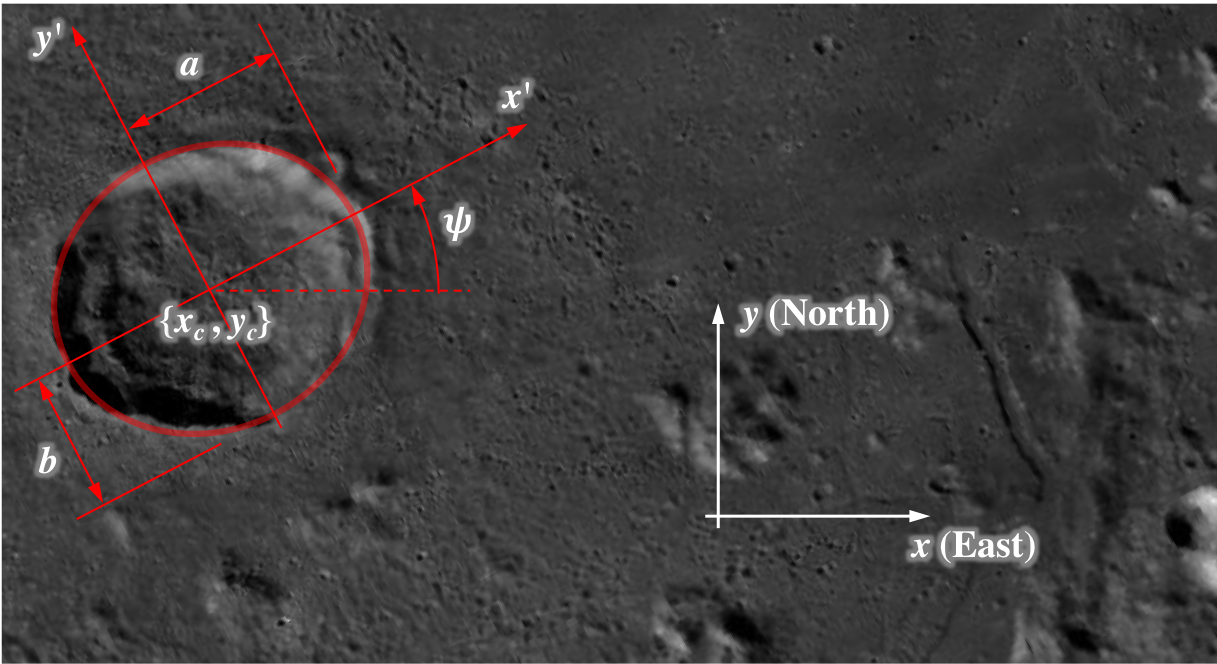}
	\caption{Visualization of ellipse parameterization in $\mathcal{P}_i$.}
	\label{fig:EllipseDefinition}
\end{figure}
Substituting this result into Eq.~\eqref{eq:2DEllipseCanonical} yields the implicit equation
\begin{align}
    \label{eq:ConicImplicit}
    A x^2 + B x y + Cy^2 + D x + F y + G = 0
\end{align}
where
\begin{align}
A & = a^2 \sin^2 \psi + b^2 \cos^2 \psi \label{eq:ConicImplicitA}\\
B & = 2 (b^2 - a^2) \cos \psi \, \sin \psi \label{eq:ConicImplicitB}  \\
C & = a^2 \cos^2 \psi + b^2 \sin^2 \psi \label{eq:ConicImplicitC} \\
D & = -2 A x_c - B y_c\\
F & = -B x_c - 2 C y_c\\
G & = A x_c^2 + B x_c y_c + C y_c^2 - a^2 b^2
\end{align}
Observe that Eq.~\ref{eq:ConicImplicit} is the generic equation for a conic, which happens to be an ellipse when $B^2 - 4AC < 0$. A simple calculation with Eqs.~\ref{eq:ConicImplicitA}--\ref{eq:ConicImplicitC} shows that $B^2 - 4AC = -4 a^2 b^2$, which will always satisfy the ellipse inequality constraint when $a \geq b  > 0$. The reverse mapping (implicit coefficients to explicit parameters) is straightforward and left to the reader (though the expression may be found in \cite{Christian:2010}).

The implicit equation from Eq.~\eqref{eq:ConicImplicit} may be written as a quadratic form using homogeneous coordinates. Therefore, letting a 2D point in the crater plane $\mathcal{P}_i$ be given by
\begin{equation}
\bx = 
\left[ \begin{array}{c c c}
          x \\
          y
   \end{array} \right]
\end{equation}
which may be written in homogeneous coordinates as
\begin{equation}
\bar{\bx} = 
\left[ \begin{array}{c c c}
          x \\
          y \\
          1
   \end{array} \right]
\end{equation}
it is easy to see that Eq.~\ref{eq:ConicImplicit} may be rewritten as
\begin{equation}
    \label{eq:ConicLocusDef}
\bar{\bx}^T \bC \, \bar{\bx} = 0
\end{equation}
where
\begin{equation}
\bC = 
\left[ \begin{array}{c c c}
          A & B/2 & D/2 \\
          B/2 & C & F/2 \\
          D/2 & F/2 & G
   \end{array} \right] 
\end{equation}

The expression in Eq.~\ref{eq:ConicLocusDef} describes a \emph{conic locus}, which is the locus of 2D points forming the path of the conic (see Fig.~\ref{fig:ConicLocusEnvelope}). The matrix $\bC$ describing a conic is a $3 \times 3$, real-valued, symmetric matrix of ambiguous scale. The ambiguous scale means that $\bC$ has only 5 degrees-of-freedom (which is also obvious from  Eq.~\ref{eq:ConicImplicit}) and that $\bC$ and $k \bC$ (with $k$ being a real-valued, non-zero scalar) describe the same conic. If $\bC$ is full rank (as it is with an ellipse), we have a proper conic---otherwise it is degenerate (either a pair of lines or a double line). Furthermore, if the conic is an ellipse, $\bC$ is indefinite---always having two eigenvalues of one sign and a third eigenvalue of the opposite sign, with none of the eigenvalues being zero (since it is a proper conic). 

\begin{figure}[b!]
\centering
\includegraphics[width=0.5\columnwidth,trim=0in 0in 0in 0in,clip]{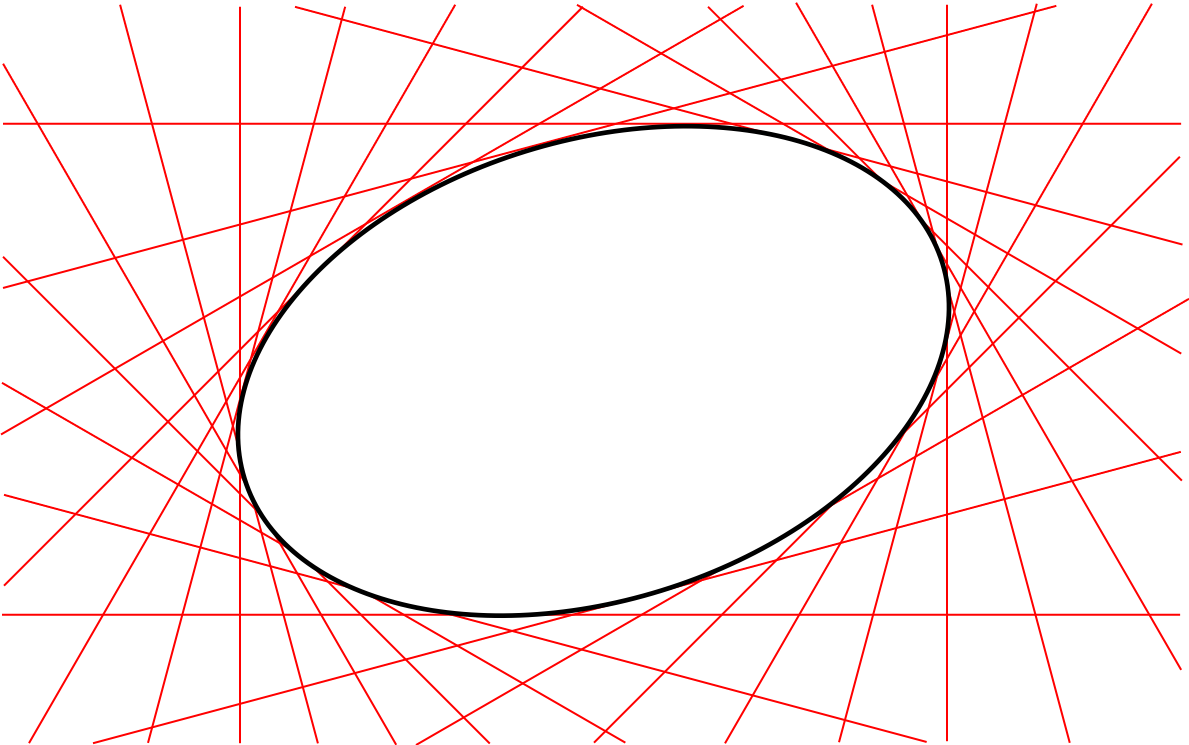}
	\caption{Visualization of simple conic locus (black) and lines belonging to the conic envelope (red).}
	\label{fig:ConicLocusEnvelope}
\end{figure}

Since $\bC$ is full rank for an ellipse, $\text{det}(\bC) \neq 0$ and
\begin{equation}
    \bC^{-1} = \text{det}(\bC)^{-1} \bC^{\ast} \propto \bC^{\ast} 
\end{equation}
where $\bC^{\ast}$ is the adjugate matrix of $\bC$. 

In projective geometry, points are dual to lines \cite{Semple:1952,Hartley:2003}. A 2D point written in homogeneous coordinates ($\bar{\bx}$) lies on the line $\bl$ if it satisfies the constraint $\bl^T \bar{\bx} = 0$. Clearly, $\bar{\bx}$ and $\bl$ can be exchanged in this relation---hence the point--line duality. To apply this duality property to conics, first observe that the line $\bl = \bC \bar{\bx}$ is tangent to the conic $\bC$ if $\bar{\bx}$ is a point on the conic \cite{Hartley:2003}. For a proper conic, one may therefore compute $\bar{\bx} = \bC^{-1} \bl$. Substitution into Eq.~\ref{eq:ConicLocusDef} yields
\begin{equation}
\bar{\bx}^T \bC \, \bar{\bx} = \left( \bC^{-1} \bl \right)^T \bC \, \left( \bC^{-1} \bl \right) = \bl^T \bC^{-T} \bC \bC^{-1} \bl = 0
\end{equation}
and, since $\bC$ is symmetric and $\bC^{-1} \propto \bC^{\ast}$,
\begin{equation}
\label{eq:ConicEnvelopeDef}
\bl^T \bC^{\ast} \, \bl = 0
\end{equation}
where $\bl$ is a line tangent to the conic. The utility of using adjugate matrices instead of inverses will become readily apparent as we proceed.

The expression in Eq.~\ref{eq:ConicEnvelopeDef} describes a \emph{conic envelope}, which is the family of tangent lines that encapsulate the ellipse (see Fig.~\ref{fig:ConicLocusEnvelope}).

\subsubsection{Three-dimensional (3D) Crater Description}\label{3Dcrater}
The elliptical crater rim is a 2D feature in 3D space. Thus, as we will show, the crater catalog data may be used to construct both the crater's plane and the disk quadric describing the 3D crater. The disk quadric is the natural and mathematically rigorous way to represent a 3D conic feature. 

Consider catalog crater $i$ with a geometric center at lunar latitude $\varphi_i$ and longitude $\lambda_i$. The 3D selenographic position of the crater center is
\begin{equation}
\label{eq:MCMFCraterLocation}
\bp_{M_i}^{(c)} = \rho_i
\left[ \begin{array}{c c c}
          \cos \varphi_i \cos \lambda_i \\
          \cos \varphi_i \sin \lambda_i \\
          \sin \varphi_i
   \end{array} \right]
\end{equation}
where $\rho_i$ is the distance of the plane $\mathcal{P}_i$ from the center of the Moon. The crater's center point is assumed to lie in the crater plane, $\bpi^T_i \bar{\bp}_{M_i}^{(c)} = 0 $. 

The Moon is very nearly a sphere. Therefore, to an excellent approximation, the crater normal is parallel to the selenographic crater center location $\bp_{M_i}^{(c)}$. If the celestial body were an oblate spheriod, the local ``up'' direction may be computed through an additional simple calculation \cite{Bowring:1985,Borkowski:1989}. Regardless, assuming a spherical Moon, define a local East-North-Up (ENU) coordinate system, such that
\begin{equation}
    \bu_i = \frac{ \bp_{M_i}^{(c)} }{\| \bp_{M_i}^{(c)} \|}
\end{equation}
\begin{equation}
    \be_i = \frac{ \bk \times \bu_i  }{ \| \bk \times \bu_i \| }
\end{equation}
\begin{equation}
    \bn_i = \frac{ \bu_i \times \be_i  }{ \| \bu_i \times \be_i \| }
\end{equation}
where $\bk^T = [0 \; 0\; 1]$ describes the selenographic description of the lunar pole. With the coordinate axes defined, construct the attitude transformation matrix (passive interpretation of a rotation \cite{Zanetti:2019}),
\begin{equation}
\bT^{E_i}_M =
\left[ \begin{array}{c c c}
          \be_i & \bn_i & \bu_i
   \end{array} \right]
\end{equation}
that transforms a vector expressed in the ENU frame of crater $i$ (defined as $E_i$) to the same vector expressed in the selenograhic frame $M$.

Since $\bu_i$ is normal to the plane and $\rho_i$ is the perpendicular distance to the center of the Moon, the vector $\bpi_i$ describing the plane $\mathcal{P}_i$ may be easily computed as
\begin{equation}
\bpi^T_i =
\left[ \begin{array}{c c c}
          \bu^T_i & -\rho_i
   \end{array} \right]
\end{equation}

Now, consider a 3D point $\bxi_{E_i}$ lying on the rim of crater $i$ as expressed in the crater's ENU frame $E_i$. This 3D point may be transformed into selenographic coordinates,
\begin{equation}
    \label{eq:3DPointOnCraterRimMoon}
    \bxi_{M} = \bp^{(c)}_{M_i} + \bT^{E_i}_M \bxi_{E_i}
\end{equation}
which may be compactly written in homogeneous coordinates
\begin{equation}
\bar{\bxi}_{M} =
\left[ \begin{array}{c c c}
          \bT^{E_i}_M & \bp^{(c)}_{M_i} \\
          \textbf{0}_{1 \times 3} & 1
   \end{array} \right]
   \bar{\bxi}_{E_i}
\end{equation}
Since the point $\bar{\bxi}_{E_i}$ lies in the plane $\mathcal{P}_i$, it is evident that 
\begin{equation}
\bar{\bxi}_{E_i} = 
\left[ \begin{array}{c c c}
          x \\
          y \\
          z \\
          1
   \end{array} \right] = 
   \left[ \begin{array}{c c c}
          x \\
          y \\
          0 \\
          1
   \end{array} \right]
\end{equation}
Thus, defining the columns of $\bT^{E_i}_M$ according to
\begin{align}
    \bT^{E_i}_M = 
    \left[ \begin{array}{c c c}
          \bt_{1_i}  & \bt_{2_i} & \bt_{3_i}
   \end{array} \right]
\end{align} 
we arrive at a relation between the 2D ellipse point $\bar{\bx}_{E_i}$ in the crater plane and its corresponding 3D point in the selenographic frame
\begin{equation}
\label{eq:IntermediateLunarProjection}
\bar{\bxi}_{M} = 
   \left[ \begin{array}{c c c}
          \bH_{M_i} \\
          \bk^T
   \end{array} \right] \bar{\bx}_{E_i}
\end{equation}
where $\bk^T = [0 \; 0\; 1]$ and $\bH_{M_i}$ is the $3 \times 3$ matrix
\begin{equation}
\label{eq:DefHMi}
\bH_{M_i} =
\left[ \begin{array}{c c c}
          \bt_{1_i}  & \bt_{2_i} & \bp^{(c)}_{M_i}
   \end{array} \right] =  
   \left[ \begin{array}{c c c}
          \bT^{E_i}_M \bS & \bp^{(c)}_{M_i}
   \end{array} \right] 
\end{equation}
and where $\bS$ is the $3 \times 2$ matrix
\begin{equation}
\label{eq:Smat}
\bS =
\left[ \begin{array}{c c c}
          \bI_{2 \times 2} \\
          \textbf{0}_{1 \times 2}
   \end{array} \right]
\end{equation}

Consider a quadratic surface in $\mathbb{P}^3$ (e.g., a sphere, an ellipsoid). Such surfaces are generally called \emph{quadrics}, with a 3D point $\bar{\bxi}$ lying on the surface if
\begin{equation}
\label{eq:QuadricLocus}
\bar{\bxi}^T \bQ \bar{\bxi} = 0
\end{equation}
where $\bQ$ is a $4 \times 4$ symmetric matrix describing the surface. This is the 3D analog to the 2D expression in Eq.~\ref{eq:ConicLocusDef}.

The 3D surface defined by Eq.~\ref{eq:QuadricLocus} is called a \emph{quadric locus}. The 3D conic locus (a curve) cannot be represented by a quadric alone. Instead the conic defines a proper quadric cone (where $\bQ$ is a $4 \times 4$ matrix of rank 3) and the 3D conic is formed by the intersection this quadric cone with the plane $\mathcal{P}_i$. Therefore, define the Moon-centered quadric cone (i.e., a cone with vertex at the center of the Moon) as
\begin{equation}
    \label{eq:QuadricConeCrateri}
    \mathcal{X}_i = \{ \bar{\bxi} \in \mathbb{P}^3 \; | \; \bar{\bxi}^T \bQ_i \bar{\bxi} = 0 \}
\end{equation}
such that the conic locus describing the elliptical crater is formed by the conic section
\begin{equation}
    \label{eq:ConicLocusConicSection1}
    \mathcal{C}_i = \mathcal{X}_i\cap \mathcal{P}_i
\end{equation}

Just as points and lines are dual in $\mathbb{P}^2$, points and planes are dual in $\mathbb{P}^3$ (a fact that should be evident from Eq.~\ref{eq:PointPlane3D}). It follows therefore, that one may construct a dual quadric, $\bQ_i^{\ast}$, such that the \emph{quadric envelope} is given by
\begin{equation}
    \mathcal{D}^{\star}_i = \{ \bpi \in \mathbb{P}^3 \; | \; \bpi^T \bQ_i^{\ast} \, \bpi = 0 \}
\end{equation}
where $\bpi$ is a plane tangent to the quadric. The quadric envelope provides a more natural way to describe a 3D conic.

The quadric envelope for a 3D conic is generally called the \emph{disk quadric} \cite{Semple:1952}, and defines all the planes tangent to the conic curve (see Fig.~\ref{fig:DiskQuadric}). This may be visualized as all the planes tangent to  an ellipsoid $x^2/a^2 + y^2/b^2 + z^2/c^2 = 1$  and letting $c \rightarrow 0$, thus resulting in a 3D disk (i.e., the 3D ellipsoid is collapsed to a plane and resembles a pancake or dinner plate). The disk quadric is, therefore, defined by a $4 \times 4$ matrix $\bQ_i^{\ast}$ that has rank 3. We may compute the disk quadric directly from the crater's conic envelope as
\begin{equation}
    \label{eq:DiskQuadricForCrater}
    \bQ_i^{\ast} \propto \left[ \begin{array}{c c c}
          \bH_{M_i} \\
          \bk^T
   \end{array} \right]
   \bC^{\ast}
   \left[ \begin{array}{c c c}
          \bH_{M_i} \\
          \bk^T
   \end{array} \right]^T =  
   \left[ \begin{array}{c c c}
          \bH_{M_i}\bC_i^{\ast}\bH_{M_i}^T & \bH_{M_i} \bC_i^{\ast} \bk  \\
          \bk^T \bC_i^{\ast} \bH_{M_i}^T & \bk^T \bC_i^{\ast} \bk
   \end{array} \right]
\end{equation}
which is a $4\times4$ symmetric matrix of rank 3.

\begin{figure}[b!]
\centering
\includegraphics[width=0.5\columnwidth,trim=0in 0in 0in 0in,clip]{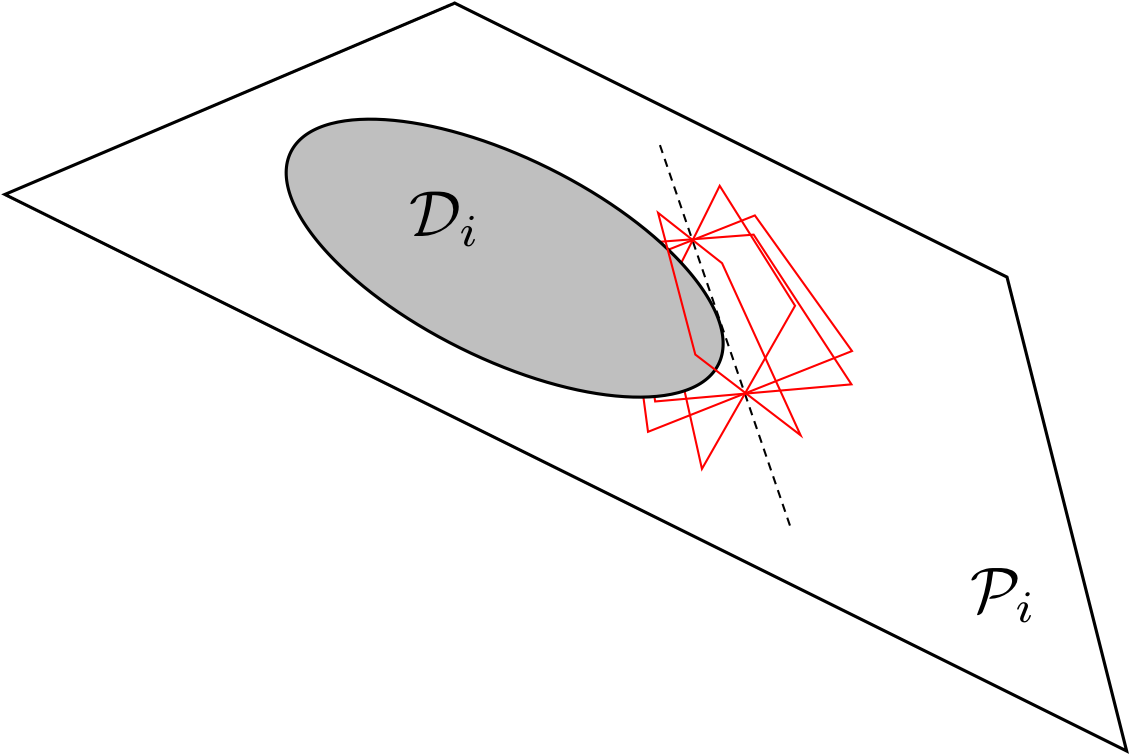}
	\caption{Visualisation of the disk quadric that defines all of planes tangent to $\mathcal{D}_i$ lying in the plane $\mathcal{P}_i$. A few example tangent planes are illustrated in red.}
	\label{fig:DiskQuadric}
\end{figure}

The disk quadric is clearly a more direct way to represent the 3D conic than the intersection of a proper quadric cone and a plane. Thus, the disk quadric is the representation of choice when one must concisely describe a 3D conic, such as a crater located on the surface of the Moon.

Each 3D conic (as defined by the disk quadric) corresponds to a symmetric $4\times 4$ matrix of rank 3, up to a scalar. The symmetric $4\times 4$ matrices form a $10$-dimensional vector space, so a nonzero $4\times 4$ symmetric matrix corresponds to a point in $\PP^9$. The constraint that the rank is $\leq 3$ defines a hypersurface in the projective space $\PP^9$ given by the vanishing of the determinant. The set $\varZ$ of all symmetric $4\times 4$ matrices of rank $3$ is an open subset in this hypersurface that parametrizes the set of conics. The dimension of $\varZ$ (and the hypersurface) is $9-1=8$.

\begin{table}[b!]
	\caption{Summary of geometric primitives related to a 3D crater.}
	\centering{}%
	\label{tab:3Dgeom}
	\begin{tabular}{ l l l }
	\hline 
	\hline 
	Coordinate- & Homogeneous & \\
	Free & Coordinate& \\
	Symbol & Representation & Description \\
	\hline 
	    $\mathcal{P}_i$ & $\bpi_i$ & Plane containing crater rim  \\
		$\mathcal{C}_i$ & $\bC_i$ & Conic locus of crater rim contour within $\mathcal{P}_i$  \\
		$\mathcal{C}^{\star}_i$  & $\bC^{\ast}_i$ & Conic envelope of crater rim contour within $\mathcal{P}_i$  \\
		$\mathcal{X}_i$ & $\bQ_i$ & Quadric cone with apex at Moon center passing through $\mathcal{C}_i$  \\
		$\mathcal{D}^{\star}_i$ & $\bQ^{\ast}_i$ & Disk quadric describing 3D crater  \\
	\hline 
	\hline 
\end{tabular}
\end{table}

\subsection{Homography and Action of a Projective Camera on a Crater Disk Quadric}
If a crater rim is a planar conic, then its projection into an image is also a conic. The projective transformation from one plane to another is a homography, thus allowing the appearance of the projected crater ellipse to be computed analytically.

If $\mathcal{C}_i$ is a 3D conic (see Eq.~\ref{eq:ConicLocusConicSection1} and Table~\ref{tab:3Dgeom}), we define its perspective projection into an image as $\mathcal{A}_i = \pi( \mathcal{C}_i )$, where $\pi:\PP^3\setminus\{\bar{\br}_M\}\to \PP^2$ and $\bar{\br}_M$ is the camera location. Thus, we describe the conic locus for $\mathcal{A}_i$ within the image as
\begin{equation}
\label{eq:ImageConicLocus}
    \mathcal{A}_i = \{ \bbu \in \mathbb{P}^2 \; | \; \bbu^T \bA_i \, \bbu = 0 \}
\end{equation}
where $\bbu^T = [u  \; v  \; 1]$ is the homogeneous coordinate representation of the image pixel coordinate $[u,v]$.  We now introduce a  coordinate system to make this projection explicit.

Consider a calibrated camera with selenographic position $\br_M$ and attitude $\bT^M_C$ that views lunar crater $i$. If $\bxi_M$ is the selenographic position of a point on the rim of crater $i$, then its location in the camera frame is simply
\begin{equation}
    \bxi_{C} = \bT^M_C \left( \bxi_M - \br_M \right)
\end{equation}
We can clearly construct the pinhole camera by letting $\bar{\bx}_C \propto \bxi_{C}$, where $\bar{\bx}_C$ is the projected image plane coordinate of the 3D  point $\bxi_{C}$. With a digital camera, however, we never observe a point $\bar{\bx}_C$, but instead we observe the 2D pixel coordinate $\bu^T = [u  \; v]$ (or $\bar{\bu}^T = [u  \; v  \; 1]$) corresponding to that location in the image plane.  Assuming the camera frame convention from \cite{Christian:2016}, we place the +$z$ axis out of the camera and along the optical axis, the +$x$ axis to the image right and in the direction of increasing pixel column number, and the $y$ axis completing the right-handed system. Furthermore, let the origin of the pixel $u$-$v$ system be in the upper left-hand corner of the image, the $u$ axis be to the right (increasing pixel column count), the $v$ axis be down (increasing pixel row count), and integer values of $[u,v]$ occurring at the pixel centers. With  these conventions, the conversion between image plane coordinates ($\bar{\bx}_{C}$) and digital image pixel coordinates ($\bbu$) is described by a simple affine transformation,
\begin{equation}
    \left[ \begin{array}{c c c}
           u  \\
           v \\
           1
   \end{array} \right]
     = \left[ \begin{array}{c c c}
          d_x & \alpha & u_p \\
          0 & d_y & v_p \\
          0& 0 & 1
   \end{array} \right]
   \left[ \begin{array}{c c c}
           x  \\
           y \\
           1
   \end{array} \right]_C
\end{equation}
or, more compactly,
\begin{equation}
    \bar{\bu} = \bK \bar{\bx}_C
\end{equation}
where the camera calibration matrix $\bK$ is a function of the five calibration parameters: $d_x$ (and $d_y$) is the ratio of the focal length to pixel pitch in the $x$ (or $y$) direction, $\alpha$ is the detector skewness, and $[u_p,v_p]$ is the coordinates of the principal point (where the optical axis intersects the image). We generally have excellent knowledge of $\bK$ for spacecraft OPNAV applications from in-flight calibration with star field images \cite{Christian:2016}.

Therefore, the 3D point $\bar{\bxi}_M$ projects to pixel coordinate $\bbu$ in the image according to
\begin{equation}
    \label{eq:EllipseProjection01}
    \bar{\bu} \propto \bP^M_C \bar{\bxi}_M
\end{equation}
where $\bP^M_C$ is the camera projection matrix for a specified absolute pose,
\begin{equation}
    \label{eq:ProjectionMatrixDef}
    \bP^M_C = \bK \left[ \begin{array}{c c c}
          \bT^{M}_{C} & -\br_C 
   \end{array} \right] = 
   \bK \bT^{M}_{C} \left[ \begin{array}{c c c}
           \bI_{3 \times 3} & -\br_M 
   \end{array} \right]
\end{equation}
The action of a projective camera on a quadric envelope is known to follow \cite{Hartley:2003}
\begin{equation}
    \label{eq:QuadricProjection}
    \bA^{\ast}_i \propto \bP^M_C \bQ^{\ast}_i \left(\bP^M_C \right)^T
\end{equation}
where $\bQ^{\ast}_i$ is from Eq.~\ref{eq:DiskQuadricForCrater}  and where $\bA_i$ (which  may be computed from $\bA^{\ast}_i$) describes the conic in the image plane tracing the apparent outline of the quadric. This allows for the analytic projection of the crater's disk quadric into  its apparent ellipse the image plane.

The result of Eq.~\ref{eq:QuadricProjection} may also be viewed as a homography. Returning  to Eq.~\ref{eq:EllipseProjection01}, substitute the result from Eq.~\ref{eq:IntermediateLunarProjection} for $\bar{\bxi}_M$,
\begin{equation}
    \bar{\bu} = \bK  \left[ \begin{array}{c c c}
          \bT^{M}_{C} & -\br_C 
   \end{array} \right]
   \left[ \begin{array}{c c c}
          \bH_{M_i} \\
          \bk^T
   \end{array} \right] \bar{\bx}_{E_i}
\end{equation}
\begin{equation}
    \bar{\bu} = \bH_{C_i} \bar{\bx}_{E_i}
\end{equation}
where we define $\bH_{C_i}$ to be the $3 \times 3$ matrix describing the homography between the crater plane $\mathcal{P}_i$ and the camera's image plane
\begin{equation}
    \label{eq:DefHCi}
    \bH_{C_i} = \bK \left[ \begin{array}{c c c}
          \bT^{M}_{C} & -\br_C 
   \end{array} \right]
   \left[ \begin{array}{c c c}
          \bH_{M_i} \\
          \bk^T
   \end{array} \right] = \bP^M_C \left[ \begin{array}{c c c}
          \bH_{M_i} \\
          \bk^T
   \end{array} \right]
\end{equation}
Since $\bH_{C_i}$ is a homography, the conic locus and conic envelope may be analytically transformed to the image plane. Observing that
\begin{align}
    0 = \bar{\bx}^T_{E_i} \bC_i \bar{\bx}_{E_i} &  =  \left( \bH_{C_i}^{-1} \bar{\bu} \right)^T \bC_i \left( \bH_{C_i}^{-1} \bar{\bu} \right) \\
    & =\bar{\bu}^T \left( \bH_{C_i}^{-T} \bC_i \bH_{C_i}^{-1} \right)\bar{\bu} \\
    & = \bar{\bu}^T \bA_i  \bar{\bu}  = 0
\end{align}
we arrive at the result
\begin{equation}
    \label{eq:HomographyPointConic}
    \bA_i  \propto \bH_{C_i}^{-T} \bC_i \bH_{C_i}^{-1}
\end{equation}
\begin{equation}
    \label{eq:HomographyPointConic2}
    \bH_{C_i}^{T} \bA_i \bH_{C_i} \propto  \bC_i 
\end{equation}
It follows, therefore, that
\begin{equation}
    \label{eq:HomographyConicEnvelope}
    \bA^{\ast}_i  \propto \bH_{C_i} \bC_i^{\ast} \bH_{C_i}^{T}
\end{equation}
\begin{equation}
    \bH_{C_i}^{-1} \bA^{\ast}_i \bH_{C_i}^{-T} \propto \bC_i^{\ast} 
\end{equation}
The result of Eq.~\ref{eq:HomographyConicEnvelope} is clearly the same as Eq.~\ref{eq:QuadricProjection}. We often find it convenient to make explicit the relative scales of $\bA_i$ and $\bC_i$, which may be done by introducing the scale $s_i$. For example,
\begin{equation}
    \label{eq:HomographyPointConicWithScale}
    \bH_{C_i}^{T} \bA_i \bH_{C_i} =  s_i  \bC_i 
\end{equation}

\section{Existence of Invariants from Projections of Crater Patterns}
\label{Sec:ExistenceOfInvariantsTopLevel}
The central premise of this work is that invariant theory is the proper framework for describing, indexing, and matching patterns of lunar crater rims as seen in digital images. Recognizing that invariant theory will be unfamiliar to many space scientists and engineers, we find it worthwhile to develop the concept fully.

Much of the existing literature pursues an \emph{ad hoc} approach for the construction of crater pattern descriptors. The lack of a formal framework has led to other authors proposing descriptors of varying dimension (ranging from two for a pair of craters \cite{Cheng:2003,Cheng:2005} to 30 for a triad of craters \cite{Park:2019}). However, for a given type of crater pattern there is a specified number of independent projective invariants. Descriptors with fewer elements than this do not fully exploit the perceptually salient information in the image. Descriptors with more elements than this increase complexity of the matching process without adding any perceptually salient information (the redundant invariants are all algebraic functions of the independent invariants). Thus, \emph{ad hoc} schemes for developing crater pattern descriptors generally provide suboptimal performance, either in terms of descriptive power or computational complexity.

\subsection{Preliminaries on Invariant Theory}
\label{Sec:InvariantTheory}

In invariant theory we study functions on a space that remain unchanged under a group of symmetries of that space. Invariant theory originated in the 19th century with pioneering work by Cayley, Clebsch, Gordan, Sylvester, Hilbert, and others. Groups of particular interest were  ${\rm GL}_n$, ${\rm SL}_n$, orthogonal groups, 
and finite groups (see~\cite{Dieudonne:1970,Dieudonne:1971}). The space of binary forms of degree $d$ (homogeneous polynomials in 2 variables of degree $d$)
with the ${\rm SL}_2$-symmetry was extensively studied in the 19th century \cite{Olver:1999}.
If $a(x,y)=a_0x^d+a_1x^{d-1}y+\cdots+a_dy^d$ is a binary form of degree $d$, then a matrix 
$$
\begin{bmatrix}
p &q\\
r &s
\end{bmatrix}
$$
acts on it by
$$
\begin{bmatrix}
p &q\\
r &s
\end{bmatrix}\cdot a(x,y)=a(px+ry,qx+sy)=a_0'x^d+a_1'x^{d-1}y+\cdots+a_d'y^d,
$$
where $a_0',\dots,a_d'$ are polynomial expressions in $a_0,a_1,\dots,a_d,p,q,r,s$.
For $d=2$, the discriminant $a_1^2-4a_0a_2$ is invariant under the ${\rm SL}_2$-action 
and every other polynomial invariant is a polynomial expression in the discriminant.
Over the 20th century, invariant theory has been generalized to an action of arbitrary algebraic group $\varG$ on arbitrary algebraic variety $\varV$. 

\subsubsection{Polynomial versus rational invariants}
Traditionally the focus has been on {\em polynomial} invariants; i.e., polynomial functions that remain unchanged under the group symmetries.
In practice, we generally desire a set of {\em fundamental} invariants. That is, we seek a set of invariants $f_1,f_2,\dots,f_r$ such that every other polynomial invariant $g$ is a polynomial expression in the fundamental invariants: $g=G(f_1,f_2,\dots,f_r)$ for some polynomial $G(x_1,\dots,x_r)$ in $r$ variables. In the language of commutative algebra,
the set of all polynomial functions on the variety $\varV$ form a {\rm commutative ring}, which we write as $\C[\varV]$ (or $\R[\varV]$ if we work over the real numbers). The subring generated by $f_1,f_2,\dots,f_r$ is the set of all polynomial expressions in $f_1,f_2,\dots,f_r$, which we may write compactly as $\C[f_1,\dots,f_r]$. We also define the {\em invariant ring} as the set of all invariant polynomials, forming a subring denoted as $\C[\varV]^\varG$.  

Now, observe that $f_1,f_2,\dots,f_r$ are fundamental invariants if and only if $\C[\varV]^\varG=\C[f_1,f_2,\dots,f_r]$.
Hilbert's Finiteness Theorem \cite{Hilbert:1890,Hilbert:1893} states that there is a finite list of fundamental invariants for groups such as ${\rm GL}_n$, ${\rm SL}_n$, and orthogonal groups. More generally, there is a finite system of fundamental invariants when the symmetry group is {\rm reductive} \cite{Nagata:1963,Haboush:1975}). For a precise definition of reductive, see~\cite{Humphreys:1975}. Not all groups are reductive though, and typically, translation groups may not be reductive. There are examples where there is no finite set of fundamental invariants \cite{Nagata:1959}, and in many examples, the smallest possible number of fundamental polynomial invariants is finite but extremely large compared to the dimension of the space. 

In practical applications, such as computer vision and spacecraft optical navigation, it may not be feasible to find a fundamental system of invariants. One remedy is to work with separating invariants (see~\cite[\S2.4]{Derksen:2015}), and it is even simpler to work with {\em rational invariants}. Thus, instead of dealing with polynomial functions, we will consider rational functions.

Let $\C(\varV)$ be the set of all rational functions on $\varV$, and $\C(\varV)^\varG$ be the set of all {\em invariant} rational functions on $\varV$. We assume that $\varV$ is an irreducible variety so that $\C(\varV)$ has the algebraic structure of a field---specifically, it is the quotient field of the ring $\C[\varV]$. It is easy to see that $\C(\varV)^\varG$ is a subfield of $\C(\varV)$.
The reader should be warned that the quotient field of $\C[\varV]^\varG$ can be strictly smaller then $\C(\varV)^\varG$.
For example, it is possible that there are only constant polynomial invariants while there are non-constant rational invariants. This is another reason why it may be preferable to work with rational invariants instead of polynomial invariants. 

Rational functions are not defined for all points in the space, but are defined for almost all points---which is good enough for practical crater identification. Invariant rational functions $f_1,f_2,\dots,f_r$ form a system of fundamental
{\rm rational} invariants if every other rational invariant $g$ is a rational expression in $f_1,f_2,\dots,f_r$; i.e., of the form $G_1(f_1,\dots,f_r)/G_2(f_1,\dots,f_r)$ where $G_1(x_1,\dots,x_r)$ and $G_2(x_1,\dots,x_r)$ are polynomials with coefficients in $\C$ and $G_2(f_1,\dots,f_r)\neq 0$. 
There always exists a finite system of fundamental {\it rational invariants} $f_1,f_2,\dots,f_r$.
In fact, it is possible to choose $r\leq 1+\dim \varV$ (see Remark~\ref{remark:s+1}).

\subsubsection{Real versus complex rational invariants}
Though the problem of spacecraft optical navigation (and computer vision, more generally) deals with geometric configurations over the real numbers, we will often work over the complex numbers. There is not a big difference. 
In the typical setup, we have a real algebraic variety $\varV_\R$ that parametrizes certain geometric configurations. For example, $\varV_\R$ could be a subvariety of the real projective space $\PP^m_\R$ defined by homogeneous polynomial equations with real coefficients. The same equations define also a projective variety $\varV_\C \subset \PP^m_\C$
in complex projective space if we view those equations over the complex numbers. We now assume that $\varV_\R$ is Zariski dense in $\varV_\C$,
which means that any complex polynomial
that vanishes on $\varV_\R$ also vanishes on $\varV_\C$. In this case, we can view the field $\R(\varV_\R)$ of real rational functions on $\varV_\R$ as a subfield of $\C(\varV_\C)$.
Moreover, the real and complex part of a rational function in $\C(\varV_\C)$ are rational functions on $\varV_\R$. This gives a decomposition
$$
\C(\varV_\C)=\R(\varV_\R)\oplus \R(\varV_\R)i,
$$
where $i=\sqrt{-1}$. We also assume that $\varG_\C$ is a complex algebraic group 
acting on the variety $\varV_\C$, such that the subgroup $\varG_\R$ of real group elements acts on $\varV_\R$ and $\varG_\R$ is Zariski dense in $\varG_\C$. Then, because $\varG_\R\subset \varG_\C$ is Zariski dense, a complex rational function is invariant under $\varG_\C$ if and only if it is invariant under the group $\varG_\R$. Thus, we find that
$$
\C(\varV_\C)^{\varG_\C}=\C(\varV_\C)^{\varG_\R}=
(\R(\varV_\R)\oplus \R(\varV_\R)i)^{\varG_\R}=
\R(\varV_\R)^{\varG_\R}\oplus \R(\varV_\R)^{\varG_\R}i.
$$
In particular, if $f_1,f_2,\dots,f_r$ are real rational invariant functions that generate the field $\R(\varV_\R)^{\varG_\R}$ over $\R$ (i.e.,
$\R(\varV_\R)^{\varG_\R}=\R(f_1,f_2,\dots,f_r)$),
then they also generate $\C(\varV_\C)^{\varG_\C}$ over $\C$ (i.e.,
$\C(\varV_\C)^{\varG_\C}=\C(f_1,f_2,\dots,f_r)$)

On the other hand,  if $\C(\varV_\C)^{\varG_\C}=\C(f_1,f_2,\dots,f_r)$
for some {\em complex} rational invariants $f_1,f_2,\dots,f_r$, then the real and complex parts of $f_1,f_2,\dots,f_r$ generate the field $\R(\varV_\R)^{\varG_\R}$ of real rational invariants.

\subsubsection{Counting independent rational invariants}
As we have seen, the proper counting of independent invariants is a recurring problem in the crater identification literature. Moreover, the usual method of counting invariants used in classical computer vision studies (e.g., a degrees-of-freedom approach) is not extensible to the problem at hand. We find it necessary, therefore, to develop a more formal framework. This framework is shown to reproduce known results for simple cases, before being extended to the problem of non-coplanar conics.

Suppose that $L$ is a field and $K$ is a subfield.
If $f_1,f_2,\dots,f_r\in L$, then $K(f_1,f_2,\dots,f_r)$ is the set of all rational expressions in $f_1,\dots,f_r$ with coefficients in $K$. This is also called the field that is generated by $f_1,f_2,\dots,f_r$ over $K$. So by definition, invariant rational functions $f_1,f_2,\dots,f_r$ are fundamental rational invariants for the action of $\varG$ on $\varV$
if and only if $\C(f_1,f_2,\dots,f_r)=\C(\varV)^\varG$.

In spacecraft navigation, computer vision, and other applications it is often desirable to know the maximum number of functionally independent invariants. This
relates to the notion of transcendence degree of a field extension.
If $f_1,f_2,\dots,f_r$ lie in a field, then we say
that $f_1,f_2,\dots,f_r$ are algebraically dependent over a subfield $K$ if there exists a nonzero polynomial $P$ with coefficients in $K$ with $P(f_1,f_2,\dots,f_r)=0$.
The transcendence degree of $L$ over the subfield $K$ is the supremum over all $r$ for which there exists $f_1,f_2,\dots,f_r\in L$ that are algebraically independent over $K$. Let $\trdeg(L/K)$ be the transcendence degree of $L$ over $K$. If $L$ is also a subfield of another field $M$, then we have the following relation (see~\cite[Theorem VI.1.11]{Hungerford:1974}):
$$
\trdeg(M/K)=\trdeg(M/L)+\trdeg(L/K).
$$
\begin{remark}\label{remark:s+1}
If $\varV$ is a variety, then $\trdeg(\C(\varV)/\C)$ is equal to the dimension of $\varV$. If $s=\dim \varV$ then we can choose
$f_1,f_2,\dots,f_s$ algebraically independent.
If $L=\C(f_1,f_2,\dots,f_s)$ then we have
$$s=\trdeg(\C(\varV)/\C)=\trdeg(K/\C)+\trdeg(\C(\varV)/K)=s+\trdeg(\C(\varV)/K)$$
This shows that $\C(\varV)/K$ is algebraic; i.e., every element of $\C(\varV)$ is algebraic over $K$. If $\C(\varV)\neq K$, then there exists an element $f_{s+1}\in \C(\varV)$ with
$$
\C(\varV)=K(f_{s+1})=\C(f_1,f_2,\dots,f_{s+1})
$$
by the Theorem of the Primitive Element (see~\cite[Theorem 4.6]{Lang:2002}).
\end{remark}

We are particularly interested in $\trdeg(\C(\varV)^\varG/\C)$ which is the maximal  number of rational invariants that are algebraically independent (over $\C$). Let $\varG\cdot v$ be the orbit of $v\in \varV$ under the $\varG$-action
and let $\varG_v=\{g\in \varG\mid g\cdot v=v\}$ be the stabilizer of $v$. Then we have $\dim \varG=\dim(\varG\cdot v)+\dim \varG_v$. If $s=\max_{v\in V} \dim(\varG\cdot v)$
is the largest dimension of an orbit, then almost all orbits in $\varV$ have dimension $s$ and we call $s$ the generic dimension of an orbit in $\varV$.
The following results follows from Rosenlicht's Theorem \cite{Rosenlicht:1956}.
\begin{theorem}\label{theorem:trdeg}
If $s$ is the dimension of a generic orbit in $\varV$,
then the maximum number of algebraically independent invariants is equal to $\dim \varV-s$.
\end{theorem}
\bproof
Rosenlicht proved in \cite{Rosenlicht:1956} that there exist a $\varG$-stable nonempty (Zariski) open subset $\varU\subseteq \varV$
that has a geometric quotient. A geometric quotient
is an algebraic variety $\varU/\varG$ together with
a surjective morphism $\pi : \varU \to \varU/\varG$ such that
the fibers of $\pi$ are exactly all $\varG$-orbits
and $\pi$ also has additional properties such 
as $\C(\varU/\varG)=\C(\varU)^\varG=\C(\varV)^\varG$. The fibers
of $\pi$ have dimension at most $s$ and almost all
fibers have dimension $s$. (Actually one can show that all the fibers have dimension $s$.) 
The maximum number of algebraically independent invariants is
$$
\trdeg(\C(\varV)^\varG/\C)=\trdeg(\C(\varU/\varG)/\C)=
\dim \varU/\varG=\dim \varU-s=\dim \varV-s.
$$
\eproof

\subsection{Rational invariants for conics in $\PP^2$}
\label{Sec:InvariantsConicsP2}
The result of Theorem~\ref{theorem:trdeg} is straightforwardly applied to the problem of counting the rational invariants for coplanar conics (i.e. a $d$-tuple of conics in $\PP^2$). We know from past results that two coplanar conics have two invariants \cite{Forsyth:1991,Mundy:1992b,Quan:1992}, that three coplanar conics have seven invariants \cite{Quan:1998}, and that a $d$-tuple of conics have $5d-8$ invariants  \cite{Heisterkamp:1997}. This is now briefly shown, before moving on to the more nuanced problem of non-coplanar conics.

Therefore, as an illustration, we can count the number of independent rational invariants for $d$ conics in $\PP^2$.
A conic in $\PP^2$ corresponds to a nonzero quadratic homogeneous polynomial in $3$ variables, up to scalar (see Eq.~\ref{eq:ConicImplicit}). So the variety of all conics can be identified with $\PP^5$ because such a polynomial has $6=5+1$ coefficients. Let $\varV$ be the variety of $d$-tuples of conics in $\PP^2$, Then $\varV\cong (\PP^5)^d$ has dimension $5d$. For $\varG$ we take the group ${\rm PGL}_3$ which has dimension $8$.
For $d=1$ the group ${\rm PGL}_3$ acts transitively on all nondegenerate quadratic forms. This means
that ${\rm PGL}_3$ has a Zariski dense orbit in $\varV=\PP^5$, so the dimension of a generic orbit is $s=5$ and the number of independent rational invariants is $\dim \varV-s=5-5=0$. For $d=2$, we can explicitly compute the stabilizer of a pair of conics (e.g., $x^2+2y^2-z^2=0$ and $2x^2+y^2-z^2=0$) and note that it is finite. 
This implies that the dimension of a generic orbit is equal to $\dim \varG=8$. This, in turn, implies that the dimension of a generic orbit is equal to $8$ for all $d\geq 2$ and the number of algebraically independent invariants is $\dim \varV-s=5d-8$.

\subsection{Rational invariants for conics in $\PP^3$}
A method for the robust identification of non-coplanar crater patterns has eluded spacecraft navigators since autonomous crater-based navigation was first explored over 25 years ago. The result is that most existing algorithms constrain the problem to local (nearly coplanar) crater patterns or to only nadir pointing images. To solve this problem requires us to first understand conics in $\PP^3$.

\subsubsection{Counting independent rational invariants for conics in $\PP^3$}
\label{Sec:ConicInvariantsInP3}
Let $\varZ$ be the $8$-dimensional variety introduced at the end of Section~\ref{3Dcrater} parameterizing conics in $\PP^3$, and let $\varV=\varZ^d$ be the space of $d$ conics in $\PP^3$. The 15-dimensional group $\varG={\rm PGL}_4$ acts on $\PP^3$, $\varZ$, and $\varV$. If $d=1$ then the group acts transitively on all nondegenerate conics so that $\varV=\varZ$ has a dense orbit
and there are no non-constant rational invariants. Thus, we search for invariants when $d \geq 2$.

Suppose that $d=2$ and consider the pair of conics
$\mathcal{C}_1$ and $\mathcal{C}_2$. If the conics are generic enough,
then the stabilizer of the pair $(\mathcal{C}_1,\mathcal{C}_2)$ is finite. To see this, suppose that $g\in {\rm PGL}_4$ lies in the connected component of the identity in the stabilizer of $(\mathcal{C}_1,\mathcal{C}_2)$. Then $g$ must also fix the planes $\mathcal{P}_1$ and $\mathcal{P}_2$ through $\mathcal{C}_1$ and $\mathcal{C}_2$ respectively.
Further, $g$ must also fix the intersections $\mathcal{C}_2\cap \mathcal{P}_1=\{\bbc_2^{(1)},\bbc_2^{(2)}\}$ and $\mathcal{C}_1\cap \mathcal{P}_2=\{\bbc_1^{(1)},\bbc_1^{(2)}\}$
and because $g$ lies in the connected component of the identity, it fixes all the points $\bbc_1^{(1)},\bbc_1^{(2)},\bbc_2^{(1)},\bbc_2^{(2)}$ on the line $\mathcal{P}_1\cap \mathcal{P}_2$ individually.
There are two points $\bbr_1^{(1)},\bbr_1^{(2)}$ on $\mathcal{C}_1$
such that the tangent lines to $\mathcal{C}_1$ at $\bbr_1^{(1)}$ and $\bbr_1^{(2)}$ go through $\bbc_2^{(1)}$. So $g$ must also fix $\bbr_1^{(1)}$
and $\bbr_1^{(2)}$. Among $\bbr_1^{(1)},\bbr_1^{(2)},\bbc_1^{(1)},\bbc_1^{(2)}$ in the plane ${\mathcal P}_1$
there are no three points on the line (Fig.~\ref{fig:CountingConicInvariant1}), because they all lie on the same conic $\mathcal{C}_1$. This implies
that $g$ must fix all the points in the plane $\mathcal{P}_1$.
Similarly it must fix all the points in the plane $\mathcal{P}_2$. From this follows  that $g$ is the identity
element in ${\rm PGL}_4$. Thus, the dimension of a general orbit is equal to $\dim {\rm PGL}_4=15$,
and the same is true for all $d\geq 2$.
It follows, therefore, that  the maximal number of algebraically independent rational invariants for $d\geq 2$ conics is $\dim \varV-15=8d-15$.

 \begin{figure}[b!]
\centering
\includegraphics[width=0.6\columnwidth,trim=0in 0in 0in 0in,clip]{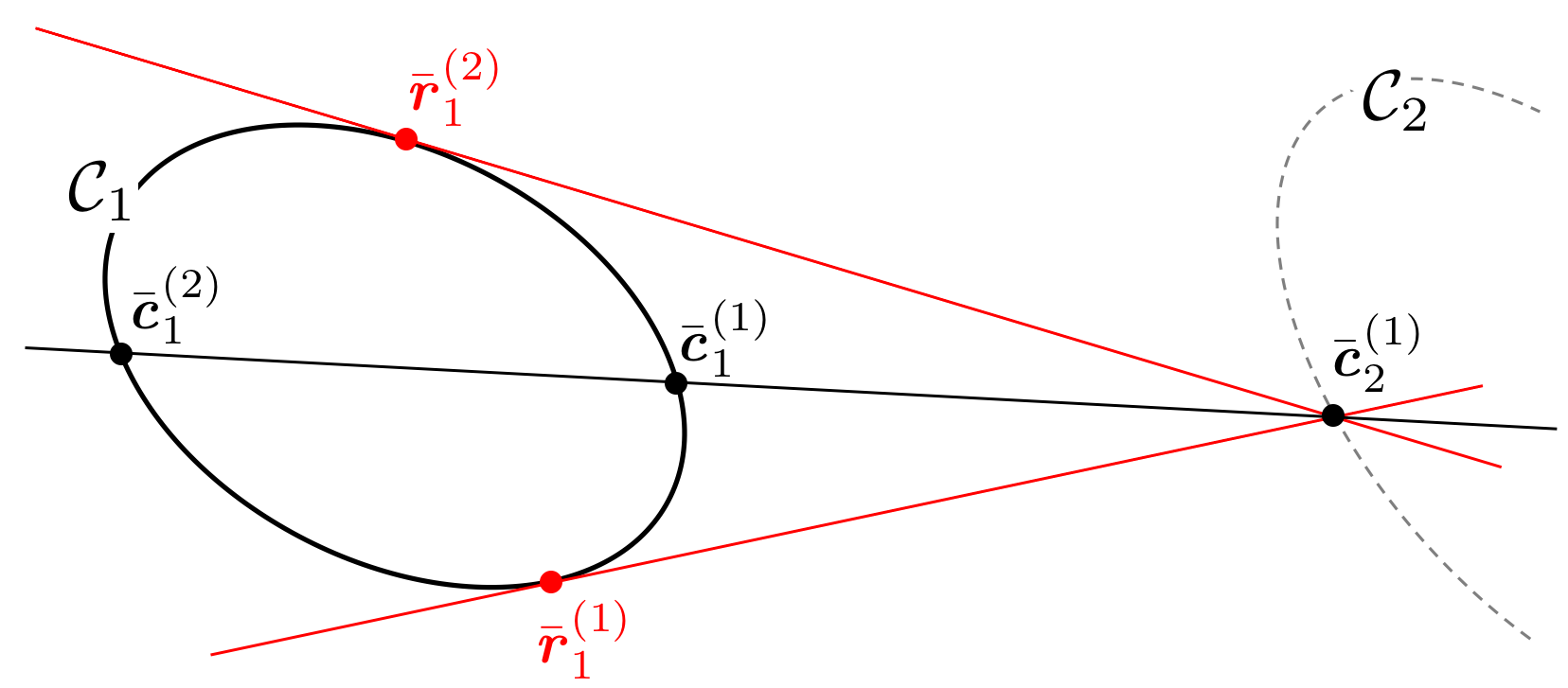}
	\caption{Since $\bbr_1^{(1)},\bbr_1^{(2)},\bbc_1^{(1)},\bbc_1^{(2)}$ all lie on $\mathcal{C}_1$  no three of these points may be on the same line.}
	\label{fig:CountingConicInvariant1}
\end{figure}

This result agrees  with the past literature that has identified one independent rational invariant for a pair of conics in $\PP^3$ \cite{Johnson:1914,Quan:1995b}. Our result, however, is more general and prepares us to study which invariants (if any) may be recovered from the projection of conics into an image.

\subsection{Calculating invariants from projections}

In Section~\ref{Sec:ConicInvariantsInP3}, we found the invariants for a $d$-tuple of conics in $\PP^3$. These, however, are not the invariants we need because we don't measure the crater rims directly in $\PP^3$. Instead, we image the lunar surface with a camera. Thus, we are in search of rational invariants that may be computed from only the projected crater rims that we see in an image. Such invariants do not always exist.

\subsubsection{Invariants from projections and the action of ${\rm PGL}_4$}
We begin by discussing the general setup for computing invariants from projections. To do this,  
we consider models $\mathcal{M}$ in $\C^3$ space
that are parameterized by some variety $\varV$.
For example, $\varV$ can consist of all $d$-tuples of points in $3D$, or all $d$-tuples of conics.
We also have camera position $\bbo\in \C^3$ and a perspective projection $\pi:\C^3\setminus\{\bo\}\to \PP^2$. We are interested in rational functions on $\varV$ (functions depending on the model $\mathcal{M}$)
that can be computed from the projected image $\pi(\mathcal{M})$ in a way that is independent on the camera position $\bbo$. We can extend $\C^3$ to the projective space $\PP^3$ and also extend the camera projection to a map $\PP^3\setminus\{\bbo\}\to \PP^2$, where $\bbo$ is equal to $\bo\in \C^3\subset \PP^3$, but viewed in $\PP^3$.
The group ${\rm PGL}_4$ acts on $\PP^3$
and contains the group of affine transformations of $\C^3$. The group ${\rm PGL}_4$ does not act on $\C^3$, but it acts on the field of rational functions on $\C^3$, because $\C^3$ and $\PP^3$
have the same rational functions. We assume
that the class of models parameterized by $\varV$
is closed under the action of ${\rm PGL}_4$
(for example, ${\rm PGL}_4$ takes conics to conics).
Then ${\rm PGL}_4$ also acts on the rational functions on $\varV$.

Suppose that $f$ is a rational function on $\varV$ such that $f(\mathcal{M})$ can be computed from the projected image $\pi(\mathcal{M})$ in such a way that it does not depend on the choice of the position $\bbo$ of the camera. So there exists a function $h$
such that $f(\mathcal{M})=h(\pi(\mathcal{M}))$
for every choice of $\bbo$. Let $\varG_{\bbo}$
be the image of the group
$$
\left\{\left[\begin{array}{cccc}
1 & 0 & 0 & 0\\
0 & 1 & 0 & 0\\
0 & 0 & 1 & 0\\
x& y & z & w \end{array}\right]\Big| w\neq 0\right\}
$$
in ${\rm PGL}_4$.
An element $g\in \varG_{\bbo}$ fixes $\bbo$ and all the lines through $\bbo$, so we have
$$
f(g\cdot \mathcal{M})=h(\pi(g\cdot \mathcal{M}))=
h(\pi(\mathcal{M}))=f(\mathcal{M}).
$$
So $f$ is invariant under the action of $\varG_{\bbo}$. Since this is true for every choice of $\bbo$, we see that $f$ is invariant under the group $\varG$ generated by all $\varG_{\bbo}$, $\bbo\in \PP^3$. One can verify that $\varG={\rm PGL}_4$. To see this, for example, we observe that the group $\varG$ is closed under conjugation with elements in ${\rm PGL}_4$. This implies that $\varG$ is a non-trivial normal subgroup of ${\rm PGL}_4$. Since ${\rm PGL}_4$ is known to be a simple group, we get $\varG={\rm PGL}_4$. This shows that any rational function that can be computed from a camera projection in a way that is independent on the choice of the position of the camera must be invariant under the action of ${\rm PGL}_4$.

\subsubsection{Invariants from projections of points in $\PP^3$}
\label{Sec:InvariantsForPointsP3}
The simplest example of invariants from projections is the case of a $d$-tuple of points in $\PP^3$. It is well known that $d$ arbitrarily placed 3D points possess no invariants that one may compute from their projection in an image \cite{Clemens:1991,Burns:1992}. Thus, a cloud of random 3D points cannot be indexed for recognition by a pose invariant descriptor, which is a critically important fact. We reproduce this known result here in our more formal framework. In doing so, we develop some of the ideas and tools necessary for the case of conics in $\PP^3$ but within the context of a simpler (and familiar) example.

Therefore, for example, suppose we have $d$ points in $\PP^3$. For $d\geq 5$ there are $3d-15$ independent invariants for the action of ${\rm PGL}_4$ (because for $d\geq 5$, the generic
stabilizer in $(\PP^3)^d$ is finite). Perspective projection for some fixed camera gives $d$ points in $\PP^2$. Unfortunately, however, no non-constant rational invariant can be computed from just the image. This may seem counter-intuitive, and we now show why this is the case. 

Suppose we fix a model $\mathcal{M}\in (\PP^3)^d$ consisting of $d$ points. The image may change in appearance with varying camera viewpoint, with the projection being governed by the action of ${\rm PGL}_4$. The variety of the possible images from the model as we move the camera cuts out a variety $\varU\subseteq (\PP^2)^d$ of dimension at most $\dim {\rm PGL}_4=15$. For $d\geq 8$, the dimension of $\varU$ is strictly smaller than the dimension of $(\PP^2)^d$ and one might expect to have at least $2d-15$ independent invariants that can be computed from the images, but this is wrong as we will see. The problem is that Theorem~\ref{theorem:trdeg} does not apply. The variety $\varU$ of possible images is {\rm not} an orbit for any group action on $(\PP^2)^d$: the group ${\rm PGL}_4$ does not act on $(\PP^2)^d$, and $\varU$ may be bigger than any ${\rm PGL}_3$ orbit.

To proceed, suppose that $\varV=(\PP^3)^d$ and $\varY=(\PP^2)^d$, $\bbo\in \PP^3$ is the position of the camera, and $\pi:\varV\to \varY$ is
the projection. (The map $\pi$ is only defined for
elements in $(\PP^3\setminus \{\bbo\})^d$.
If $h\in \C(\varY)$ is a rational function, then we can pull it back to get a rational function $\pi^\star h=h\circ \pi\in \C(\varV)$, where $\circ$ is the composition. Using the inclusion $\pi^\star:\C(\varY)\hookrightarrow \C(\varV)$ we may view
$\C(\varY)$ as a subfield of $\C(\varV)$. We are interested in rational invariants that can be computed from the image in $\varY$. Such invariants are exactly elements in the intersection field $\C(\varY)\cap \C(\varV)^\varG$.
We get the following diagram of field extensions
with their transcendence degrees:
$$
\xymatrix{
& \C(\varV)\ar@{-}[ld]_{d}\ar@{-}[rd]^{15} & \\
\C(\varY)\ar@{-}[rd]^{?}\ar@{.}[rdd]_{2d} & & \C(\varV)^\varG\ar@{-}[ld]_{?}\\
& \C(\varY)\cap \C(\varV)^\varG\ar@{-}[d]^{?} &\\
& \C &}.
$$
We will see that $\C(\varY)\cap \C(\varV)^\varG=\C$,
which means that that there are no non-constant
rational invariants that can be computed from the image. Suppose that $f\in \C(\varY)\cap \C(\varV)^\varG$, or to be more precise, there exists an $h\in \C(\varY)$ such that $f=h\circ \pi\in \C(\varV)^\varG$. We will show that $f$ (and $h$) must be constant.

To understand what is happening, we introduce an equivalence relation $\sim$ on $(\PP^2)^d$, the variety of possible images of $d$ points in $\PP^2$. We say $\mathcal{I}_1\sim \mathcal{I}_2$ is true
when there exists a model $\mathcal{M}$ such that $\mathcal{I}_1,\mathcal{I}_2$
both appear as images of $\mathcal{M}$ (but possible from different camera positions). 
The relation $\sim$
is {\em not} an equivalence relation. It satisfies
the reflexivity axiom ($\mathcal{I}\sim \mathcal{I}$) and the symmetry axiom ($\mathcal{I}_1\sim \mathcal{I}_2$ $\Leftrightarrow$ $\mathcal{I}_2\sim \mathcal{I}_1$). However, it does not satisfy the transitivity axiom
(if $\mathcal{I}_1\sim \mathcal{I}_2$ and $\mathcal{I}_2\sim \mathcal{I}_3$, then $\mathcal{I}_1\sim \mathcal{I}_3$). If $\pi$ and $\pi'$ are the projections
with respect to different camera positions/orientations, then $\pi'$ is obtained
from $\pi$ by a projective linear transformation $g$, so that $\pi'(\mathcal{M})=\pi(g\cdot \mathcal{M})$ for any model $\mathcal{M}$. If $\mathcal{I}\sim \mathcal{I}'$, then we have $\mathcal{I}=\pi(\mathcal{M})$ and
$\mathcal{I}'=\pi'(\mathcal{M})$ for some model and some camera projections $\pi,\pi'$. This means that
$h(\mathcal{I})=h(\pi(\mathcal{M}))=f(\mathcal{M})=f(g\cdot \mathcal{M})=f(\pi(g\cdot \mathcal{M}))=h(\mathcal{I}')$.

Let $\equiv$ be the equivalence relation generated by $\sim$. So we say that $\mathcal{I}\equiv \mathcal{I}'$
if and only if there is a finite sequence of images
$\mathcal{I}=\mathcal{I}_0,\mathcal{I}_1,\mathcal{I}_2,\dots,\mathcal{I}_r=\mathcal{I}'$ such that $\mathcal{I}_0\sim \mathcal{I}_1,\mathcal{I}_1\sim \mathcal{I}_2,\dots,\mathcal{I}_{r-1}\sim \mathcal{I}_r$.
If $\mathcal{I}\equiv \mathcal{I}'$ and $\mathcal{I}_0,\mathcal{I}_1,\dots,\mathcal{I}_r$ are as above,
then $f(\mathcal{I})=f(\mathcal{I}_0)=f(\mathcal{I}_1)=\cdots=f(\mathcal{I}_r)=f(\mathcal{I}')$.

We will now show that $\mathcal{I}\equiv \mathcal{I}'$ for all images $\mathcal{I},\mathcal{I}'\in \varY$. Let $(\bbp_1,\bbp_2,\dots,\bbp_d)\in (\PP^2)^d=\varY$ and $\bbq_1\in \PP^2$. Define $\bbo\in \PP^3$
as the position of a camera and $\pi:\PP^3\to\PP^2$ as the camera projection.
Let $\mathcal{L}_1,\mathcal{L}_2,\dots,\mathcal{L}_d,{\mathcal N}_1$ be the lines through $\bbo$ in $\PP^3$
corresponding to the points $\bbp_1,\bbp_2,\dots,\bbp_d,\bbq_1$ respectively.
Now, as shown in Fig.~\ref{fig:ProjInvariantPoints}, define $\mathcal{P}$ as the plane through the lines $\mathcal{L}_1$ and $\mathcal{N}_1$ (which intersect at $\bbo$).
 Let us choose another camera position $\bbo'$ in
the plane $\mathcal{P}$ randomly. 
We randomly choose points $\bba_1,\bba_2,\dots,\bba_d$ on the lines
$\mathcal{L}_1,{\mathcal L}_2,\dots,{\mathcal L}_d$ respectively (i.e., the points
$\bba_1,\bba_2,\dots,\bba_d$ are in general position).
Let ${\mathcal L}_j'$ be the line through $\bbo'$ and $\bba_j$ for all $j$. 
Since $\bbo'$ and $\bba_1$ lie in the plane $\mathcal{P}$, so does the line $\mathcal{L}_1'$.
Therefore, the lines $\mathcal{L}_1'$ and ${\mathcal N}_1$ intersect at some point $\bbb_1$.
Now we have $(\bbp_1,\bbp_2,\dots,\bbp_d)=\pi(\bba_1,\bba_2,\dots,\bba_d)$ and $(\bbp_1',\bbp_2',\dots,\bbp_d')=\pi'(\bba_1,\bba_2,\dots,\bba_d)$ so $(\bbp_1,\bbp_2,\dots,\bbp_d)\sim (\bbp_1',\bbp_2',\dots,\bbp_d')$.
Moreover, we have $(\bbp_1',\bbp_2'\dots,\bbp_d')=\pi'(\bbb_1,\bba_2,\dots,\bba_d)$
and $(\bbq_1,\bbp_2,\dots,\bbp_d)=\pi(\bbb_1,\bba_2,\dots,\bba_d)$, so $(\bbp_1',\bbp_2',\dots,\bbp_d')\sim (\bbq_1,\bbp_2,\dots,\bbp_d)$. From this follows that $(\bbp_1,\bbp_2,\dots,\bbp_d)\equiv (\bbq_1,\bbp_2,\dots,\bbp_d)$.
Repeating this argument shows that for any points $\bbp_1,\bbp_2,\dots,\bbp_d,\bbq_1,\bbq_2,\dots,\bbq_d\in \PP^2$ we have
$$(\bbp_1,\bbp_2,\dots,\bbp_d)\equiv (\bbq_1,\bbp_2,\dots,\bbp_d)\equiv(\bbq_1,\bbq_2,\bbp_3,\dots,\bbp_d)\equiv \cdots\equiv (\bbq_1,\bbq_2,\dots,\bbq_d).
$$
So we have $(\bbp_1,\bbp_2,\dots,\bbp_d)\equiv (\bbq_1,\bbq_2,\dots,\bbq_d)$. We conclude that $h(\bbp_1,\bbp_2\dots, \bbp_d)=h(\bbq_1,\bbq_2,\dots,\bbq_d)$. In other words,
$h$ and $f$ are constant.

 \begin{figure}[b!]
\centering
\includegraphics[width=0.6\columnwidth,trim=0in 0in 0in 0in,clip]{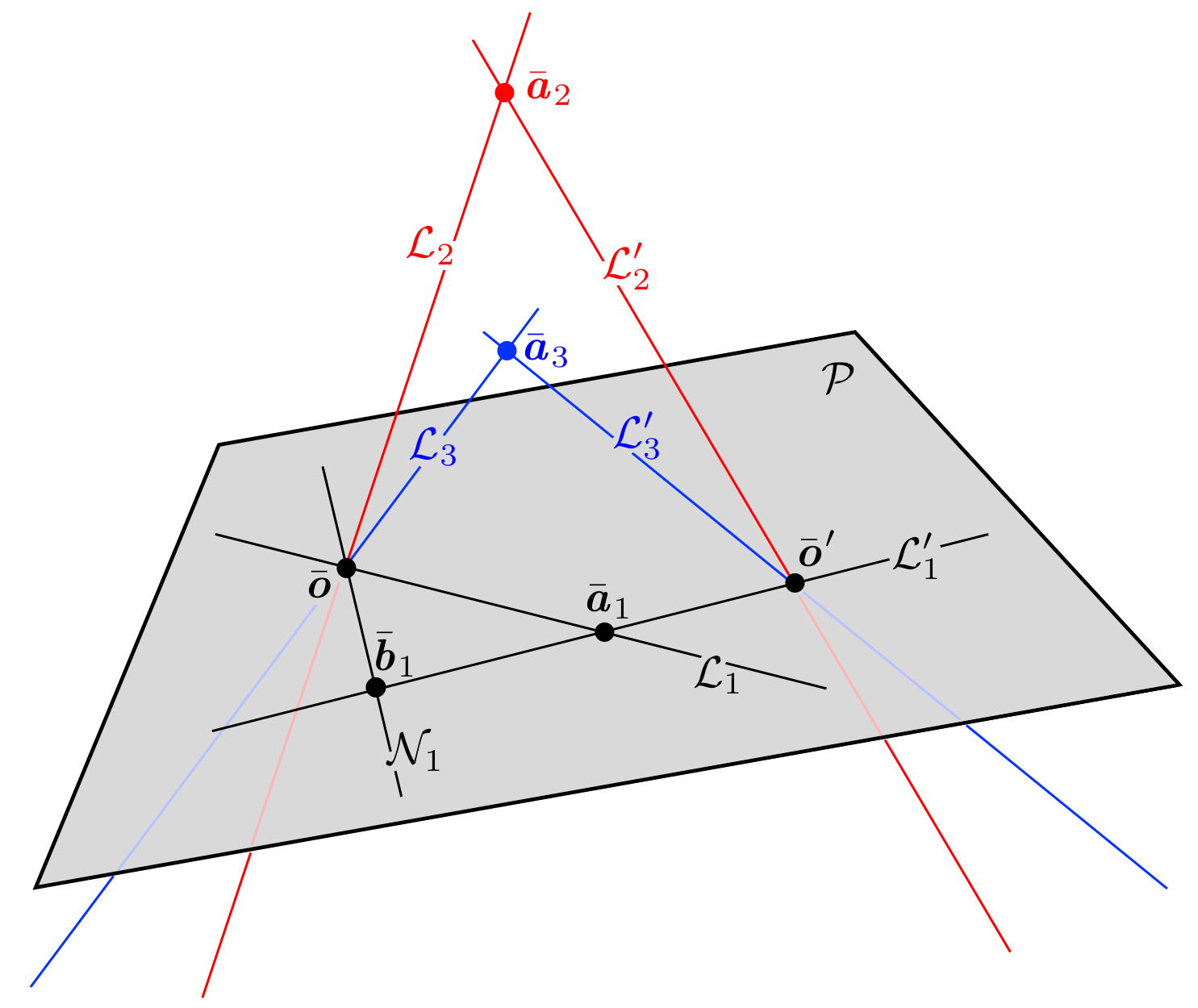}
	\caption{Visualization of geometric quantities used to show the  non-existence of invariants from the projection of arbitrary points in $\PP^3$. Illustrated here are the line-point sets $\mathcal{L}_1 \cap \mathcal{L}'_1 = \bba_1$,  $\mathcal{N}_1 \cap \mathcal{L}'_1 = \bbb_1$, $\mathcal{L}_2 \cap \mathcal{L}'_2 = \bba_2$, and $\mathcal{L}_3 \cap \mathcal{L}'_3 = \bba_3$ that result from the two camera locations $\bar{\bbo}$ and $\bar{\bbo}'$. All quantities in black lie in the plane $\mathcal{P}$.}
	\label{fig:ProjInvariantPoints}
\end{figure}

\subsubsection{Invariants from projections of conics in $\PP^3$}
\label{Sec:InvariantsProjArbitraryConicsP3}
The corollary to a $d$-tuple of arbitrary points in $\PP^3$ is a $d$-tuple of arbitrarily placed conics in $\PP^3$.
We will show that there are no non-constant invariants for $d$ conics in $\PP^3$ that can
be computed from a perspective projection. 
The argument is similar to the argument in Section~\ref{Sec:InvariantsForPointsP3} using an equivalence relation on $d$-tuples of conics in $\PP^2$. We use an explicit geometric construction to show that all $d$-tuples are equivalent.

As before, let $\varZ$ be the variety of conics in $\PP^3$ defined at the end of Section~\ref{3Dcrater}. Then, let $\varV=\varZ^d$ be the variety of $d$-tuples of conics in $\PP^3$ and $\varY=(\PP^5)^d$
be the variety of $d$-tuples of conics in $\PP^2$. We define a relation $\sim$ on $\varY$
by $(\mathcal{A}_1,\mathcal{A}_2,\dots,\mathcal{A}_d)\sim (\mathcal{A}_1',\mathcal{A}_2',\dots,\mathcal{A}_d')$ if there exists two camera projections $\pi$ and $\pi'$ and conics
$(\mathcal{C}_1,\mathcal{C}_2,\dots,\mathcal{C}_d)$ with $\pi(\mathcal{C}_1,\mathcal{C}_2,\dots,\mathcal{C}_d)=(\mathcal{A}_1,\mathcal{A}_2,\dots,\mathcal{A}_d)$ and $\pi'(\mathcal{C}_1,\mathcal{C}_2,\dots,\mathcal{C}_d)=(\mathcal{A}_1',\mathcal{A}_2',\dots,\mathcal{A}_d')$. Let $\equiv$ be the equivalence relation generated by $\sim$.

Suppose that $\mathcal{B}_1,\mathcal{A}_1,\mathcal{A}_2,\dots,\mathcal{A}_d\subseteq \PP^2$
are conics. Let $\ell^{(1)},\ell^{(2)}$ be two lines in $\PP^2$ that
are tangent to both $\mathcal{A}_1$ and $\mathcal{B}_1$
with $\mathcal{A}_1,\mathcal{B}_1$ within the same region of $\PP^2\setminus (\ell^{(1)}\cup \ell^{(2)})$ (see Fig.~\ref{fig:CountingConicInvariant2}). 
 \begin{figure}[b!]
\centering
\includegraphics[width=0.5\columnwidth,trim=0in 0in 0in 0in,clip]{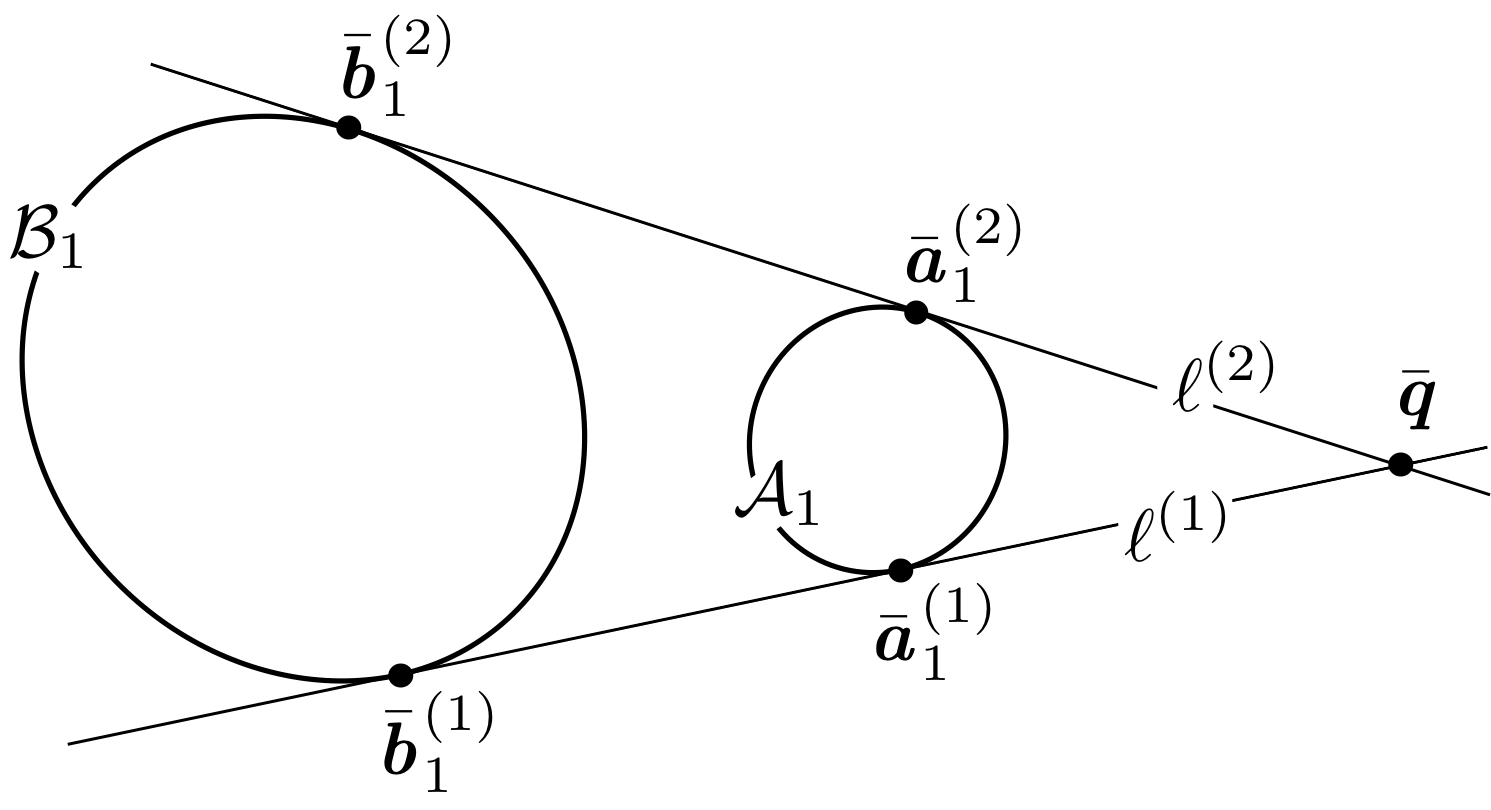}
	\caption{Visualization of geometric quantities used to show the  non-existence of invariants from the projection of arbitrary conics in $\PP^3$. The conics and lines in this graphic all lie within the image plane. The backprojection of these planar quantities into $\PP^3$ form quadric cones (for $\mathcal{A}_1$ and $\mathcal{B}_1$) or  planes (for $\ell^{(1)}$  and $\ell^{(2)}$)}.
	\label{fig:CountingConicInvariant2}
\end{figure}
We denote the point at which $\ell^{(j)}$ is tangent ${\mathcal A}_1$ (respectively ${\mathcal B}_1$) by $\bba_1^{(j)}$ (respectively $\bbb_1^{(j)}$).
Also, let $\bbq$ be the intersection point of $\ell^{(1)}$ and $\ell^{(2)}$. Choose a point $\bbo\in \PP^3$
(the center of the camera) and let $\pi:\PP^3\setminus\{\bbo\}\to \PP^2$ be the camera projection. Let
$\mathcal{X}_i=\pi^{-1}(\mathcal{A}_i)\cup \{\bbo\}$ be the cone with top $\bbo$ that corresponds to ${\mathcal A}_i\subset \PP^2$.
Also, define $\mathcal{Y}_1=\pi^{-1}(\mathcal{B}_1)\cup\{\bbo\}$.
The lines $\ell^{(j)}$  correspond to the plane
$\mathcal{L}^{(j)}=\pi^{-1}(\ell^{(j)})\cup \{\bbo\}$ for $j=1,2$.
The planes $\mathcal{L}^{(1)}$ and ${\mathcal L}^{(2)}$
intersect in the line $\mathcal{Q}=\pi^{-1}(\bbq)\cup\{\bbo\}$, as shown  in Fig.~\ref{fig:ConesInvariant3D}.
We choose another point $\bbo'$ on the line $\mathcal{Q}$ and $\pi':\PP^3\setminus\{\bbo'\}\to \PP^2$ is the camera perspective projection with respect to the camera position $\bbo'$.
We choose some random plane ${\mathcal P}$ in $\PP^3$.
Then $\mathcal{C}_j=\mathcal{P}\cap \mathcal{X}_j$ is a conic in $\PP^3$ and $\pi(\mathcal{C}_j)=\mathcal{A}_j$.
Let $\mathcal{X}_j'$ be the cone with top $\bbo'$ through $\mathcal{C}_j$.
By construction, we have $\pi(\mathcal{C}_1,\mathcal{C}_2,\dots,\mathcal{C}_d)=(\mathcal{A}_1,\mathcal{A}_2,\dots,\mathcal{A}_d)$
and $\pi'(\mathcal{C}_1,\mathcal{C}_2,\dots,\mathcal{C}_d)=(\mathcal{A}_1',\mathcal{A}_2',\dots,\mathcal{A}_d')$ for some conics $\mathcal{A}_1',\dots,\mathcal{A}_d'$ in $\PP^2$. We will construct a conic $\mathcal{D}_1$ in $\PP^3$ with $\pi(\mathcal{D}_1)=\mathcal{B}_1$ and $\pi'(\mathcal{D}_1)=\mathcal{A}_1'$.
Note that $\mathcal{X}_1'$ is a cone with top $\bbo'$ that is
tangent to the planes $\mathcal{L}^{(1)}$ and $\mathcal{L}^{(2)}$.
The intersection $\mathcal{S}^{(j)}=\mathcal{Y}_1\cap \mathcal{L}^{(j)}=\pi^{-1}(\bbb^{(j)}_1)\cup \{\bbo\}$ is a line through $\bbo$,
and
$\mathcal{T}^{(j)}=\mathcal{X}_1'\cap \mathcal{L}^{(j)}$ is a line through $\bbo'$ for $j=1,2$. Let $\bbr^{(j)}$ be the intersection point of the lines $\mathcal{S}^{(j)}$ and $\mathcal{T}^{(j)}$
in the plane $\mathcal{L}^{(j)}$
for $j=1,2$. Then we have $\pi(\bbr^{(j)})=\bbb^{(j)}_1$.
 Moreover
we choose a point $\bbr^{(3)}\neq \bbr^{(1)},\bbr^{(2)}$ that lies on $\mathcal{Y}_1\cap \mathcal{X}_1'$.  Let $\mathcal{R}$ be the plane through $\bbr^{(1)},\bbr^{(2)},\bbr^{(3)}$. Define $\mathcal{D}_1=\mathcal{R}\cap \mathcal{X}_1'$. Then $\mathcal{D}_1$ is tangent to the planes $\mathcal{L}^{(1)}$ and $\mathcal{L}^{(2)}$ at the points $\bbr^{(1)}$ and $\bbr^{(2)}$ respectively. Now $\pi(\mathcal{D}_1)$
and $\mathcal{B}_1$ both are tangent to $\ell^{(1)}$ and $\ell^{(2)}$ at the points
$\bbb^{(1)}=\pi(\bbr^{(1)})$ and $\bbb^{(2)}=\pi(\bbr^{(2)})$ respectively,
and both contain the point $\pi(\bbr^{(3)})$.
It follows that $\pi(\mathcal{D}_1)=\mathcal{B}_1$ because
a conic in $\PP^2$ is determined by $3$ points
and the tangent lines at 2 of the points.
One would expect the intersection $\mathcal{Y}_1\cap \mathcal{X}_1'$ of two quadratic surfaces to be a curve of degree $4$.
In this case, the curve is reducible, and a union of two conics, and $\mathcal{D}_1$ is one of them.
 \begin{figure}[h!]
\centering
\includegraphics[width=0.75\columnwidth,trim=0in 0in 0in 0in,clip]{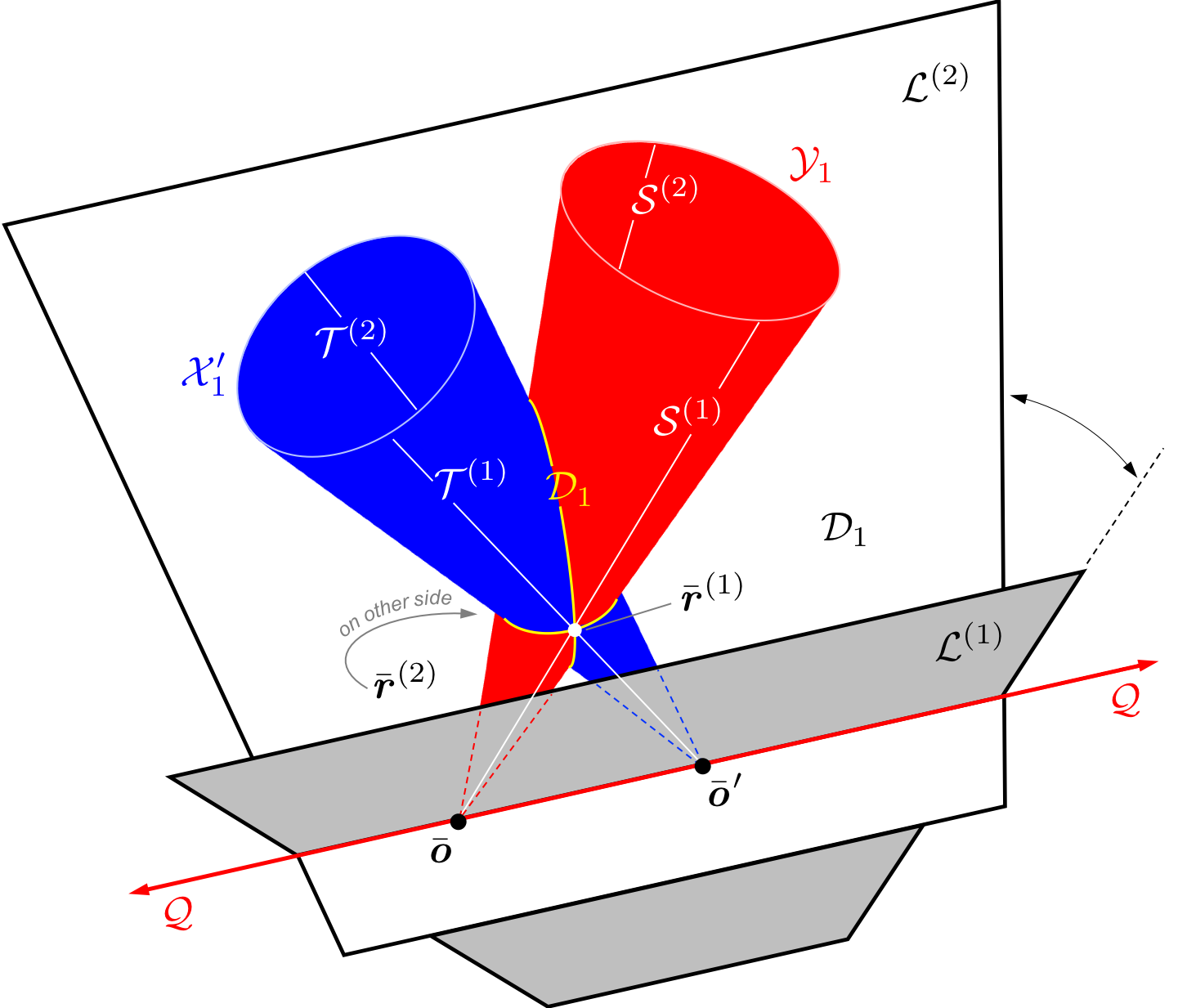}
	\caption{Visualization of geometric quantities used to show the  non-existence of invariants from the projection of arbitrary conics in $\PP^3$. The lines $\mathcal{T}^{(1)}$  and $\mathcal{S}^{(1)}$ lie in plane $\mathcal{L}^{(1)}$ (lines are white, plane is gray). The lines $\mathcal{T}^{(2)}$  and $\mathcal{S}^{(2)}$ lie in plane $\mathcal{L}^{(2)}$.}
	\label{fig:ConesInvariant3D}
\end{figure}

It follows that
$(\mathcal{B}_1,\mathcal{A}_2,\mathcal{A}_3,\dots,\mathcal{A}_d)=\pi(\mathcal{D}_1,\mathcal{C}_2,\dots,\mathcal{C}_d)$
and $(\mathcal{A}_1',\mathcal{A}_2',\dots,\mathcal{A}_d')=\pi'(\mathcal{D}_1,\mathcal{C}_2,\dots,\mathcal{C}_d)$, and therefore $(\mathcal{B}_1,\mathcal{A}_2,\dots,\mathcal{A}_d)\sim (\mathcal{A}_1',\mathcal{A}_2',\dots,\mathcal{A}_d')$. If we combine this with
$(\mathcal{A}_1',\mathcal{A}_2,\dots,\mathcal{A}_d')\sim (\mathcal{A}_1,\mathcal{A}_2,\dots,\mathcal{A}_d)$
then we get $(\mathcal{B}_1,\mathcal{A}_2,\dots,\mathcal{A}_d)\equiv (\mathcal{A}_1,\mathcal{A}_2,\dots,\mathcal{A}_d)$. Repeating the argument shows
that 
$$
(\mathcal{A}_1,\mathcal{A}_2,\dots,\mathcal{A}_d)\equiv (\mathcal{B}_1,\mathcal{B}_2,\dots,\mathcal{B}_d)
$$
for all conics $\mathcal{A}_1,\mathcal{A}_2,\dots,\mathcal{A}_d,\mathcal{B}_1,\mathcal{B}_2,\dots,\mathcal{B}_d$.
This implies that there is no nonconstant rational invariant for $d$ conics in $\PP^3$ that can be computed from a projection.

That no rational invariants exist for the projection of $d$ arbitrary conics is an important finding (which we believe to be a novel result). This clearly precludes the use of invariants for lost-in-space crater identification about an arbitrarily shaped body. Fortunately, celestial bodies large enough to exhibit substantial cratering do not have an arbitrary shape. The surface of these bodies can usually be modeled regionally---if not globally---as a nondegenerate quadric surface. Specifically, we note that planets, moons, and dwarf planets are generally ellipsoidal in global shape \cite{Melosh:2011}  (this is certainly the case for the Moon), while the smaller minor planets are often globally irregular but regionally ellipsoidal. Comets nuclei often show no such structure. Further, many large bodies (e.g., the Moon) appear nearly planar over sufficiently small portions of the surface. We know from before that invariants exist for coplanar conics (Section~\ref{Sec:InvariantsConicsP2}). We now show that invariants also exist for non-coplanar conics lying on a quadratic surface (a quadric).

\subsection{Counting invariants for conics on a nondegenerate quadric surface}
\label{Sec:InvariantsProjQuadSurfConicsP3}
As we have seen, there are no nontrivial invariants for $d$-tuples of arbitrary conics in 3D space that can be computed from a camera image. In the case of craters on the Moon, the conics we are interested in are not arbitrary, but lie on the surface of the Moon. Thus, instead of considering $d$-tuples of arbitrary conics,
we consider the variety $\varV$ of $d$-tuples of conics for which there exists a nondegenerate quadric surface that contains all of them. We find projective invariants to exist for $d \geq 3$ conics lying on the same nondegenerate quadric surface.

\subsubsection{Pairs of conics on a nondegenerate quadric surface}
First consider the case $d=2$. Suppose that $\mathcal{C}_1$ and $\mathcal{C}_2$ are two conics that lie in the planes given by the linear equations $f_1=0$ and $f_2=0$ respectively. Then $\mathcal{C}_1$ and $\mathcal{C}_2$ lie in the degenerate quadric surface $f_1f_2=0$. Suppose both conics also lie on a nondegenerate surface defined by a quadratic equation $g=0$. This surface is not unique, because for every $t$, the surface
defined by $g+t f_1f_2=0$ contains both conics.
If two conics $\mathcal{C}_1$ and $\mathcal{C}_2$ lie on a nondegenerate surface defined by $g=0$,
then $\mathcal{C}_1$ is defined by $f_1=g=0$
and $\mathcal{C}_2$ is defined by $f_2=g=0$.
The intersection $\mathcal{C}_1\cap \mathcal{C}_2$ is defined by $f_1=f_2=g=0$ and consists of two (possibly complex) points. 
On the other hand, if $\mathcal{C}_1\cap \mathcal{C}_2$ consists of two points, one can show that both conics lie on a nondegenerate quadric surface. To find a pair $(\mathcal{C}_1,\mathcal{C}_2)$ in $\varV$, we can choose the conic $\mathcal{C}_1$ arbitrarily in the 8-dimensional variety of conics, we choose the plane $\mathcal{P}_2$ that contains  $\mathcal{C}_2$ in the 3-dimensional space of hyperplanes, Now the conic $\mathcal{C}_2$ is an element of the $5$-dimensional variety of conics in $\mathcal{P}_2$, but there are two constraints  because the conic has to go through the two points of $\mathcal{C}_1\cap \mathcal{P}_2$.
So the dimension of the variety $\varV$ is $8+3+(5-2)=14$. The 15-dimensional group ${\rm PGL}_4$ acts on $\varV$. Let $\varH$ be the connected component of the identity in the stabilizer of $(\mathcal{C}_1,\mathcal{C}_2)\in \varV$. The conics $\mathcal{C}_1$ and $\mathcal{C}_2$ intersect in two (possibly complex) points $\bbc_1$ and $\bbc_2$.
So $g$ will fix the points $\bbc_1$ and $\bbc_2$. If $\mathcal{L}$ is the line through $\bbc_1$ and $\bbc_2$ then $g$ will map $\mathcal{L}$ to itself. If $\bbc_3$ is a 3rd point on $\mathcal{L}$ then $g$ may map $\bbc_3$ to another point of $\mathcal{L}$. However,
if $g$ also fixes the point $\bbc_3$ then a similar argument as in Section~\ref{Sec:ConicInvariantsInP3} shows that $g$ must fix the planes $\mathcal{P}_1$ and $\mathcal{P}_2$ pointwise, and must be the identity. This shows that the stabilizer of the pair $(\mathcal{C}_1,\mathcal{C}_2)$ is at most $1$-dimensional. Since the generic stabilizer has dimension at most $1$, the dimension of a generic orbit is at least $15-1=14$.
Since $\dim \varV=14$, this implies that there are no rational invariants for a pair of conics lying on a nondegenerate quadric surface.

\subsubsection{Many ($d\geq3$) conics on a nondegenerate quadric surface}
Let us now assume that $d\geq 3$. If a quadric surface $S$ contains 3 distinct conics $\mathcal{C}_1,\mathcal{C}_2,\mathcal{C}_3$
then $S$ is the unique quadric surface that contains these three conics.
To see this, suppose that $S'$ is another quadric surface through the 3 conics.
Suppose that $S$ and $S'$
are defined by the equations $h=0$ and $h'=0$ respectively,  where $h$ and $h'$  are homogeneous quadratic polynomials in $4$ variables.
Let $\mathcal{P}_i$ be the plane through $\mathcal{C}_i$ given by $f_i=0$ where $f_i$ is a linear function. Let $\bba\in \PP^3$.
We can multiply $h$ and $h'$ with nonzero scalars such that $(h-h')(\bba)=0$.
In the plane $\mathcal{P}_i$, the restriction
of the quadratic polynomial $h-h'$ vanishes on $\mathcal{C}_i\cup \{\bba\}$. This implies that $h-h'$ is divisible by $f_i$ for $i=1,2,3$.
But then $h-h'$ is divisible by $f_1f_2f_3$
and $h-h'$ must be zero.

To parametrize $d$ conics that lie on a quadric surface, we can first choose the surface that is given by a quadratic polynomial in 4 variables up to a scalar. Such a polynomial has 10 coefficients, so the quadric surface is determined by $10-1=9$ parameters. Now, each of the conics is determined by a hyperplane section of the quadric surface. Hyperplanes are parameterized by points in $\PP^3$. So
the variety $\varV$ of $d$-tuples of conics that lie on a common quadric surface has dimension $9+3d$.
The stabilizer of a generic point in $\varV$ is finite, so the dimension of a generic orbit is $15=\dim {\rm PGL}_4$. So the number of independent rational invariants is $(9+3d)-15=3d-6$ for $d\geq 3$.

\section{Computing Invariants from Crater Rims in an Image}
\label{Sec:ComputingInvariantsTop}
The results of Section~\ref{Sec:ExistenceOfInvariantsTopLevel} established the existence of invariants for conics lying on either a plane or a nondegenerate quadric surface. While the coordinate-free approach used in the previous section is a powerful tool for studying such invariants, we must ultimately impose a coordinate system for practical computation of these invariants from data. Doing so is straightforward, but does require some additional mathematical machinery. The details are now discussed.

\subsection{Geometry of Invariants for Non-Coplanar Conics}
\label{Sec:NonCoplanarInvariants}
\subsubsection{Pair of Arbitrary Non-Coplanar Conics}
\label{Sec:GeomArbitraryNonCoplanarConics}
Two arbitrary 3D conics do not generally intersect one another, even over the complex numbers. This is easy to see through the intersection of surfaces (illustrated in Fig.~\ref{fig:NonCoplanarConicGeneric}). The line formed by the intersection of the two conic planes intersects each conic at two places, producing four intersection points. The cross-ratio of these four colinear points is the unique 3D invariant for a pair of non-coplanar conics. This was observed in \cite{Maybank:1992} and \cite{Quan:1995b}, which we now rederive by other means as we build towards a methods for computing projective invariants. Further, the results of Section~\ref{Sec:ConicInvariantsInP3} tell us that this is the only 3D invariant for a pair of non-coplanar conics. We also show that this 3D invariant does not lead to a useful projective invariant that may be constructed from an image of these two conics.

\begin{figure}[b!]
\centering
\includegraphics[width=0.45\columnwidth,trim=0in 0in 0in 0in,clip]{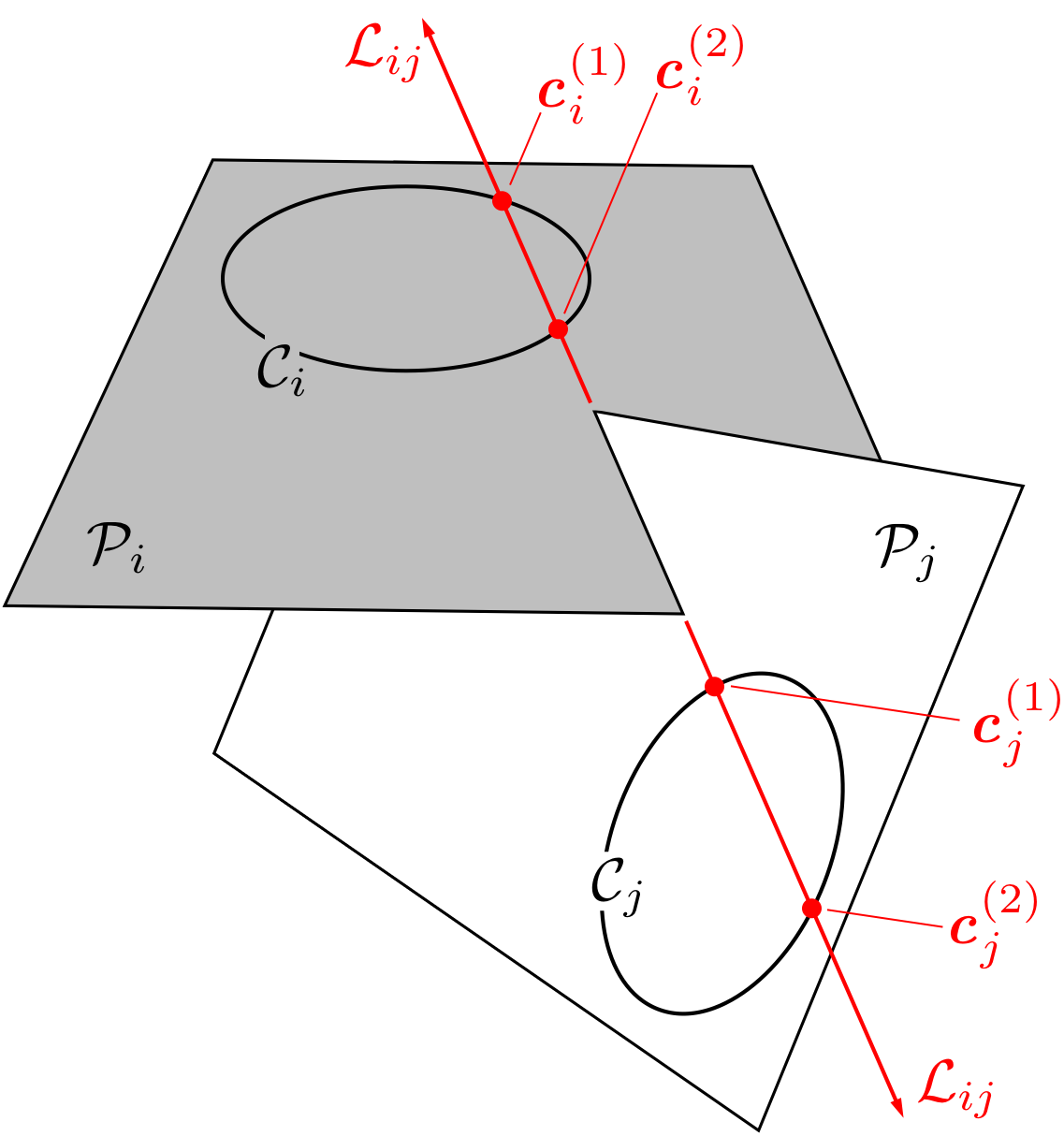}
	\caption{Two arbitrarily placed non-coplanar conics do not intersect one another. The common 3D line $\mathcal{L}_{ij}$ formed by the intersection of their planes intersects each of the two conics in two places, creating four (usually distinct) points. The two intersection points for a particular conic are complex valued if $\mathcal{L}_{ij}$ does not physically intersect that conic. The four colinear intersection points (possibly complex valued) may be used to form a cross-ratio, which is the single unique 3D invariant for a pair of 3D conics.}
	\label{fig:NonCoplanarConicGeneric}
\end{figure}

Suppose we have crater $i$ described by the 3D conic $\mathcal{C}_i$ that lies in plane $\mathcal{P}_i$. Let the Moon-centered quadric cone of crater $i$ be given by $\mathcal{X}_i$ as described in Eq.~\ref{eq:QuadricConeCrateri}. By construction, this cone must pass through the conic locus $\mathcal{C}_i$, and we interpret $\mathcal{C}_i$ as the conic section from Eq.~\ref{eq:ConicLocusConicSection1}.

Now, suppose we have two craters: crater $i$ and crater $j$. We compute their intersection by substitution of Eq.~\ref{eq:ConicLocusConicSection1} as,
\begin{equation}
    \mathcal{C}_i \cap  \mathcal{C}_j = \left( \mathcal{P}_i \cap \mathcal{X}_{i} \right) \cap \left( \mathcal{P}_j \cap \mathcal{X}_{j} \right) = \left( \mathcal{P}_i \cap \mathcal{P}_{j} \right) \cap \mathcal{X}_{i} \cap  \mathcal{X}_{j}
\end{equation}
Define the  3D line $\mathcal{L}_{ij}$ as the intersection of planes $\mathcal{P}_i$ and $\mathcal{P}_j$,
\begin{equation}
    \label{eq:3DLineLij}
    \mathcal{L}_{ij} =  \mathcal{P}_i \cap \mathcal{P}_j
\end{equation}
such that
\begin{equation}
    \mathcal{C}_i \cap  \mathcal{C}_j = \mathcal{L}_{ij} \cap \mathcal{X}_{i} \cap  \mathcal{X}_{j} = \left( \mathcal{L}_{ij}  \cap \mathcal{X}_{i} \right) \cap \left( \mathcal{L}_{ij}  \cap \mathcal{X}_{j} \right)
\end{equation}

The intersection of a line with a quadric cone, e.g. $\mathcal{L}_{ij}  \cap \mathcal{X}_{i}$, produces two points. Define the intersection points  of this line with crater $i$ as $\{\bc^{(1)}_i,\bc^{(2)}_i\}$. It is easy to see, as illustrated in Fig.~\ref{fig:NonCoplanarConicGeneric}, that these intersection  points are not generally the same for two arbitrary conics. Therefore,
\begin{equation}
    \mathcal{C}_i \cap  \mathcal{C}_j = \{\bc^{(1)}_i,\bc^{(2)}_i\} \cap \{\bc^{(1)}_j,\bc^{(2)}_j\} = \varnothing
\end{equation}
Observe that the four points  $\{\bc^{(1)}_i,\bc^{(2)}_i,\bc^{(1)}_j,\bc^{(2)}_j\}$ are colinear since they all lie on the line $\mathcal{L}_{ij}$. They are also distinct, meaning their cross ratio $\rho_{ij}$ is a 3D invariant of the conic pair \cite{Quan:1995b}
\begin{equation}
    \rho_{ij} = \text{Cr}(\bc^{(1)}_i,\bc^{(2)}_i,\bc^{(1)}_j,\bc^{(2)}_j)
\end{equation}

The difficulty with this 3D invariant is that it is not recoverable from a single image of the two conics. The two 3D conics will project into two coplanar conics in the image plane. The image conics have four (often complex valued) intersection points, which are not necessarily colinear. In the absence of other constraints, these intersection points in the image plane are not related to the 3D line $\mathcal{L}_{ij}$  or its intersection points $\{\bc^{(1)}_i,\bc^{(2)}_i,\bc^{(1)}_j,\bc^{(2)}_j\}$ with the two 3D conics. This is because the two conics do not actually intersect one another. That no projective invariant exists from a pair of non-coplanar conics agrees with the findings of Section~\ref{Sec:InvariantsProjArbitraryConicsP3}. The non-existence of projective invariants for three or more arbitrarily placed non-coplanar conics follows by similar arguments.

In the special case where the conics do intersect one another, their intersection is preserved when imaged from an arbitrary pose by a projective camera. One of the most flexible ways to ensure two 3D conics intersect over the complex numbers is to constrain them to lie on a nondegenerate quadric surface (e.g., sphere, ellipsoid). This is now shown.

\subsubsection{Pair of Non-Coplanar Conics on a Nondegenerate Quadric Surface}
\label{Sec:GeomQuadSurfNonCoplanarConics}
Consider craters modeled as conics lying on a nondegenerate quadric surface $\mathcal{S}$. For the case of lunar craters, we  observe that the Moon is very nearly a sphere. Therefore, to a good approximation, define the surface of Moon by the quadric locus parameterized by $\bQ_M$,
\begin{equation}
    \mathcal{S} = \left\{ \bar{\bxi} \in \mathbb{P}^3 \; | \; \bar{\bxi}^T \bQ_M  \bar{\bxi} =  0\right\}
\end{equation}
For a spherical moon of radius $R_M$, 
\begin{equation}
    \bQ_M = \left[ \begin{array}{c c c}
          R^{-2}_M \bI_{3 \times 3} & \textbf{0}_{3 \times 1} \\
          \textbf{0}_{1 \times 3} & -1
   \end{array} \right]
\end{equation}

Now, consider crater $i$ lying on plane $\mathcal{P}_i$.
For the 3D rim of a crater $i$ to also be on the surface of the Moon, it must lie on the intersection of this plane with the $\mathcal{S}$. Therefore, the 3D conic locus for crater $i$ must be the intersection
\begin{equation}
    \label{eq:CraterConicSurfaceIntersect}
    \mathcal{C}_i = \mathcal{P}_i \cap \mathcal{S}
\end{equation}
which produces a circular crater if $\mathcal{S}$ is a sphere. An ellipsoidal (non-spherical) body would generally produce an elliptical crater. The discussion that follows holds for both spherical and ellipsoidal bodies.

Further, as in Section~\ref{Sec:GeomArbitraryNonCoplanarConics} and described in Eq.~\ref{eq:ConicLocusConicSection1}, the crater may also be viewed as the conic section produced by intersecting plane $\mathcal{P}_i$ with the Moon-centered quadric cone $\mathcal{X}_i$. However, since the crater must lie on both $\mathcal{X}_i$ and $\mathcal{S}$, it is clear that we can also write the 3D crater conic as,
\begin{equation}
    \mathcal{C}_i = \mathcal{X}_i \cap \mathcal{S}
\end{equation}
Therefore, we see that 
\begin{equation}
    \mathcal{C}_i = \mathcal{X}_i \cap \mathcal{S} = \mathcal{P}_i \cap \mathcal{S}
\end{equation}

Suppose we have two craters: $\mathcal{C}_i$ and $\mathcal{C}_j$. Because $\mathcal{C}_i$ and $\mathcal{C}_j$ are formed by the intersections of planes $\mathcal{P}_i$ and $\mathcal{P}_j$ with the quadric surface $\mathcal{S}$, we observe that,
\begin{equation}
    \label{eq:ConicIntersectionPlanes}
    \mathcal{C}_i \cap  \mathcal{C}_j =  \left( \mathcal{P}_i \cap \mathcal{S} \right) \cap  \left( \mathcal{P}_j \cap \mathcal{S} \right) = \left( \mathcal{P}_i \cap \mathcal{P}_j \right) \cap \mathcal{S}  =  \mathcal{L}_{ij} \cap \mathcal{S} 
\end{equation}
where $\mathcal{L}_{ij}$ is the 3D line formed by the intersection of the two  crater planes (see Eq.~\ref{eq:3DLineLij}).

If the two craters physically intersect, the line $\mathcal{L}_{ij}$ pierces $\mathcal{S}$ and the intersection points are real. If the craters don't physically intersect, the intersection occurs over the complex numbers. Let $\bar{\bs}^{(1)}_{ij}$ and $\bar{\bs}^{(2)}_{ij}$ be the two intersection points
\begin{equation}
    \label{eq:LineMoonIntersect12}
    \{ \bar{\bs}^{(1)}_{ij}, \bar{\bs}^{(2)}_{ij} \} = \mathcal{L}_{ij} \cap \mathcal{S} 
\end{equation}
Therefore, define $\mathcal{P}_{ij}$ as the plane passing through the center of the Moon and the line $\mathcal{L}_{ij}$ (or, equivalently passing through the center of the Moon, $\bar{\bs}^{(1)}_{ij}$, and $\bar{\bs}^{(2)}_{ij}$). Clearly, the numerical stability of computing $\mathcal{P}_{ij}$ in this manner is poor if the craters are close to one another and $\mathcal{P}_{i}$ and $\mathcal{P}_{j}$ are nearly coplanar. Better numeric stability may be achieved by taking a different approach.

\begin{figure}[b!]
\centering
\includegraphics[width=0.8\columnwidth,trim=0in 0in 0in 0in,clip]{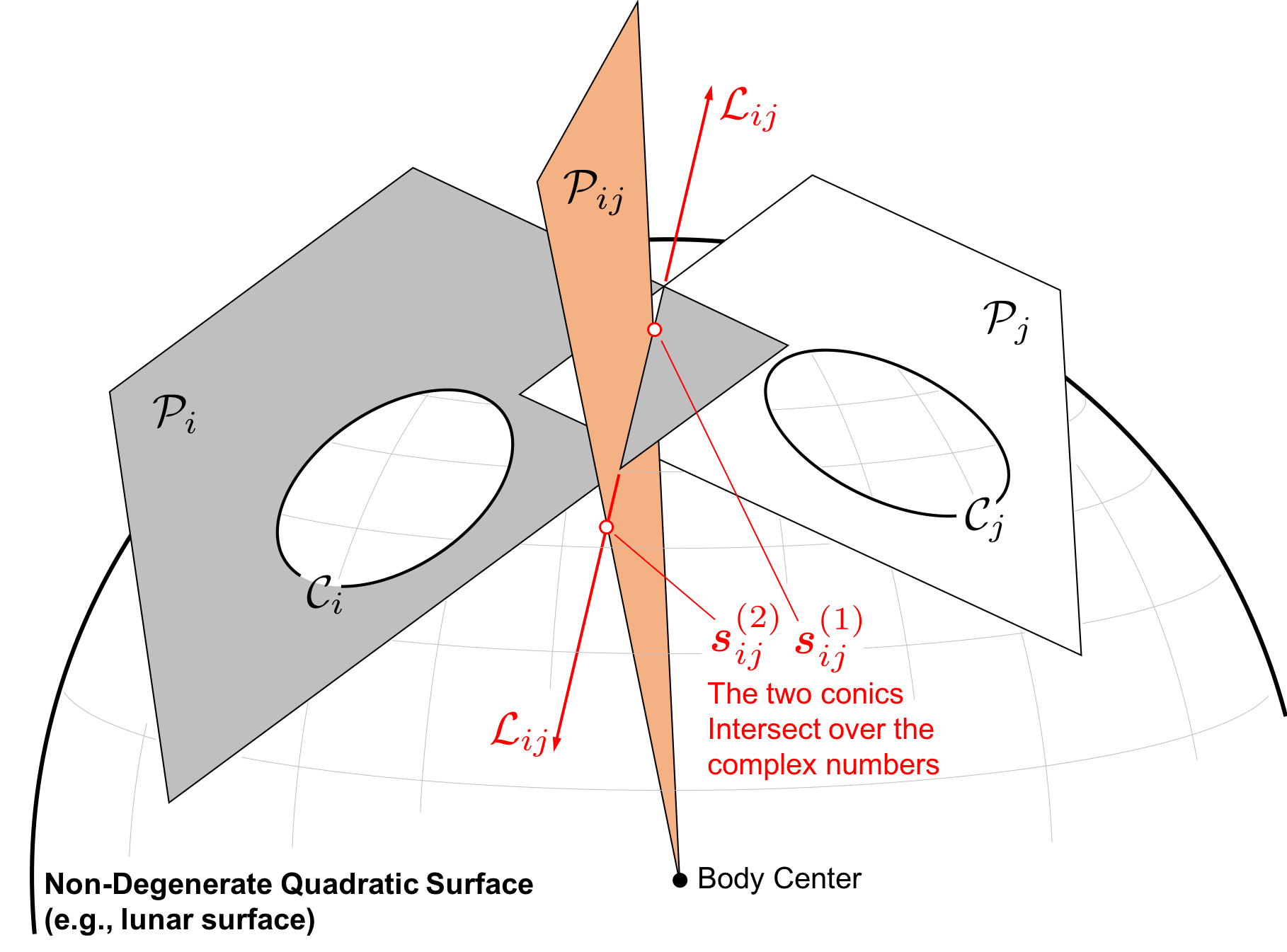}
	\caption{Two non-coplanar conics $\mathcal{C}_{i}$ and $\mathcal{C}_{j}$ lying on a nondegenerate quadric surface must intersect in two points (which happen to be complex valued in this example). The two intersection points lie on the line $\mathcal{L}_{ij}$ formed by the intersection of the two crater planes. The intersection points also lie in the plane $\mathcal{P}_{ij}$ formed by the join of $\mathcal{L}_{ij}$ and the body center.}
	\label{fig:ConicsOnSphere}
\end{figure}

The conic intersection from Eq.~\ref{eq:ConicIntersectionPlanes} may also be written in terms of the quadric cones instead of the planes,
\begin{equation}
    \mathcal{C}_i \cap  \mathcal{C}_j =  \left( \mathcal{X}_{i} \cap \mathcal{S} \right) \cap  \left( \mathcal{X}_{j} \cap \mathcal{S} \right) = \left( \mathcal{X}_{i} \cap \mathcal{X}_{j} \right) \cap \mathcal{S} 
\end{equation}
Without constraints, the intersection of two cones $\mathcal{X}_{i} \cap \mathcal{X}_{j}$ is a set of four lines passing through the center of the Moon that pierce $\mathcal{S}$ at eight different points.  Four of the points will be on the wrong side and come from conics that don't actually exist; these can be ignored. When the conics lie on a nondegenerate quadric surface, we find the other four intersection points are a set of repeated intersections. To see this, combine the above results to find,
\begin{equation}
    \left( \mathcal{X}_{i} \cap \mathcal{X}_{j} \right) \cap \mathcal{S} = \mathcal{C}_i \cap  \mathcal{C}_j = \mathcal{L}_{ij} \cap \mathcal{S} = \{ \bar{\bs}^{(1)}_{ij}, \bar{\bs}^{(2)}_{ij} \}
\end{equation}

The advantage of using the quadric cone intersections to describe $\bar{\bs}^{(1)}_{ij}$ and $\bar{\bs}^{(2)}_{ij}$ is that this remains well defined and more numerically stable as the craters become close and their planes become nearly parallel. This is especially important when viewing craters at a local level, where craters are nearly coplanar.

Since the points $\bar{\bs}^{(1)}_{ij}$ and $\bar{\bs}^{(2)}_{ij}$ lie on $\mathcal{C}_i$, the projection of these points lie on the projection of $\mathcal{C}_i$ (which we called $\mathcal{A}_i$). The same holds for $\mathcal{C}_j$, and its projection $\mathcal{A}_j$. Therefore,the projection of $\bar{\bs}^{(1)}_{ij}$ and $\bar{\bs}^{(2)}_{ij}$ is recoverable as two of the intersection points of the image conics. That is,
\begin{equation}
    \bar{\bu}^{(k)}_{ij} = \bP^M_C \bar{\bs}^{(k)}_{ij} \subset \mathcal{A}_i \cap \mathcal{A}_j, \quad k=1,2
\end{equation}
where $\bP^M_C$ is the projection matrix from Eq.~\ref{eq:ProjectionMatrixDef} and the pixel coordinate $\bar{\bu}^{(k)}_{ij}$ is the projection of $\bar{\bs}^{(k)}_{ij}$ into the image.

The challenge, therefore, is to determine which of the four intersection  points of $\mathcal{A}_i$ and $\mathcal{A}_j$ are the projection of $\bar{\bs}^{(1)}_{ij}$ and $\bar{\bs}^{(2)}_{ij}$. There are six possible combinations of these four points, corresponding to six possibilities for the projection of line $\mathcal{L}_{ij}$ into  the line $\ell_{ij}$. We now discuss how to uniquely compute $\ell_{ij}$.

Consider the two  conics $\mathcal{A}_i$ and $\mathcal{A}_j$ described by the conic locus matrices $\bA_i$ and $\bA_j$. We recall that
\begin{equation}
    \bar{\bu}^T \bA_{i} \bar{\bu} = 0 \text{ and }
    \bar{\bu}^T \bA_{j} \bar{\bu} = 0
\end{equation}
It follows that we can form a pencil of conics parameterized by the scalar ratio $\lambda : \mu$ that also pass through the same four intersection points as $\bA_{i}$  and $\bA_{i}$,
\begin{equation}
    \bar{\bu}^T \left( \lambda \bA_{i} + \mu \bA_{j} \right) \bar{\bu} = 0
\end{equation}
If we force the matrix $\lambda \bA_{i} + \mu \bA_{j}$ to be a degenerate conic,  then it becomes two lines (or a double line) and we may find its intersection with the two conics. We find six possibilities, one of which is the line $\ell_{ij}$ that we seek.

Letting $\mu=1$ for easy computation, we force the conic to be degenerate by setting the determinant to zero,
\begin{equation}
    \text{det}\left( \lambda \bA_{i} + \bA_{j} \right) = 0
\end{equation}
Since both $\bA_{i}$ and $\bA_{j}$ are full rank and well-conditioned, we may rewrite this as
\begin{equation}
    \label{eq:ConicIntersectEigValProb}
    \text{det}\left[ \bA_{j}\left( -\bA_{i} \right)^{-1} - \lambda \bI_{3 \times 3}  \right] = 0
\end{equation}
which is a simple $3 \times 3$ eigenvalue problem. When there is no actual intersection of the two conics in the image, only one of the resulting degenerate conics will produce real valued lines \cite{RichterGebert:2011}. The specific line we seek is $\ell_{ij}$, which is the projection of $\mathcal{L}_{ij}$. As the intersection of two non-coplanar planes, $\mathcal{L}_{ij}$ is real valued by construction. Thus, its projection $\ell_{ij}$ is also real valued. Therefore we always seek the eigenvalue leading to a degenerate conic of real valued lines, discarding the four complex valued line possibilities. It will soon become apparent which eigenvalue to choose.

Therefore, suppose we choose one of the eigenvalues to construct the degenerate conic formed by
\begin{equation}
    \bB_{ij} = \lambda  \bA_{i} + \bA_{j}
\end{equation}
where we observe that $\bB_{ij}$ has rank 2. The remaining steps are now to find the lines described by $\bB_{ij}$, determine if they are real valued, and then determine which one is $\ell_{ij}$. 

A degenerate conic made of two lines, $\bg$ and $\bh$, is defined by the symmetric matrix
\begin{equation}
    \bB_{ij} = \bg \bh^T + \bh \bg^T
\end{equation}
where our task is now to find $\bg$ and $\bh$ given $\bB_{ij}$. There are a variety of ways to accomplish this task. Our specific approach is similar to the framework outlined in \cite{RichterGebert:2011}. Therefore, proceeding in this way, we first recall that the intersection of the two lines in homogeneous coordinates is computed as \cite{Hartley:2003}
\begin{equation}
    \bar{\bz} = \bg \times \bh
\end{equation}
Further, observing that
\begin{equation}
    \left[ \bar{\bz} \times \right] = \bg \bh^T - \bh \bg^T
\end{equation}
we find that that
\begin{equation}
    \bB_{ij} + \left[ \bar{\bz} \times \right] = \left(\bg \bh^T + \bh \bg^T\right) + \left(  \bg \bh^T - \bh \bg^T \right) = 2 \bg \bh^T
\end{equation}
is a rank one matrix with columns proportional to $\bg$ and rows proportional to $\bh$. We must now find $\bz$.

To do this, a quick calculation will confirm that
\begin{equation}
    \bB^{\ast}_{ij} = \left( \bg \bh^T + \bh \bg^T \right)^{\ast} = -\left(\bg \times \bh  \right) \left(\bg \times \bh  \right)^T = - \bar{\bz} \bar{\bz}^T
\end{equation}
Or, in terms of the original ellipses,
\begin{equation}
    \bB^{\ast}_{ij} = \left(  \lambda  \bA_{i} + \bA_{j}  \right)^{\ast} = - \bar{\bz} \bar{\bz}^T
\end{equation}
The diagonals of $-\bB^{\ast}_{ij}$ are the squares of the elements of $\bar{\bz}$. Thus, defining the columns and elements of $\bB^{\ast}_{ij}$ as
\begin{equation}
    \bB^{\ast}_{ij} = \left[ \begin{array}{c c c}
          \bb^{\ast}_1 & \bb^{\ast}_2 & \bb^{\ast}_3
   \end{array} \right] = 
   \left[ \begin{array}{c c c}
          b^{\ast}_{11} & b^{\ast}_{12} & b^{\ast}_{13} \\
          b^{\ast}_{21} & b^{\ast}_{22} & b^{\ast}_{23} \\
          b^{\ast}_{31} & b^{\ast}_{32} & b^{\ast}_{33} 
   \end{array} \right]
\end{equation}
any column of $\bB^{\ast}_{ij}$ may be scaled to find $\bar{\bz}$ according to
\begin{equation}
    \bar{\bz}  = -\bb^{\ast}_k / \sqrt{-b^{\ast}_{kk}}
\end{equation}
where any $k\in\{1,2,3\}$ with $b_{kk} \neq 0$ will work. Best numerical performance is achieved by selecting the value for $k$ that yields $\max_k  \| b^{\ast}_{kk} \|$. If the ellipses do not intersect, only one eigenvalue from Eq.~\ref{eq:ConicIntersectEigValProb} will produce a real valued estimate of $\bz$. The eigenvalue to choose is the one that makes the diagonal of $\bB^{\ast}_{ij}$ negative.

Therefore, with  $\bar{\bz}$ known, compute $\bD$
\begin{equation}
    \bD = \bB_{ij} + \left[ \bar{\bz} \times \right] = 2 \bg \bh^T
\end{equation}
we can then compute $\bg$ from any non-zero column of $\bD$ and $\bh$ from any non-zero row of $\bD$. In general, it is best to find the element of $\bD$ with the largest absolute value and pick the corresponding row and column for $\bg$ and $\bh$.

Given the two lines $\bg$ and $\bh$, one of these corresponds to $\ell_{ij}$. The other does not. Each of the lines $\bg$ and $\bh$ divide $\PP^2$ into two regions. Since the line we seek  comes from the projection of $\mathcal{L}_{ij} = \mathcal{P}_i \cap \mathcal{P}_j$,  the two conics must not lie in the same region. That is, we seek a line that passes between the two conics. Only one line will satisfy this condition and this is the line we choose for $\ell_{ij}$.

If we were to compute the intersection of $\ell_{ij}$ with either $\mathcal{A}_i$ or $\mathcal{A}_j$ we would obtain two points that are the projection $\bar{\bs}^{(1)}_{ij}$ and $\bar{\bs}^{(2)}_{ij}$. Fortunately such a computation is not necessary since we work directly with $\ell_{ij}$ in subsequent discussions.

Briefly, we observe that $\mathcal{C}_i$ and the two points $\bar{\bs}^{(1)}_{ij}$ and $\bar{\bs}^{(2)}_{ij}$ all lie in the plane $\mathcal{P}_i$. While a projective invariant exists for two points and a conic (all coplanar) \cite{Forsyth:1991}, the points must not lie on the conic. Since $\bar{\bs}^{(1)}_{ij}$ and $\bar{\bs}^{(2)}_{ij}$ lie on  $\mathcal{C}_i$ by construction, we seek an alternative way of constructing an invariant.

Consider instead the 2D line $\ell_{ij}$ formed by the projection of 3D line $\mathcal{L}_{ij}$ (or, equivalently, by the join of image points $\bar{\bu}^{(1)}_{ij}$ and $\bar{\bu}^{(2)}_{ij}$). While there is not an invariant for a single line and a conic, there is an invariant for two lines and  a conic \cite{Rothwell:1995}. This motivates the study of invariants for a triad of craters.

\subsubsection{Triad of Non-Coplanar Conics on a Nondegenerate Quadric Surface}
\label{Sec:NonCoplanarTriadInvariants}
Suppose a projective camera observes three craters: $\mathcal{C}_i$, $\mathcal{C}_j$, and $\mathcal{C}_k$. The intersection of the corresponding planes ($\mathcal{P}_i$, $\mathcal{P}_j$, and $\mathcal{P}_k$) produces the three 3D lines $\mathcal{L}_{ij}$, $\mathcal{L}_{ik}$, and $\mathcal{L}_{jk}$. Three planes intersect in a point, which is also the location where the 3D line formed by two planes intersects the third plane. This point, defined as $a_{ijk}$, is the apex of the pyramid formed by the three planes,
\begin{equation}
    a_{ijk} 
    = \mathcal{P}_i \cap \mathcal{P}_j \cap \mathcal{P}_k 
    = \mathcal{L}_{ij} \cap \mathcal{P}_k
    = \mathcal{L}_{ik} \cap \mathcal{P}_j
    = \mathcal{L}_{jk} \cap \mathcal{P}_i
\end{equation}

Now, let the three 3D craters $\mathcal{C}_i$, $\mathcal{C}_j$, and $\mathcal{C}_k$ project to 2D conics $\mathcal{A}_i$, $\mathcal{A}_j$, and $\mathcal{A}_k$ in the image. For any given pair of 3D conics (e.g.,  $\mathcal{C}_i$ and $\mathcal{C}_j$), the 3D line through their intersection points (e.g., $\mathcal{L}_{ij}$) is coplanar with both of these 3D conics and it is possible to recover its projection (e.g., $\ell_{ij}$) from just the projected image conics (e.g., $\mathcal{A}_i$ and $\mathcal{A}_j$). Thus, for a triad of observed craters in an image, we may compute the three lines $\ell_{ij}$, $\ell_{ik}$, and $\ell_{jk}$. Since $\mathcal{C}_i$, $\mathcal{L}_{ij}$, $\mathcal{L}_{ik}$ are all coplanar, an invariant exists in the image using $\mathcal{A}_i$, $\ell_{ij}$, and $\ell_{ik}$. The same is true for craters $\mathcal{C}_j$ and $\mathcal{C}_k$.

It is easy to see that such an invariant exists by use of a cross ratio, and this invariant may be computed directly by use of a Cayley-Klein metric. The usual cross-ratio applies to four points on a line. However, by the duality of points and lines in $\mathbb{P}^2$, we can also form a cross-ratio of four lines passing through a point. Therefore, as shown in Fig.~\ref{fig:ConicLineCrossRatio}, consider a conic  $\mathcal{A}_i$. Let the  two lines $\ell_{ij}$ and $\ell_{ik}$ be described by the $3\times1$ vectors $\bl_{ij}$ and $\bl_{ik}$, and let these two lines intersect at the point $\bar{\bp} = \bl_{ij} \times \bl_{ik}$. Using the pole-polar relation for a conic, we may find the two lines from $\bar{\bp}$ that are tangent to the conic $\mathcal{A}_i$, which we call $\bw_1$ and $\bw_2$. Since these four lines go through the point $\bar{\bp}$, we may form a cross ratio that  is a projective invariant
\begin{equation}
    \label{eq:ConicAndTwoLinesCrossRatio}
    \text{Cr}\left( \bl_{ij}, \bw_1,  \bw_2,  \bl_{ik}   \right)
\end{equation}
\begin{figure}[b!]
\centering
\includegraphics[width=0.5\columnwidth,trim=0in 0in 0in 0in,clip]{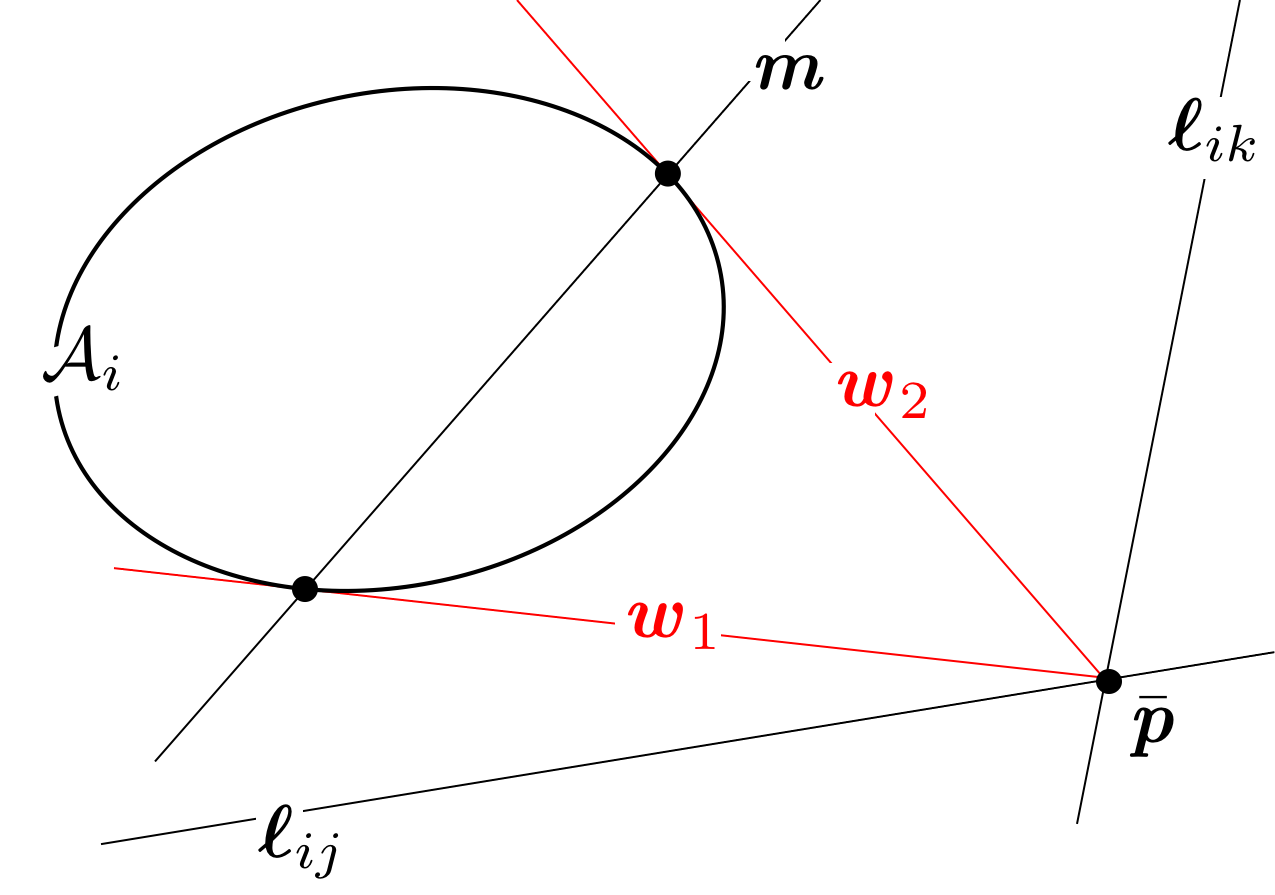}
	\caption{Visualization of contributing components to invariant for a conic $\mathcal{A}_i$ and two lines $\bl_{ij}$ and $\bl_{jk}$. The lines $\bw_1$ and $\bw_2$ pass through $\bp$ and are tangent to $\mathcal{A}_i$. The line $\bbm$ is the polar of point $\bar{\bp}$ with respect to $\mathcal{A}_i$.}
	\label{fig:ConicLineCrossRatio}
\end{figure}
We can now form a classical Cayley-Klein metric as
\begin{equation}
    \label{eq:CayleyKleinCrossRatio}
    d(\bl_{ij}, \bl_{ik}) = \frac{1}{2}\ln\left[ \text{Cr}\left( \bl_{ij}, \bw_1,  \bw_2,  \bl_{ik}   \right) \right]
\end{equation}
which, from hyperbolic geometry, is known to be equivalent to
\begin{equation}
    \label{eq:CayleyKleinToConic}
    \alpha_i = \cosh\left[ d(\bl_{ij}, \bl_{ik}) \right]  =  \frac{  \bl^T_{ij} \bA^{\ast}_{i} \bl_{ik} }{ \sqrt{ \left( \bl^T_{ij} \bA^{\ast}_{i} \bl_{ij} \right) \left( \bl^T_{ik} \bA^{\ast}_{i} \bl_{ik} \right)} }
\end{equation}
where the adjugate matrix $\bA^{\ast}_{i}$ describes the  conic envelope of $\mathcal{A}_i$. We may also write, therefore,
\begin{equation}
    \label{eq:CayleyKleinToConic2}
    \alpha^2_i = \cosh^2 \left[ d(\bl_{ij}, \bl_{ik}) \right]  =   \frac{ \left( \bl^T_{ij} \bA^{\ast}_{i} \bl_{ik} \right)^2 }{ \left( \bl^T_{ij} \bA^{\ast}_{i} \bl_{ij} \right) \left( \bl^T_{ik} \bA^{\ast}_{i} \bl_{ik} \right)}
\end{equation}
Since the cross ratio  from Eq.~\ref{eq:ConicAndTwoLinesCrossRatio} is a projective invariant, it follows immediately that the Cayley-Klein metric from Eq.~\ref{eq:CayleyKleinCrossRatio} and its hyperbolic cosine in Eq.~\ref{eq:CayleyKleinToConic} are also projective invariants.

At first, the Cayley-Klein development may seem of purely academic interest. However, it is of critical importance when building a searchable index for lunar crater identification. It is only through this relation that we fully understand why the metric for $\alpha^2_i$ (the typical invariant found in the literature for a coplanar set of a conic and two lines \cite{Rothwell:1995}) loses descriptiveness for small distances $d(\bl_{ij}, \bl_{ik}) $. The spacecraft navigator will immediately see that this is analogous to how small inter-star angles $\theta$ lose descriptiveness when cataloged as $\cos \theta$ instead of by $\theta$ (i.e., we need to tighten the matching tolerance as $\theta \rightarrow 0$ when using a $\cos \theta$ index). Thus, we prefer to index on $d$ instead of $\alpha^2$.

If we choose to only consider craters that do not physically intersect  (i.e. do not overlap), then we are guaranteed that the cross ratio is positive. This implies that $d \geq 0$ and is real, and that the denominator on  the right-hand side of Eq.~\ref{eq:CayleyKleinToConic} is also real. Since $d\geq 0$, we know that 
\begin{equation}
d = \text{arcosh} \, \alpha_i  = \ln\left[ \alpha_i +\sqrt{\alpha_i^2 - 1} \right]
\end{equation}
Therefore, we may relate the cross-ratio, the Cayley-Klein metric $(\bl_{ij}, \bl_{ik})$, and the invariant $\alpha_i$ according to
\begin{equation}
\ln\left[ \alpha_i +\sqrt{\alpha_i^2 - 1} \right] = d(\bl_{ij}, \bl_{ik}) = \frac{1}{2}\ln\left[ \text{Cr}\left( \bl_{ij}, \bw_1,  \bw_2,  \bl_{ik}   \right) \right] 
\end{equation}
where we choose to construct the invariant to index with the Cayley-Klein metric,
\begin{equation}
    \label{eq:DefJi}
    J_i = d(\bl_{ij}, \bl_{ik}) = \text{acosh}  \left\{ 
    \frac{  \bl^T_{ij} \bA^{\ast}_{i} \bl_{ik} }{ \sqrt{ \left( \bl^T_{ij} \bA^{\ast}_{i} \bl_{ij} \right) \left( \bl^T_{ik} \bA^{\ast}_{i} \bl_{ik} \right)} } \right\} 
\end{equation}
An identical procedure produces corresponding invariants for craters $j$ and $k$
\begin{equation}
    \label{eq:DefJj}
    J_j = d(\bl_{ij}, \bl_{jk}) = \text{acosh}  \left\{ 
    \frac{  \bl^T_{ij} \bA^{\ast}_{j} \bl_{jk} }{ \sqrt{ \left( \bl^T_{ij} \bA^{\ast}_{j} \bl_{ij} \right) \left( \bl^T_{jk} \bA^{\ast}_{j} \bl_{jk} \right)} } \right\} 
\end{equation}
\begin{equation}
    \label{eq:DefJk}
    J_k = d(\bl_{ik}, \bl_{jk}) = \text{acosh}  \left\{ 
    \frac{  \bl^T_{ik} \bA^{\ast}_{k} \bl_{jk} }{ \sqrt{ \left( \bl^T_{ik} \bA^{\ast}_{k} \bl_{ik} \right) \left( \bl^T_{jk} \bA^{\ast}_{k} \bl_{jk} \right)} } \right\} 
\end{equation}

We know from our review of the literature that many previous attempts at crater identification have accidentally employed feature descriptors with elements that are not independent. Thus, before proceeding further, it is necessary to ensure that the invariants $J_i$, $J_j$, and $J_k$ are independent. Since we know from Section~\ref{Sec:InvariantsProjQuadSurfConicsP3} that a $d$-tuple of conics on a quadric surface has $3d-6$ invariants (leading to 9-6=3 invariants for this case), it follows that if the three invariants $J_i$, $J_j$, and $J_k$ are independent then we have found all the algebraically independent invariants that exist. This is now shown.

We observe that the invariants $\alpha_i^2,\alpha_j^2,\alpha_k^2$ (and hence also $\alpha_i,\alpha_j,\alpha_k$ and $J_i$, $J_j$, $J_k$) are algebraically independent. We calculate these invariants for a particular model, since it suffices to show that the invariants are independent when restricted to special subclasses of models. We assume that the lunar surface is the sphere $x^2+y^2+z^2=1$,
and that the three craters are the circles defined by intersecting the sphere with the planes given by $x=t_1$, $y=t_2$ and $z=t_3$ respectively. Within the plane $z=t_3$ we have the crater given by $x^2+y^2=1-t_3^2$ and the lines $x=t_1$ and $y=t_2$. So we get
\begin{equation}
\bA_3=\begin{pmatrix} 
1 & 0 & 0\\
0 & 1 & 0\\
0 & 0 & t_3^2-1\end{pmatrix}\mbox{ and } \bA^*_3=\begin{pmatrix} 
t_3^2-1 & 0 & 0\\
0 & t_3^2-1 & 0\\
0 & 0 & 1\end{pmatrix}
\end{equation}
\begin{equation}
\ell_{13}=\begin{pmatrix}
1\\0\\-t_1
\end{pmatrix}
\mbox{ and }
\ell_{23}=\begin{pmatrix}
0\\
1\\
-t_2
\end{pmatrix}
\end{equation}
Thus, we may compute $\alpha_3^2$
\begin{equation}
\alpha_3^2=\frac{(\ell_{13}\bA_3^*\ell_{23})^2}{(\ell^T_{12}\bA_3^*\ell_{12})(\ell^T_{13}\bA_3^*\ell_{13})}=\frac{t_1^2t_2^2}{(t_1^2+t_3^2-1)(t_2^2+t_3^2-1)}
\end{equation}
and, by symmetry, we also find
\begin{equation}
\alpha_{2}^2=\frac{t_1^2t_3^2}{(t_1^2+t_2^2-1)(t_2^2+t_3^2-1)}
\end{equation}
\begin{equation}
\alpha_1^2=\frac{t_2^2t_3^2}{(t_1^2+t_2^2-1)(t_1^2+t_3^2-1)}
\end{equation}
It is now easy to verify that the Jacobi matrix 
\begin{equation}
J(\alpha_1^2,\alpha_2^2,\alpha_3^2)=\Big(\frac{\partial \alpha_i^2}{\partial t_j}\Big)_{1\leq i,j\leq 3}
\end{equation}
is invertible for some choice of $t_1,t_2,t_3$ (e.g., $t_1=t_2=t_3=1/2$). This implies that $\alpha_1^2,\alpha_2^2,\alpha_3^2$ (and, therefore, $J_i,J_j,J_k$) are algebraically independent.

The utility of this framework is now briefly demonstrated. Consider a regional crater pattern such as the one shown in Fig.~\ref{fig:InvariantProjectionPANGU}, where there is notable curvature of the Moon. This example has about 18.1 deg between the surface normal of crater $\mathcal{A}_j$ and $\mathcal{A}_k$,  thus necessitating the consideration of non-coplanar invariants. Two substantially different views of the same crater pattern are shown. The image on the left is nearly nadir pointing and the image on  the right is pointed over 30 deg off nadir. As can be seen, the non-coplanar invariants $J_1$, $J_2$, and $J_3$ are identical in both images since they are formal projective invariants. This makes evident the power of such invariants for crater pattern recognition.

\begin{figure}[b!]
\centering
\includegraphics[width=1\columnwidth,trim=0in 0in 0in 0in,clip]{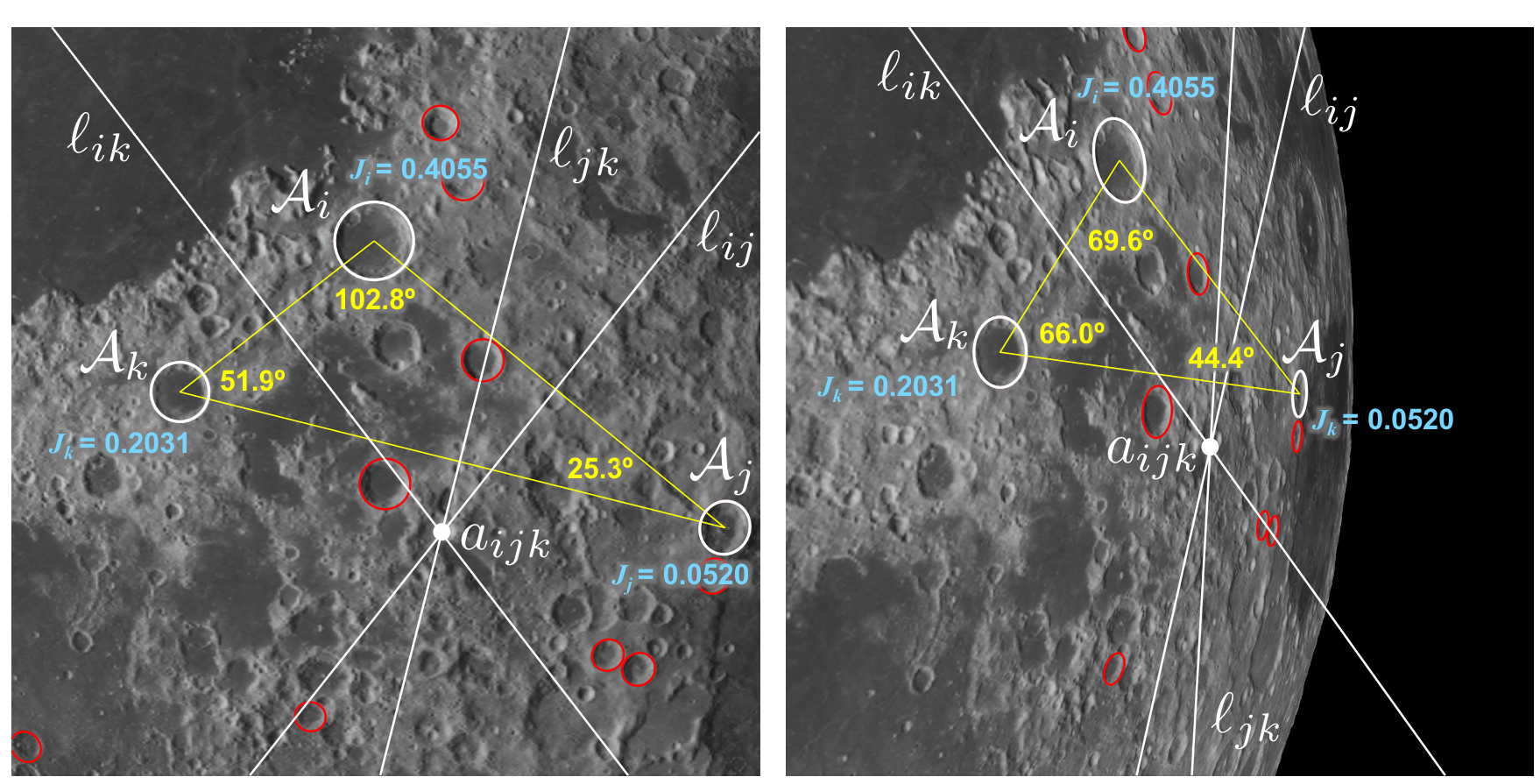}
	\caption{Example of invariants $J_i$, $J_j$, and $J_k$ (blue) for the same triad of craters ($\mathcal{A}_i$, $\mathcal{A}_j$, and $\mathcal{A}_k$) seen from two different vantage points. The invariants remain unchanged under perspective projection from any viewing geometry. The triangle interior angles (yellow) are not a projective invariant. Synthetic images of the lunar surface were produced using PANGU Planet Surface Simulation Software developed by the Space Technology Centre at the University of Dundee, Scotland \cite{Parkes:2009}.}
	\label{fig:InvariantProjectionPANGU}
\end{figure}

It should be stressed that only specific items in Fig.~\ref{fig:InvariantProjectionPANGU} are coplanar, with the vast majority of items being non-coplanar. Consequently general combinations of lines or crater ellipses from left image cannot be transformed to the right image with a common homography.

Also shown in Fig.~\ref{fig:InvariantProjectionPANGU} are the crater triad interior angles (yellow). We highlight these here because crater triangle interior angles are often proposed as descriptors for lost-in-space crater pattern recognition, e.g. \cite{Hanak:2010}. It is clear, however, that the interior angles are not projective invariants and do not provide a robust means of pattern indexing if off-nadir viewing is possible.

\subsection{Geometry of Invariants for Coplanar Conics}
\label{Sec:CoplanarInvariants}
When viewing only a very small portion of the lunar surface, the observed craters are often nearly coplanar.  Smaller craters are also more elliptical. Thus, patterns of small craters close to one another are better described using a triad of coplanar craters. Here, we revisit the problem of invariants of two and three coplanar conics using the framework of Semple and Kneebone and \cite{Semple:1952} and of Quan \cite{Quan:1992,Quan:1998}

\subsubsection{Pairs of Coplanar Conics}
A pair of coplanar conics has two projective invariants \cite{Forsyth:1991}. We saw in Section~\ref{Sec:GeomArbitraryNonCoplanarConics} that two arbitrarily placed 3D conics have no intersections and in Section~\ref{Sec:GeomQuadSurfNonCoplanarConics} that two 3D conics on a nondegenerate quadric surface have two intersections. We observe now that two coplanar conics always have four intersections. These four intersection points (possibly over the complex numbers, possibly repeated) persist under the action of a projective camera, which permits easy computation of two invariants.

Four specified points on a conic may be used to construct a cross-ratio, which is a projective invariant. It may initially be unclear why this is the case since the four points around the ellipse are not collinear, thus a brief explanation is warranted. As discussed earlier, using the duality of points and lines in $\mathbb{P}^2$, we can also form a projective invariant by the cross-ratio of four lines passing through a point. Chasles' theorem \cite{Semple:1952} states that the cross-ratio is a constant for the pencil of four lines from four points on a conic to a fifth point also on the conic (Fig.~\ref{fig:ChaslesTheorem}). Therefore, we may directly form two invariants: the first by finding the cross ratio of the lines from the four intersection points to any other point on first conic, and the second by finding the cross ratio of the lines from the same four intersection points to any other point on the second conic. 

\begin{figure}[b!]
\centering
\includegraphics[width=0.5\columnwidth,trim=0in 0in 0in 0in,clip]{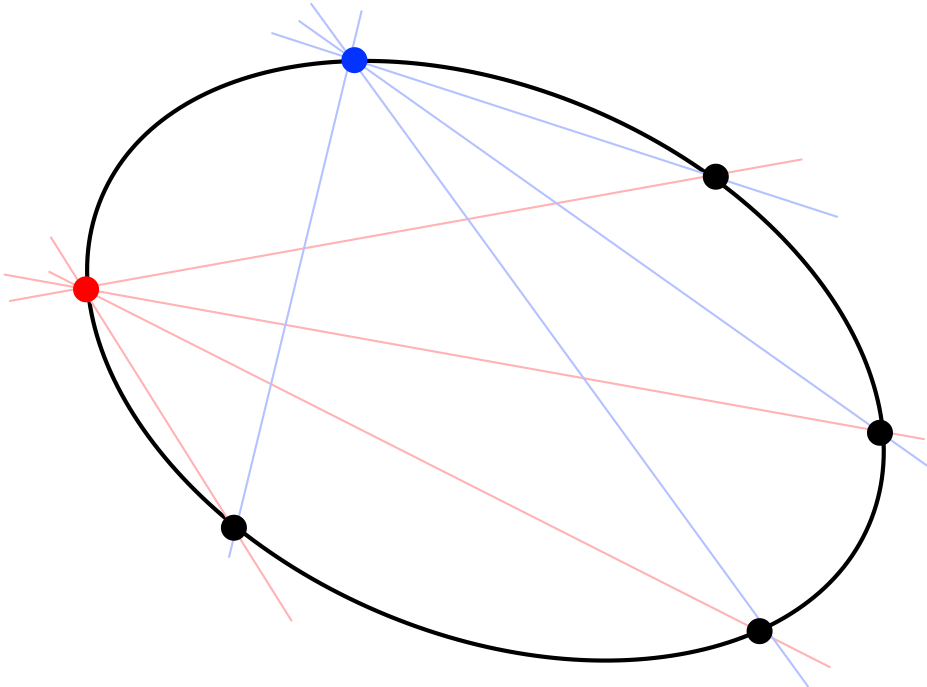}
	\caption{Chasles' theorem states that the cross-ratio formed by the family of red lines is the same as the cross-ratio  of the family of blue lines.}
	\label{fig:ChaslesTheorem}
\end{figure}

Therefore, let $\{\bar{\ba}_k\}_{k=1}^4 = \mathcal{A}_i \cap \mathcal{A}_j$ be the four intersection points of the conics in the image plane and let $\bar{\bp}$ be an arbitrary point on conic $\mathcal{A}_i$. Since the points are written in homogeneous coordinates, we may form the line joining $\bar{\bp}$ and $\bar{\ba}_k$ as $\bw_k = \bar{\bp} \times \bar{\ba}_k$. Thus, the cross ratio $\text{Cr}\left(\bw_1,\bw_2,\bw_3,\bw_4 \right)$ for ellipse $\mathcal{A}_i$ is a projective invariant, and is the same for any choice of $\bp$ on the ellipse.
A second invariant may be computed in  exactly the same way by picking an arbitrary point  on ellipse $\mathcal{A}_j$. This cross-ratio may also be rewritten in terms of the eigenvalues of the matrix $\bA_i^{-1}\bA_j$ \cite{Maybank:1992}.

Explicit computation of these cross ratios is not efficient, so an alternate technique is more appropriate---though the two alternate invariants are rational functions of the cross ratios.

A more conveniently computable form of the invariants for a pair of coplanar conics may be found by direct analysis of the matrices $\bA_i$ and $\bA_j$ describing the image conics. To see this, consider the pencil of conics $\lambda \bA_i + \mu \bA_j$ parameterized by the scalars $\lambda$ and $\mu$. Recall now that 3D crater $\mathcal{C}_i$ projects to image ellipse $\mathcal{A}_i$ according to the homography from Eq.~\ref{eq:HomographyPointConicWithScale}.  Since the craters are assumed coplanar their appearance in an image is related by a common homography,
\begin{equation}
   \bH_{C}^T \left( \lambda \bA_i + \mu \bA_j \right) \bH_{C}  = s_i \lambda \bC_i + s_j \mu \bC_j
\end{equation}
Taking the determinant yields
\begin{equation}
   | \bH_{C} |^2 \, \left| \lambda \bA_i + \mu \bA_j \right| = | s_i \lambda \bA'_i + s_j \mu \bA'_j |
\end{equation}
which may be expanded as
\begin{align}
   | \bH_{C} |^2 \, |\left| \lambda \bA_i + \mu \bA_j \right| & = | \bH_{C} |^2 \, \left( \Theta_1 \lambda^3 + \Theta_2 \lambda^2 \mu + \Theta_3 \lambda \mu^2 + \Theta_4 \mu^3 \right) \nonumber \\
   & =  \Theta'_1 \lambda^3 + \Theta'_2 \lambda^2 \mu + \Theta'_3 \lambda \mu^2 + \Theta'_4 \mu^3 \\
   & = | s_i \lambda \bA'_i + s_j \mu \bA'_j | \nonumber
\end{align}
Thus we find that 
\begin{align}
    \label{eq:Theta1and4constraint}
   | \bH_{C} |^2 \Theta_1 = \Theta'_1 \quad  | \bH_{C} |^2 \Theta_4 = \Theta'_4
\end{align}
\begin{align}
   | \bH_{C} |^2 \Theta_2 = \Theta'_2 \quad | \bH_{C} |^2 \Theta_3 =\Theta'_3
\end{align}
The coefficients ${ \Theta_1,\Theta_2,\Theta_3,\Theta_4}$ are only \emph{semi-invariants} since their numerical value changes with $|\bH_{C}|$ and the arbitrary choices of $s_i$  and $s_j$. These may easily be turned into \emph{rational invariants} by considering appropriate ratios. To see this, first express the coefficients in terms of $\bA_i$ and $\bA_j$,
\begin{equation}
    \label{eq:Theta1and4}
   \Theta_1 = \left| \bA_i \right| \quad \quad
   \Theta_4 = \left| \bA_j \right|
\end{equation}
\begin{equation}
   \Theta_2 = \text{Tr}\left[ \bA^{\ast}_i \bA_j \right] = 
    \left| \bA_i \right| \, \text{Tr}\left[ \bA^{-1}_i \bA_j \right] = \Theta_1 \, \text{Tr}\left[ \bA^{-1}_i \bA_j \right]
\end{equation}
\begin{equation}
   \Theta_3 = \text{Tr}\left[ \bA^{\ast}_j \bA_i \right]= 
    \left| \bA_j \right| \, \text{Tr}\left[ \bA^{-1}_j \bA_i \right] = \Theta_4 \, \text{Tr}\left[ \bA^{-1}_j \bA_i \right]
\end{equation}
Likewise, we find that
\begin{equation}
    \label{eq:ThetaPrime1and4}
   \Theta'_1 = s^3_i \left| \bC_i \right|  \quad \quad \Theta'_4 = s^3_j \left| \bC_j \right|
\end{equation}
\begin{equation}
   \Theta'_2 = \frac{s_j}{s_i} \Theta'_1 \, \text{Tr}\left[ \bC_i^{-1} \bC_j \right]
\end{equation}
\begin{equation}
   \Theta'_3 = \frac{s_i}{s_j} \Theta'_4 \, \text{Tr}\left[ \bC_j^{-1} \bC_i \right]
\end{equation}
As shown in \cite{Semple:1952}, the simplest pair of independent rational invariants are
\begin{equation}
   \frac{\Theta_1 \Theta_3}{\Theta_2^2} = \frac{\Theta'_1 \Theta'_3}{{\Theta'_2}^2}
\quad \text{and}  \quad
   \frac{\Theta_2 \Theta_4}{\Theta_3^2} = \frac{\Theta'_2 \Theta'_4}{{\Theta'_3}^2}
\end{equation}
which we quickly verify produce a pair of non-trivial scalars that are not dependent on $\bH_C$, $s_i$,  or $s_j$ (hence, they are the two fully generic invariants for a pair of conics):
\begin{equation}
   \frac{|\bA_j|}{|\bA_i|} \frac{\text{Tr}\left[ \bA_j^{-1} \bA_i \right]}{\left(\text{Tr}\left[ \bA_i^{-1} \bA_j \right]\right)^2}
   =\frac{\Theta_1 \Theta_3}{\Theta_2^2} = \frac{\Theta'_1 \Theta'_3}{{\Theta'_2}^2} = \frac{|\bC_j|}{|\bC_i|} \frac{\text{Tr}\left[ \bC_j^{-1} \bC_i \right]}{\left(\text{Tr}\left[ \bC_i^{-1} \bC_j \right]\right)^2}
\end{equation}
\begin{equation}
   \frac{|\bA_i|}{|\bA_j|} \frac{\text{Tr}\left[ \bA_i^{-1} \bA_j \right]}{\left(\text{Tr}\left[ \bA_j^{-1} \bA_i \right]\right)^2}
   =\frac{\Theta_2 \Theta_4}{\Theta_3^2} = \frac{\Theta'_2 \Theta'_4}{{\Theta'_3}^2}= \frac{|\bC_i|}{|\bC_j|} \frac{\text{Tr}\left[ \bC_i^{-1} \bC_j \right]}{\left(\text{Tr}\left[ \bC_j^{-1} \bC_i \right]\right)^2}
\end{equation}
Since the left-hand side of these expressions are independent of $\bH_C$, the same value is computed for any arbitrary projection of the conics $\mathcal{C}_i$ and $\mathcal{C}_j$. These results are in agreement with \cite{Quan:1992}, which also follows \cite{Semple:1952}.

Following the convention of \cite{Forsyth:1991}, we now observe that the equations simplify considerably by choosing the scale such that $|\bA_i|=|\bA_j|=1$, making the coefficients $\Theta_1$ and $\Theta_4$ trivial. Using this simplification, substitute Eq.~\ref{eq:Theta1and4} and \ref{eq:ThetaPrime1and4} into Eq.~\ref{eq:Theta1and4constraint} to find
\begin{align}
   | \bH |^2 = s_i^3 = s_j^3
\end{align}
which, in turn, leads to
\begin{equation}
   \Theta_2 = \Theta'_2 \quad \rightarrow \quad \text{Tr}\left[ \bA^{-1}_i \bA_j \right]  = \text{Tr}\left[ \bC_i^{-1} \bC_j \right]
\end{equation}   
\begin{equation}
   \Theta_3 = \Theta'_3 \quad \rightarrow \quad \text{Tr}\left[ \bA^{-1}_j \bA_i \right]  = \text{Tr}\left[ \bC_j^{-1} \bC_i \right]
\end{equation}  

Therefore, when conics are pre-scaled to $|\bA_i|=|\bA_j|=1$, the two unique invariants are simply
\begin{equation}
   I_{ij} = \Theta_2 = \text{Tr}\left[ \bA^{-1}_i \bA_j \right]
\end{equation}
\begin{equation}
   I_{ji} = \Theta_3 = \text{Tr}\left[ \bA^{-1}_j \bA_i \right]
\end{equation}
We note that \cite{Forsyth:1991} follows a completely different approach to finding $I_{ij}$ and $I_{ji} $, and even more geometrically-inspired derivations may be found in \cite{Mundy:1992b}. The specific approach shown here is chosen because of its extensibility to finding invariants of three or more coplanar conics.

These two invariants for a pair of coplanar conics were used by Cheng, et al., for crater identification in \cite{Cheng:2003,Cheng:2005}. We find, however, that better performance is often achieved by considering a triad of craters.

\subsubsection{Triads of Coplanar Conics}
\label{Sec:TriadOfCoplanarConics}

It is established in Section~\ref{Sec:InvariantsConicsP2} that there exist $5d-8$ projective invariants for a $d$-tuple of coplanar  conics ($d \geq 2$). Thus, for a triad of coplanar conics, there exists $15-8 = 7$ projective invariants. This fact has been known for some time  \cite{Quan:1998,Heisterkamp:1997}. 

The procedure for finding these seven invariants follows the same framework used for a pair of coplanar conics. Therefore, consider the determinant of a \emph{net} of three conics \cite{Semple:1952}. As observed in \cite{Quan:1998}, this evaluates to
\begin{align}
    \label{eq:TriadCharEqn}
    | \lambda \bA_i + \mu \bA_j + \sigma \bA_k | = &
    \Theta_1 \lambda^3 +
    \Theta_2 \lambda^2 \mu +
    \Theta_3 \lambda \mu^2 +
    \Theta_4 \mu^3 + 
    \Theta_5 \lambda^2 \sigma \\
    & + \Theta_6 \lambda \sigma^2 +
    \Theta_7 \sigma^3 + 
    \Theta_8 \mu^2 \sigma +
    \Theta_9 \mu \sigma^2 +
    \Theta_{10} \lambda \mu \sigma \nonumber
\end{align}
We immediately see that the coefficients of the first nine terms correspond to the pair-wise combinations of the three conics, with results that follow directly from the two-conic case
\begin{equation}
   \Theta_1 = \left| \bA_i \right|, \quad 
   \Theta_4 = \left| \bA_j \right|, \quad 
   \Theta_7 = \left| \bA_k \right| 
\end{equation}
and that
\begin{equation}
   \Theta_2 = \Theta_1 \, \text{Tr}\left[ \bA^{-1}_i \bA_j \right], \quad
   \Theta_3 = \Theta_4 \, \text{Tr}\left[ \bA^{-1}_j \bA_i \right]
\end{equation}
\begin{equation}
   \Theta_5 = \Theta_1 \, \text{Tr}\left[ \bA^{-1}_i \bA_k \right], \quad
   \Theta_6 = \Theta_7 \, \text{Tr}\left[ \bA^{-1}_k \bA_i \right]
\end{equation}
\begin{equation}
   \Theta_8 = \Theta_4 \, \text{Tr}\left[ \bA^{-1}_j \bA_k \right], \quad
   \Theta_9 = \Theta_7 \, \text{Tr}\left[ \bA^{-1}_k \bA_j \right]
\end{equation}
Note that $\Theta_{10}$ is the only coefficient that simultaneously depends on all three of the conics, thus making it the only term unique to a triad (and not to a pair). This may be computed as
\begin{equation}
   \Theta_{10} = \frac{1}{2}\text{Tr}\left\{ \left[ \left( \bA_j + \bA_k \right)^{\ast} - \left( \bA_j - \bA_k \right)^{\ast} \right] \bA_i \right\}
\end{equation}
where the reader is briefly reminded that $\bA^{\ast}$ is the adjugate of $\bA$. Computation of this result (which, interestingly,  does not appear in \cite{Quan:1998}, \cite{Heisterkamp:1997}, or any other reference discussing triads of coplanar conics) is tedious but straightforward, and is left as an exercise to the reader.

The ten coefficients of Eq.~\ref{eq:TriadCharEqn} may be used to define the seven unique invariants. As with the  pair of coplanar conics, the simplest approach is to choose the arbitrary scale of the matrices $\bA_i$, $\bA_j$, and $\bA_k$ such that $|\bA_i|=|\bA_j|=|\bA_k|=1$. Thus the coefficients $\Theta_1$, $\Theta_4$, and $\Theta_7$ become trivial, and the remaining seven coefficients become the unique invariants. That is, with  $|\bA_i|=|\bA_j|=|\bA_k|=1$, we find  that
\begin{equation}
    \label{eq:DefI1andI2}
   I_{ij} = \text{Tr}\left[ \bA^{-1}_i \bA_j \right]
    \quad 
   I_{ji} = \text{Tr}\left[ \bA^{-1}_j \bA_i \right]
\end{equation}
\begin{equation}
\label{eq:DefI3andI4}
   I_{ik} = \text{Tr}\left[ \bA^{-1}_i \bA_k \right]
    \quad 
   I_{ki} = \text{Tr}\left[ \bA^{-1}_k \bA_i \right]
\end{equation}
\begin{equation}
\label{eq:DefI5andI6}
   I_{jk} = \text{Tr}\left[ \bA^{-1}_j \bA_k \right]
    \quad 
   I_{kj} = \text{Tr}\left[ \bA^{-1}_k \bA_j \right]
\end{equation}
\begin{equation}
\label{eq:DefI7}
   I_{ijk} = \text{Tr}\left\{ \left[ \left( \bA_j + \bA_k \right)^{\ast} - \left( \bA_j - \bA_k \right)^{\ast} \right] \bA_i \right\}
\end{equation}
Therefore, if  three coplanar craters $\mathcal{C}_i$, $\mathcal{C}_j$, and $\mathcal{C}_k$ project into three image conics $\mathcal{A}_i$, $\mathcal{A}_j$, and $\mathcal{A}_k$ (described by the matrices $\bA_i$, $\bA_j$, and $\bA_k$),  then the seven scalar values in Eqs.~\ref{eq:DefI1andI2} to \ref{eq:DefI7} will remain exactly the same under perspective projection regardless of the camera position and attitude.

\section{Creating a Crater Pattern Descriptor from Invariants}
\label{Sec:PatternDescriptor}
Having identified projective invariants that may be computed from the  observed crater rim geometry in an image (three for a triad of non-coplanar conics on a quadric surface and seven for a triad of coplanar conics), we wish  to use this information to construct a pattern descriptor. This descriptor must have a structure that can be consistently reproduced to allow for matching against a static index, which requires some care since the invariants from Section~\ref{Sec:ComputingInvariantsTop} depend on the ordering of the observed craters. We note that the sensitivity of invariants to ordering is not inherently a bad thing (after all, we must specifically match each individual crater to the catalog)---though it does result in a few different methods for building descriptors.

The index is built using descriptors for crater triads. We must choose, therefore, whether we want to impose ordering of the three craters within the triad before or after matching to the index. If we choose the former, then the invariants developed before may be used directly. If we choose the latter, we must either  sort the projective invariants from Section~\ref{Sec:ComputingInvariantsTop} into a prescribed order or transform them  into \emph{projective and permutation} ($p^2$) invariants. These conventions each have their advantages, which we now discuss.

\subsection{Crater Pattern Descriptors with Projective Invariants}
\label{Sec:ProjInvDescriptor}
Suppose we observe a triad of craters that we want to match to a precomputed  index in an order-dependent fashion. Since the crater pattern must be observed from  above (i.e., the camera must be looking downwards from above the lunar surface because it cannot be inside the Moon), there are three possible orderings of the craters. If we choose to arrange these in a clockwise fashion within the image (see Fig.~\ref{fig:CraterPermutation}) then the three possibilities are $(1,2,3)$, $(3,1,2)$, and $(2,3,1)$. For a match to be successful, the observed descriptor must match the index descriptor. Therefore, if the descriptor entries are order dependent, then a match only occurs when the observation ordering matches index ordering. 

There are at least two options for matching to an index with order-dependent descriptors. First, if each crater triad has a single index entry, the index must be searched three times (once with each possible observation ordering). Alternatively, the index could include three entries for each crater triad (one for each ordering) and the larger index is searched only once. The first method (one index per triad) performs best in the presence of measurement noise.

\begin{figure}[b!]
\centering
\includegraphics[width=0.95\columnwidth,trim=0in 0in 0in 0in,clip]{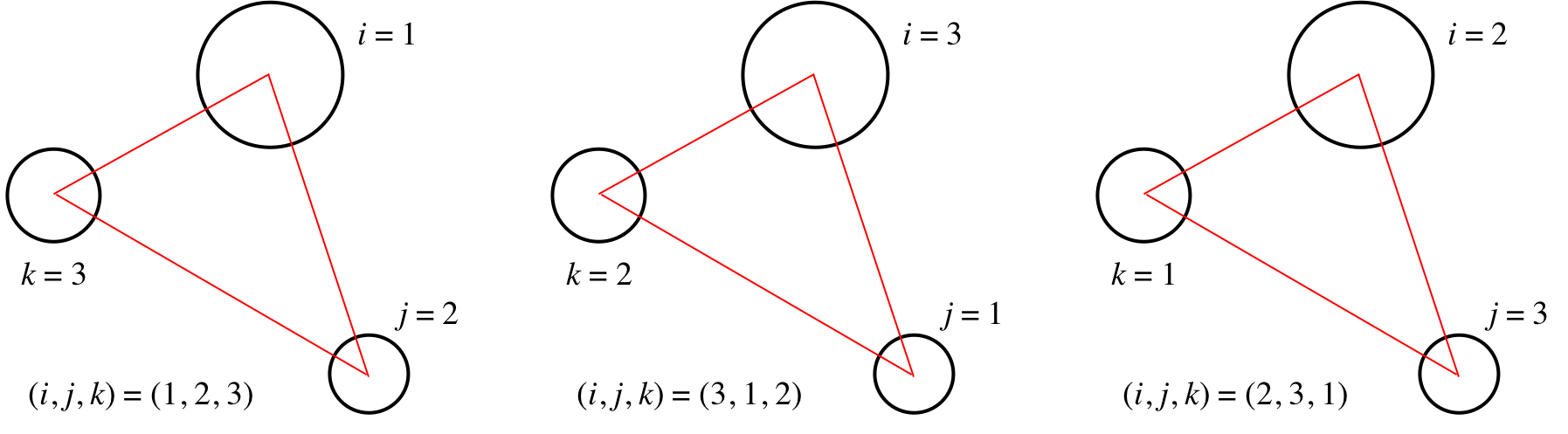}
	\caption{Crater triad labels in an image are assigned in a clockwise fashion. There are three possible permutations of these labels, which may be interpreted as a cyclic permutation of the labels.}
	\label{fig:CraterPermutation}
\end{figure}

The pattern descriptors in this scenario are simply a concatenation of the invariants from Section~\ref{Sec:ComputingInvariantsTop}. Thus, for a triad of conics on a quadric surface, we build a three-element descriptor:
\begin{equation}
   \bJ^T_{ijk} =\left[ \begin{array}{c c c c c c c}
          J_i & J_j & J_k
   \end{array} \right]
\end{equation}
where the individual elements are from Eqs.~\ref{eq:DefJi} to \ref{eq:DefJk}.

Likewise, for a triad of conics on a plane, we build a seven-element descriptor:
\begin{equation}
   \bI^T_{ijk} =\left[ \begin{array}{c c c c c c c}
          I_{ij} & I_{jk} & I_{ki} & I_{ji} &I_{kj} & I_{ik} &  I_{ijk} 
   \end{array} \right]
\end{equation}
where the individual elements are from Eqs.~\ref{eq:DefI1andI2} to \ref{eq:DefI7}.

\subsection{Crater Pattern Descriptors with Sorted Projective Invariants}
\label{Sec:SortedDescriptor}

As before, suppose we observe a triad of craters. Now, instead of having an order-dependent descriptor, suppose that we want to match to a precomputed  index in an order-independent fashion. There are a variety of approaches one might use to achieve this objective. We will see, however, that some care is required to arrive at a robust solution.

A very common technique for achieving permutation invariance is to always permute the order of the descriptor elements such the smallest (or largest) element appears first. For example, one possible descriptor is $\bJ'^T_{ijk} = [J_i',J_j',J_k']$, where $J_i'\leq J_j'\leq J_k'$ is a rearrangement of $(J_i,J_j,J_k)$, i.e., we sort the numbers $J_i,J_j,J_k$ from small to large and let the result be the descriptor. This arrangement was employed by Hanak  \cite{Hanak:2009,Hanak:2010} to define the clockwise/counterclockwise sense of a crater pattern by ordering the legs of the crater triangle from shortest to longest. Similarly, for their lost-in-space algorithm, Maass, et al., \cite{Maass:2020} presort the entries in their index of crater triangles (which  consists of two triangle interior angles) such that the smallest interior angle is first. Other examples of this idea abound within the space navigation community, ranging from crater pattern identification to star pattern identification.

Returning to the problem at hand, we know from our labeling convention that the crater pattern indices may only undergo a cyclic permutation. Thus, we may cyclically permute the entries until smallest valued invariant is in the first position---that is, where $(J_i',J_j',J_k')$ is a cyclic permutation of $(J_i,J_j,J_k)$ with $J'_i = \min(J_i,J_j,J_k)$.

Thus, for a triad of conics on a quadric surface, we build a three-element descriptor:
\begin{equation}
   \bJ^T_{ijk} =\left[ \begin{array}{c c c c c c c}
          J'_i & J'_j & J'_k
   \end{array} \right]
\end{equation}
where we label observed crater $i$ such that  $J'_i = \min(J_i,J_j,J_k)$ and the rest of the primed invariants follow by cyclic permutation. The individual elements are from Eqs.~\ref{eq:DefJi} to \ref{eq:DefJk}.

Likewise, for a triad of conics on a plane, we build a seven-element descriptor:
\begin{equation}
   \bI^T_{ijk} =\left[ \begin{array}{c c c c c c c}
          I'_{ij} & I'_{jk} & I'_{ki} & I'_{ji} &I'_{kj} & I'_{ik} &  I'_{ijk} 
   \end{array} \right]
\end{equation}
where we label observed crater $i$ such that $I'_{ij} = \min(I_{ij},I_{jk},I_{ki})$ and the rest of the primed invariants follow by cyclic permutation of the crater labels $i,j,k$ (and, to be clear, not by cyclic permutation of the elements of the descriptor vector). The individual elements are from Eqs.~\ref{eq:DefI1andI2} to \ref{eq:DefI7}.

The usual argument for anchoring the descriptor element ordering with the smallest (or largest) element is that it is simple. We find in practice, however, that this ordering scheme is not optimally descriptive and is not very robust to noise. Indeed, it is not uncommon for noise to cause two similarly-valued invariants to switch position. If this happens to the value used to anchor the descriptor permutation then it renders the entire descriptor useless---even though the invariants themselves still describe the crater pattern well. Thus, we cycle through the sorted projective invariant descriptors exactly as we do for the unsorted case (Section~\ref{Sec:ProjInvDescriptor}). However, the presorting works most of the time, and greatly reduces the trials required (on average) to find a match.

The main problem with this approach is how sorting populates the $k$-d search  space. By sorting the descriptor to have the smallest element first, we preferentially populate the search space for smaller values along this dimension. This doesn't matter so long as the measurement noise is very small---but, when supplied with noisy measurements, this reduces the likelihood that the correct descriptor will correspond with the first nearest neighbor. The end result is that an index of sorted descriptors will yield fewer matches (as compared with  Section~\ref{Sec:ProjInvDescriptor}) when the data is noisy. Thus, while it may provide speed-up for low noise cases, both matching performance and speed quickly degrade as noise increases.

\subsection{Crater Pattern Descriptors with Projective and Permutation ($p^2$)  Invariants}
\label{Sec:p2Descriptor}

Our objective is to construct a descriptor that is \emph{invariant} to the order in which the craters are observed. It should be no surprise, therefore, that invariant theory is (once again) a useful mathematical framework for achieving this objective. There is some precedence for this approach using the so-called projective and permutation ($p^2$) invariants. The $p^2$ invariants are discussed for a set of five points in \cite{Lenz:1994} and for a pair of conics in \cite{Maybank:1992}. An attempt was made to apply this to the problem of crater identification in \cite{Park:2019}, though this approach suffered from a number of mistakes (see Section~\ref{Sec:BackgroundCraterID}). Within the context of this past work, we now introduce the first complete set of $p^2$ invariants for a triad of conics on a quadric surface or on a plane (e.g., $p^2$ invariants for a triad crater rims on the Moon). 

We briefly remind the reader that the $p^2$ invariants introduced here are algebraic functions of the projective invariants from Section~\ref{Sec:ComputingInvariantsTop}. Thus, the $p^2$ invariants do not represent any new (independent) information. They simply express the information we already have in a different way.

\subsubsection{Non-Coplanar Crater Patterns (Three-Element Descriptor)}
We first develop the $p^2$ invariants for a triad of conics lying on a quadric surface. We know there exists exactly three algebraically independent invariants for this case, and we are able to compute three such invariants $J_i, J_j, J_k$ directly from an image using Eqs.~\ref{eq:DefJi} to \ref{eq:DefJk}. Our goal is to compute rational functions of $J_i, J_j, J_k$ that are invariant to a cyclic permutation of the crater labels.

Consider, therefore, the triple of numbers $\{x,y,z\}$. The standard generators of the ring of polynomial invariants for the symmetric group, $S_3$, of all coordinate permutations are the elementary symmetric functions (see \cite[Chapter 4]{Neusel:2007})
\begin{align}
e_1 &=x+y+z \\
e_2 &=xy+yz+zx \\
e_3 &=xyz
\end{align}
which are clearly invariant to permutations of $\{x,y,z\}$.
To obtain the generators of the invariant ring for the cyclic group $\Z/3\Z$, which we may also view as the alternating group $A_3$, we require one more invariant:
\begin{equation}
\Delta=(x-y)(y-z)(z-x)
\end{equation}
(see for example \cite[Example 4.24]{Neusel:2007}).
We observe that $\Delta$ is invariant under cyclic permutations, but that odd permutations change its sign. While these polynomials could certainly be used to construct $p^2$ invariants and build a pattern descriptor, we note that the polynomials are of different degree which often leads to undesirable numerical properties for indexing. Thus, we instead look for three rational functions, where the difference in the degrees of the numerator and denominator are the same for all three functions.

To do this, we begin by defining a polynomial map
\begin{equation}
F(x,y,z)=(F_1(x,y,z),F_2(x,y,z),F_3(x,y,z))
\end{equation}
where $F_1,F_2,F_3$ are continuous rational functions with the property that $F(x,y,z)=F(x',y',z')$
if and only if $(x',y',z')$ is a cyclic permutation of $(x,y,z)$.
Let the cyclic group $\Z/3\Z$ act on $\R^3$ by permuting the coordinates.
Then we desire $F_1,F_2,F_3$ to be invariant functions on $\R^3$.

Whenever a finite group acts by a linear transformation, there exists a particular coordinate change  where the action of the group becomes diagonal. That is, in some coordinate system, the matrix describing the action of each group element is a diagonal matrix. This follows from two facts from the representation theory of finite groups (see~\cite[Chapter 1]{FultonHarris:1991}). First, every representation of a finite group is a direct sum of irreducible representations. Second, every irreducible representation of an finite abelian group is 1-dimensional. For the specific problem under consideration here (the action of $\Z/3\Z$ on $\R^3$), we can make this very explicit: namely, the change of coordinates making the action diagonal is a discrete Fourier transform (DFT) (which involves complex numbers).
To make these ideas explicit, let $a,b,c$ be
\begin{align}
a & =x+y+z \\
b & =x+\zeta y+\zeta^2 z \\
c & =x+\zeta^2y+\zeta z
\end{align}
where $\zeta=e^{2\pi i/3}=-\frac{-1+\sqrt{3}i}{2}$. Then the generator $\sigma$ of $\Z/3\Z$ acts by the diagonal matrix
$$
\left[\begin{array}{ccc}
1 & 0 & 0\\
0 & \zeta & 0\\
0 & 0 & \zeta^2
\end{array}\right] 
$$
We define $F_1$, $F_2$, and $F_3$ as rational functions of $a,b,c$,
\begin{align}
F_1(x,y,z) & = a =x+y+z \label{eq:F1}\\
F_2(x,y,z) & = \frac{b^2}{c}+\frac{c^2}{b} \label{eq:F2}\\
           & =\frac{2(x^3+y^3+z^3)+12xyz-3(x^2y+y^2x+y^2z+z^2y+z^2x+x^2z)}{x^2+y^2+z^2-(xy+yz+zx)}  \nonumber\\
F_3(x,y,z) & =\frac{1}{i}(\frac{b^2}{c}-\frac{c^2}{b}) \label{eq:F3} \\ 
& =\frac{-3\sqrt{3}(x-y)(y-z)(z-x)}{x^2+y^2+z^2-(xy+yz+zx)} \nonumber
\end{align}
The denominator of $F_2$ and $F_3$ is equal to $\frac{1}{2}[(x-y)^2+(y-z)^2+(z-x)^2]$ and is equal to $0$ if and only if $x=y=z$. We can also verify that
$$
F_2^2+F_3^2=2[(x-y)^2+(y-z)^2+(z-x)^2]=2bc
$$
This implies that if $(x,y,z)$ is a sequence that converges to
a point $(t,t,t)$, then $F_2(x,y,z)$ and $F_3(x,y,z)$ converge to $0$. Consequently, we can extend $F_2$ and $F_3$ to continuous functions on all of $\R^3$ by defining them to be $0$ whenever $x=y=z$.
\begin{lemma}
\label{Lemma:CyclicPermF}
 $F(x,y,z)=F(x',y',z')$
if and only if $(x',y',z')$ is a cyclic permutation of $(x,y,z)$
for all $(x,y,z),(x',y',z')\in \R^3$.
\end{lemma}
\begin{proof}
Suppose now that $F(x,y,z)=F(x',y',z')$. Let $a,b,c$ be as given in  Eqs.~\ref{eq:F1}--\ref{eq:F3}
and define similarly $a'=x'+y'+z'$, etc. To begin, observe that
if $F_2(x,y,z)=F_2(x',y',z')=F_3(x,y,z)=F_3(x',y',z')=0$
then we must have $x=y=z$ and $x'=y'=z'$ and from $F_1(x,y,z)=3x=3x'=F_1(x',y',z')$ it follows that $(x,y,z)=(x',y',z')$.
Now suppose that $F_2(x,y,z)^2+F_3(x,y,z)^2$ (and, therefore,
$F_2(x',y',z')^2+F_3(x',y',z')^2$) is nonzero. Then $b,b',c,c'$ are all nonzero. Since $F_2(x,y,z)=F_2(x',y',z')$ and $F_3(x,y,z)=F_3(x',y',z')$ we get 
$$\frac{b^2}{c}=\frac{F_2(x,y,z)+iF_3(x,y,z)}{2}=\frac{F_2(x',y',z')+iF_3(x',y',z')}{2}=\frac{(b')^2}{c'}
$$
and
$$
2bc=F_2(x,y,z)^2+F_3(x,y,z)^2=F_2(x',y',z')^2+F_3(x',y',z')^2=2b'c'.
$$
This implies that
$$
b^3=\frac{b^2}{c}(bc)=\frac{(b')^2}{c'}b'c'=(b')^3.
$$
Therefore $b'=\zeta^sb$ for some $s\in \{0,1,2\}$.
Taking complex conjugation, we get that $c'=\zeta^{-s}c$.
We conclude that $(a,b,c)$ and $(a',b',c')$ lie in the same orbit $\Z/3\Z$, which means that $(x',y',z')$ is a permutation of $(x,y,z)$.
\end{proof}
One can also verify that the field of rational invariants $\R(x,y,z)^{\Z/3\Z}$ is generated by $F_1,F_2,F_3$.

\subsubsection{Coplanar Crater Patterns (Seven-Element Descriptor)}
The development of $p^2$ invariants for a triad of conics lying on a plane follows similar arguments as for conics on a quadric, though it requires more care due to additional relationships between the seven elements in the descriptor. Therefore, recall from Eqs.~\ref{eq:DefI1andI2}--\ref{eq:DefI7} that the seven invariants consists of three invariant pairs, $(I_{ij},I_{ji}), (I_{jk},I_{kj}), (I_{ki},I_{ik})$, and a 7th invariant for the entire triad $I_{ijk}$. Because they're constructed from a common set of crater observations, the two invariants within any given pair always undergo a common permutation. We also observe that $I_{ijk}$ is already symmetric and invariant to the ellipse label order.

Simply by applying the results of the previous section, we can find six $\Z/3\Z$-invariant functions from $F_\ell(I_{ij},I_{jk},I_{ki})$ and $F_\ell(I_{ji},I_{kj},I_{ik})$.
The descriptiveness of these six invariants, however, is not optimal because members of a common pair are not forced to undergo the same permutation. Therefore, we introduce two new invariants. 

Again making use of a DFT, define $a_\ell, b_\ell, c_\ell$ as 
\begin{align}
a_\ell &=x_\ell+y_\ell+z_\ell\\
b_\ell &=x_\ell+\zeta y_\ell+\zeta^2 z_\ell\\
c_\ell &=x_\ell+\zeta^2 y_\ell+\zeta z_\ell
\end{align}
for $\ell=1,2$. Note that $c_\ell=\overline{b}_\ell$
where $x_\ell,y_\ell,z_\ell$ are real.
A generator of $\Z/3\Z$ acts on $(x_1,y_1,z_1)$ and $(x_2,y_2,z_2)$ by the same cyclic permutation of coordinates,
and it acts on $b_1,b_2$ with a scalar $\zeta$ and on $c_1,c_2$ with a scalar $\zeta^2$. We see, therefore, that $b_1c_2$ is invariant.
We define $G_1(x_1,y_1,z_1,x_2,y_2,z_2)$ and $G_2(x_1,y_1,z_1,x_2,y_2,z_2)$ be the real and imaginary part
of $b_1c_2$, which we compute as
\begin{multline*}
    b_1c_2=(x_1+\zeta y_1+\zeta^2 z_1)(x_2+\zeta^2 y_2+\zeta z_2)=\\=(x_1x_2+y_1y_2+z_1z_2)+\zeta(x_1z_2+y_1x_2+z_1y_2)+
\zeta^2(x_1y_2+y_1z_2+z_1x_2)
\end{multline*}
We may isolate the real part to obtain $G_1$
\begin{multline*} 
G_1(x_1,y_1,z_1,x_2,y_2,z_2)=\\=(x_1x_2+y_1y_2+z_1z_2)-{\textstyle \frac{1}{2}}(x_1z_2+y_1x_2+z_1y_2+x_1y_2+y_1z_2+z_1x_2)\\=
{\textstyle \frac{3}{2}}(x_1x_2+y_1y_2+z_1z_2)-{\textstyle \frac{1}{2}}(x_1+y_1+z_1)(x_2+y_2+z_2)
\end{multline*}
and the imaginary part to obtain $G_2$
\begin{multline*}
G_2(x_1,y_1,z_1,x_2,y_2,z_2)=\frac{\sqrt{3}}{2}\big[(x_1z_2+y_1x_2+z_1y_2)-(x_1y_2+y_1z_2+z_1x_2)\big]\\=-\frac{\sqrt{3}}{2}\left|\begin{array}{ccc}1 & 1 & 1\\ x_1 & y_1 & z_1\\ x_2 & y_2 & z_2
\end{array}\right|
\end{multline*}
\begin{lemma}
Suppose that $x_\ell,y_\ell,z_\ell,x_\ell',y_\ell',z_\ell'\in \R$ for $\ell=1,2$ such that
\begin{eqnarray*}
F(x_1,y_1,z_1)&=&F(x_1',y_1',z_1')\\
F(x_2,y_2,z_2)&=&F(x_2',y_2',z_2')\\
G(x_1,y_1,z_1,x_2,y_2,z_2)&=&G(x_1',y_1',z_1',x_2',y_2',z_2')
\end{eqnarray*}
Then
$$
(x_1,x_2),(y_1,y_2),(z_1,z_2)
$$
is a cyclic permutation of
$$
(x_1',x_2'),(y_1',y_2'),(z_1',z_2')
$$
\end{lemma}
\begin{proof}
By Lemma~\ref{Lemma:CyclicPermF} we know that $(x_1',y_1',z_1')$ is a cyclic permutation of $(x_1,y_1,z_1)$ and
$(x_2',y_2',z_2')$ is a cyclic permutation of $(x_2,y_2,z_2)$, i.e., $b_1'=\zeta^s b_1$ and $b_2'=\zeta^tb_2$ for some $s,t\in \{0,1,2\}$. It is not yet clear that this is the same
permutation, i.e., whether $s=t$.
From $G(x_1,y_1,z_1,x_2,y_2,z_2)=G(x_1',y_1',z_1',x_2',y_2',z_2')$ follows that $b_1\overline{b}_2=b_1c_2=b_1'c_2'=b_1'\overline{b}_2'$. On the other hand, $b_1\overline{b}_2=b_1'\overline{b}_2'=(\zeta^s b_1)(\zeta^{-t}\overline{b}_2)=\zeta^{s-t}b_1\overline{b}_2$.
So either $b_1=0$, $b_2=0$, or $s=t$. If $t=s$, then we are done.
If $b_1=0$ then $x_1=y_1=z_1=x_1'=y_1'=z_1'$ and we are also done because $(x_2',y_2',z_2')$ is a cyclic permutation of $(x_2,y_2,z_2)$. The case $b_2=0$ goes similarly.
\end{proof}

It follows, therefore, that there are many approaches to choose a descriptor for a set of three coplanar craters. We summarize three obvious choices, with the third being the one we usually recommend. 

The most obvious approach is to form a seven-element descriptor as $$[F(I_{ij},I_{jk},I_{ki}),F(I_{ji},I_{kj},I_{ik}),I_{ijk}]$$ 
As was noted above, the descriptiveness of such a scheme is not optimal, because this descriptor remains the same when the values of $I_{ij},I_{jk},I_{ki}$ cyclically rotate while the values of $I_{ji},I_{kj},I_{ik}$ stay the same. This means that 3 different configurations might have the same descriptor.

A second option is to form a nine-element descriptor
    $$[F(I_{ij},I_{jk},I_{ki}),F(I_{ji},I_{kj},I_{ik}),\widetilde{G}(I_{ij},I_{jk},I_{ki},I_{ji},I_{kj},I_{ik}),I_{ijk}]$$ 
This descriptor is more descriptive but uses more space. Here, we have replaced the invariant $G$ with $\widetilde{G}$, which we define as
    \begin{multline*}
    \widetilde{G}(x_1,y_1,z_1,x_2,y_2,z_2)=\\\frac{G(x_1,y_1,z_1,x_2,y_2,z_2)}{\sqrt[4]{\Big((x_1-y_1)^2+(y_1-z_1)^2+(z_1-x_1)^2\Big)\Big((x_2-y_2)^2+(y_2-z_2)^2+(z_2-x_2)^2\Big)}}
    \end{multline*}
This is an essential step, since $G$ scales quadratically with $(x_1,y_1,z_1,x_2,y_2,z_2)$  while $F$ scales linearly. Using $G$ instead of $\widetilde{G}$ results in a poorly scaled descriptor that complicates (and sometimes precludes) nearest neighbor searches with efficient data structures. Conversely, the function $\widetilde{G}$ scales linearly. We note that $\widetilde{G}$ is not a rational function because of the $4$-th root, but it does extend to a continuous function on $\R^6$ by defining the function to be 0 whenever $x_1=y_1=z_1$ or $x_2=y_2=z_2$. 

A third option (and the one we suggest) is to use the seven-element descriptor
    $$
    [F(I_{ij},I_{jk},I_{ki}),F_1(I_{ji},I_{kj},I_{ik}),\widetilde{G}(I_{ij},I_{jk},I_{ki},I_{ji},I_{kj},I_{ik}),I_{ijk}]
    $$ 
For generic configurations, this descriptor is as good as the nine-element descriptor. The subfield of $L=\C(x_1,y_1,z_1,x_2,y_2,z_2)$ generated by $F_\ell(x_1,y_1,z_1)$, $\ell=1,2,3$, $F_1(x_2,y_2,z_2)$ and $G_{\ell}(x_1,y_1,z_1,x_2,y_2,z_2)$, for $\ell=1,2$ is equal to $K=\C(a_1,a_2,b_1^3,b_1c_1,b_1c_2,c_1b_2)$. It is easy to see that $L=K(b_1)$ and $b_1$ has degree three over $K$. So the degree of the field extension $L/K$ is at most three, $K$ is contained in the fixed field $L^{\Z/3\Z}$
and the degree of the extension $L/L^{\Z/3\Z}$ is three.
This implies that $K=L^{\Z/3\Z}$. In other words, every $\Z/3\Z$-invariant rational function in $I_{ij},I_{jk},I_{ki},I_{ji},I_{kj},I_{ik}$ is a rational function in the first 6 elements of the descriptor.
    
However, in the degenerate case $I_{ij}=I_{jk}=I_{ki}=t$ for some fixed $t$
we  have $G(I_{ij},I_{jk},I_{ki},I_{ji},I_{kj},I_{ik})=0$.
In that case, the descriptor cannot distinguish between different configurations of the same pattern (e.g., a pattern and it's mirror). This is to be expected, since the degenerate case corresponds to equally sized craters around an equilateral triangle. Here, while there may be enough information to uniquely identify the triad (i.e., match the triad descriptor to the  database), there is not sufficient information to disambiguate the specific crater labels with the projected crater rims alone. No special action is required, since this ambiguity is naturally handled in the pattern verification process (Section~\ref{Sec:IndexMatchingTop}).

\subsection{Remarks on Choosing a Descriptor Convention}
In the sections above we introduce descriptors based on the projective invariants, sorted projective invariants, and on the $p^2$ invariants. Which one is best is often application dependent. 

Descriptors built directly on the projective invariants (Section~\ref{Sec:ProjInvDescriptor}) have better matching performance when presented with noisy data (see left-hand plot of Fig.~\ref{fig:DescriptorStats}), though they require the index be searched three times per observed pattern (one for each possible permutation). 

Conversely, the descriptors built with the sorted invariants (Section~\ref{Sec:SortedDescriptor}) or $p^2$ invariants (Section~\ref{Sec:p2Descriptor}) only require one index search per observed pattern, but are more sensitive to measurement noise. As a consequence the sorted invariant and $p^2$ invariant descriptors are faster for low noise situations, but slower for high noise situations (see right-hand plot of Fig.~\ref{fig:DescriptorStats}). This may seem counter-intuitive. The explanation, however, is straightforward. As measurement noise increases, the sorted invariant and $p^2$ invariant descriptors must attempt more triads (on average, as  compared to the ordered projective invariant descriptor) before finding a correct nearest neighbor match. Thus, with large amounts of measurement noise, the ordered projective invariant descriptor tends to find a match sooner despite needing three index searches per triad---and finding a match sooner results in a lower (faster) run-time. 

\begin{figure}[t!]
\centering
\includegraphics[width=1\columnwidth,trim=0in 0in 0in 0in,clip]{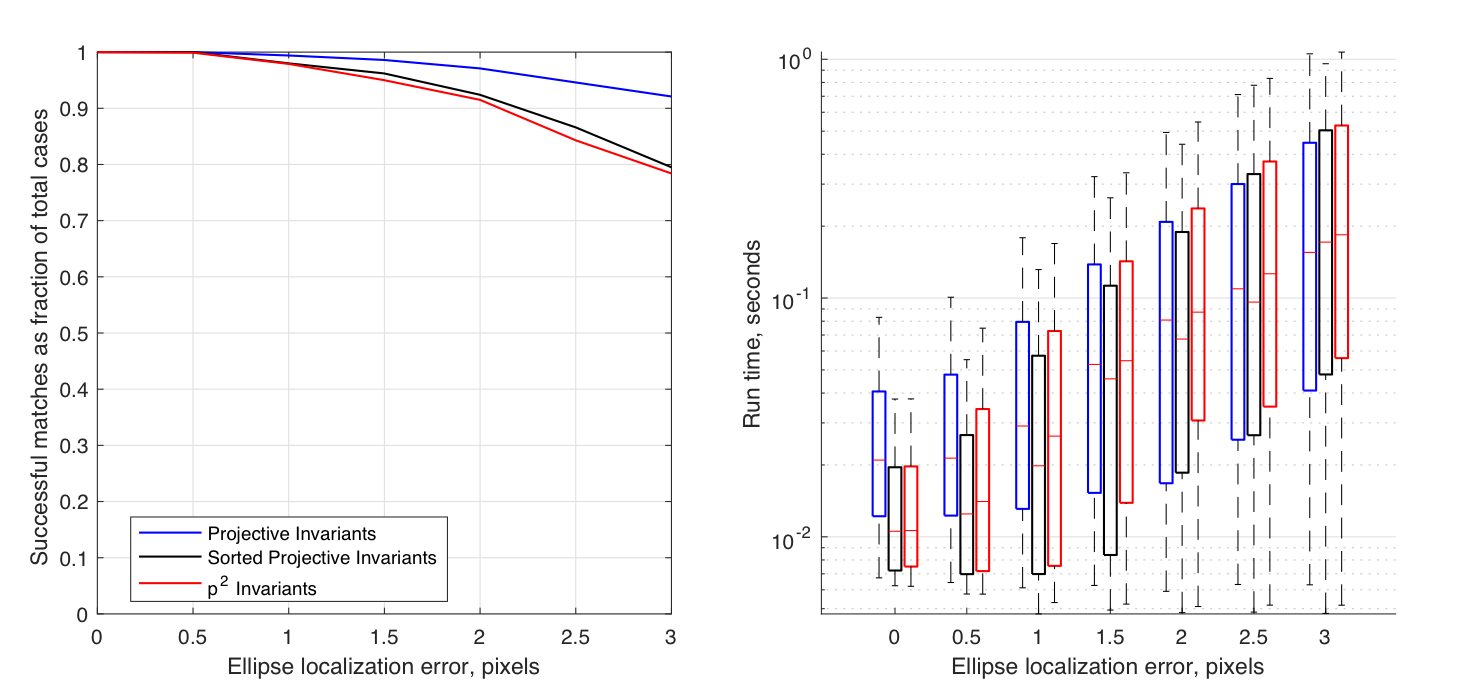}
	\caption{A comparison of crater matching performance and execution speed for 1,000 random images of the Moon. These images correspond to a spacecraft at 600 km altitude placed randomly around the Moon (uniform distribution over the lunar sphere). Matching is performed using the three-element (non-coplanar) crater triad descriptor from  Section~\ref{Sec:NonCoplanarTriadInvariants}. The left-hand plot shows the fraction of cases with a successful match. The balance of cases returned no match, and there were no cases where an incorrect match was returned. The right-hand plot shows execution time for the entire matching pipeline (conics as input, verified match as output) using the three different descriptors. These times correspond to non-optimized prototype code and should be taken to represent relative (rather than absolute) execution time.}
	\label{fig:DescriptorStats}
\end{figure}

\section{Building a Global Crater Index}
\label{Sec:BuildCrateIndex}
Global lunar crater identification requires multi-scale indexing and careful catalog curation. When close to the Moon, a spacecraft sees small craters that are nearly coplanar. Conversely, when far from the Moon, a spacecraft sees larger craters distributed over the lunar sphere (not coplanar). This immediately suggests at least two crater indexes be built---one for small patterns of nearly coplanar craters and another for larger patterns of non-coplanar craters. In practice, we find that the most efficient real-time performance is achieved by creating a single index for the planar case and a hierarchy of indexes for the non-coplanar case.  These indexes may be combined into a single large index or kept as separate indexes. We find the latter to often be the better choice.

The lunar crater database from \cite{Robbins:2018} contains 1.3M craters over about 1--2 km in diameter. It is immediately obvious that consideration of all the $\binom{1.3\text{M}}{3}=3.66\times10^{17}$ combinations of possible crater triads is not reasonable. 
This further motivates construction of a hierarchy of indexes with different scales to better manage the combinations of craters that could plausibly be observed at the same time. 

We propose the Hierarchical Equal Area isoLatitude Pixelization (HEALPix)\footnote{Software implementations of HEALPix in C++, FORTRAN, and Python are freely available online at \url{https://healpix.sourceforge.io}.} framework \cite{Gorski:2005} be used to subdivide the lunar surface into equal area regions, which we refer to as \emph{surface pixels}. HEALPix was developed for subdividing the celestial sphere in support of science objectives for the Wilkinson Microwave Anisotropy Probe (WMAP) mission \cite{Bennett:2003}, and has since found widespread use for analysis of data from other NASA and ESA missions [e.g., Cosmic Background Explorer (COBE), Planck]. It is also used for managing the creation of star quadrilaterals in the present state-of-the-art for star identification, calibration, and alignment of astrometric images \cite{Lang:2010}. We believe our work to be the first application of HEALPix to the management of lunar surface features.

All indexes are built using the same fundamental approach, but at different scales. The procedure is as follows: 

First, the lunar surface is tiled into $N_{pix} = 12 (2^{2k})$ surface pixels of equal area. Next, a list is constructed for each pixel containing the catalog entries for craters within a specified size range (e.g., minimum/maximum diameter) and whose catalog fit was constructed using at least 90\% of the rim's circumference.

Second, given a list of usable craters in each HEALPix surface pixel, we loop through all the pixels to create crater triads. At each pixel, we consider craters from the $3\times3$ HEALPix grid centered about the reference pixel. From these 9 surface pixels, all possible triads are formed where (1) craters do not intersect one another and (2) the triad center lies within the reference (center) surface pixel. Valid crater triads are arranged in a clockwise order and the scale-appropriate projective invariant descriptors are computed. To better illustrate this, Fig.~\ref{fig:ApolloHEALPix} shows an example $3\times3$ HEALPix grid (region of support for surface pixel 6318) overlayed on an example Metric Camera image from Apollo 17. 

\begin{figure}[b!]
\centering
\includegraphics[width=0.7\columnwidth,trim=0in 0in 0in 0in,clip]{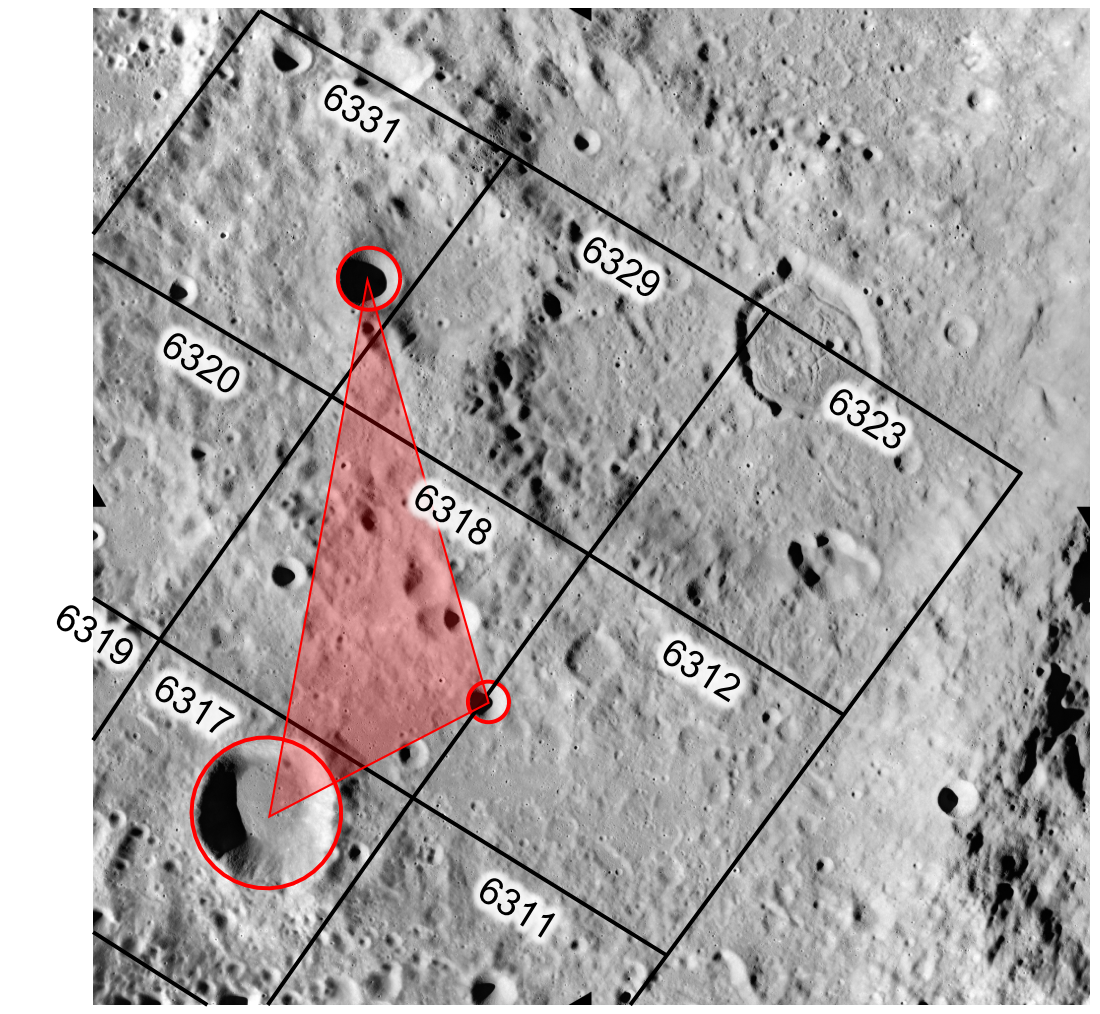}
	\caption{Example lunar crater triad belonging to HEALPix 6318. Overlay is on Apollo Metric Camera image AS17-M-0684. (Credit for raw scans of Apollo flight film images: NASA/JSC/ASU. See \cite{Robinson:2008,Lawrence:2008}.  }
	\label{fig:ApolloHEALPix}
\end{figure}

Third, once the triads belonging to all surface pixels have been computed, the results are stored in an efficiently searchable data structure. The authors have found excellent performance using either a $k$-d tree \cite{Bentley:1975} or $n$-d $k$-vector \cite{Arnas:2020}, though other reasonable choices exist. We briefly remark that the 1-d $k$-vector \cite{Mortari:2000,Mortari:2013} has seen extensive use for star pattern matching in space applications \cite{Mortari:2004,Spratling:2009},  and the $n$-d $k$-vector has been proposed for space applications as well \cite{Leake:2020}. The specific choice of data structure is not of primary concern in the present analysis so long as a reasonable selection is made.

\begin{figure}[b!]
\centering
\includegraphics[width=0.65\columnwidth,trim=0in 0in 0in 0in,clip]{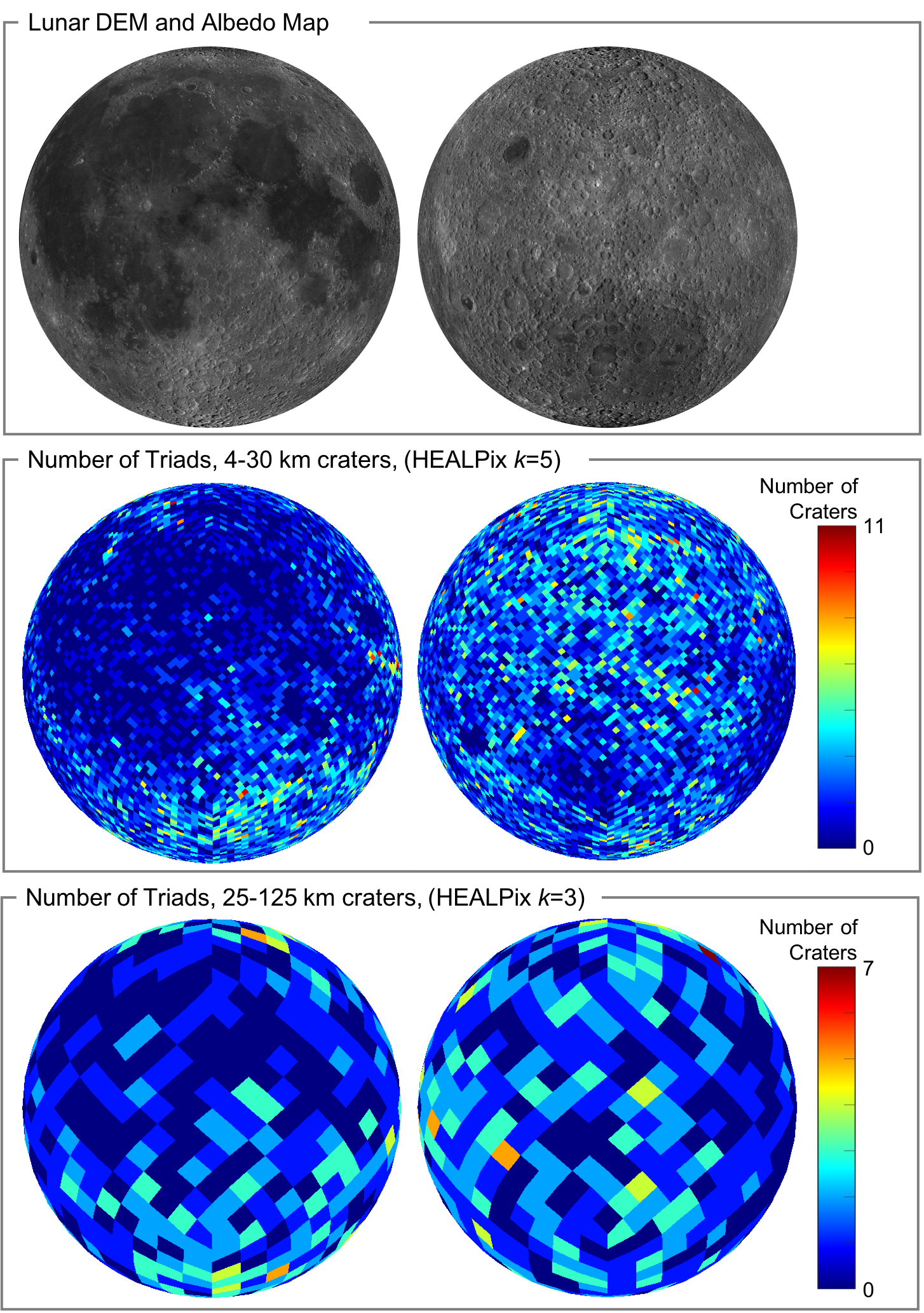}
	\caption{Graphical depiction of the number of triads in each HEALPix surface pixel. Left column shows orthographic projection of lunar near side, while the right column shows the same for the lunar far side.}
	\label{fig:CraterHEALPix}
\end{figure}

\begin{figure}[b!]
\centering
\includegraphics[width=1\columnwidth,trim=0in 0in 0in 0in,clip]{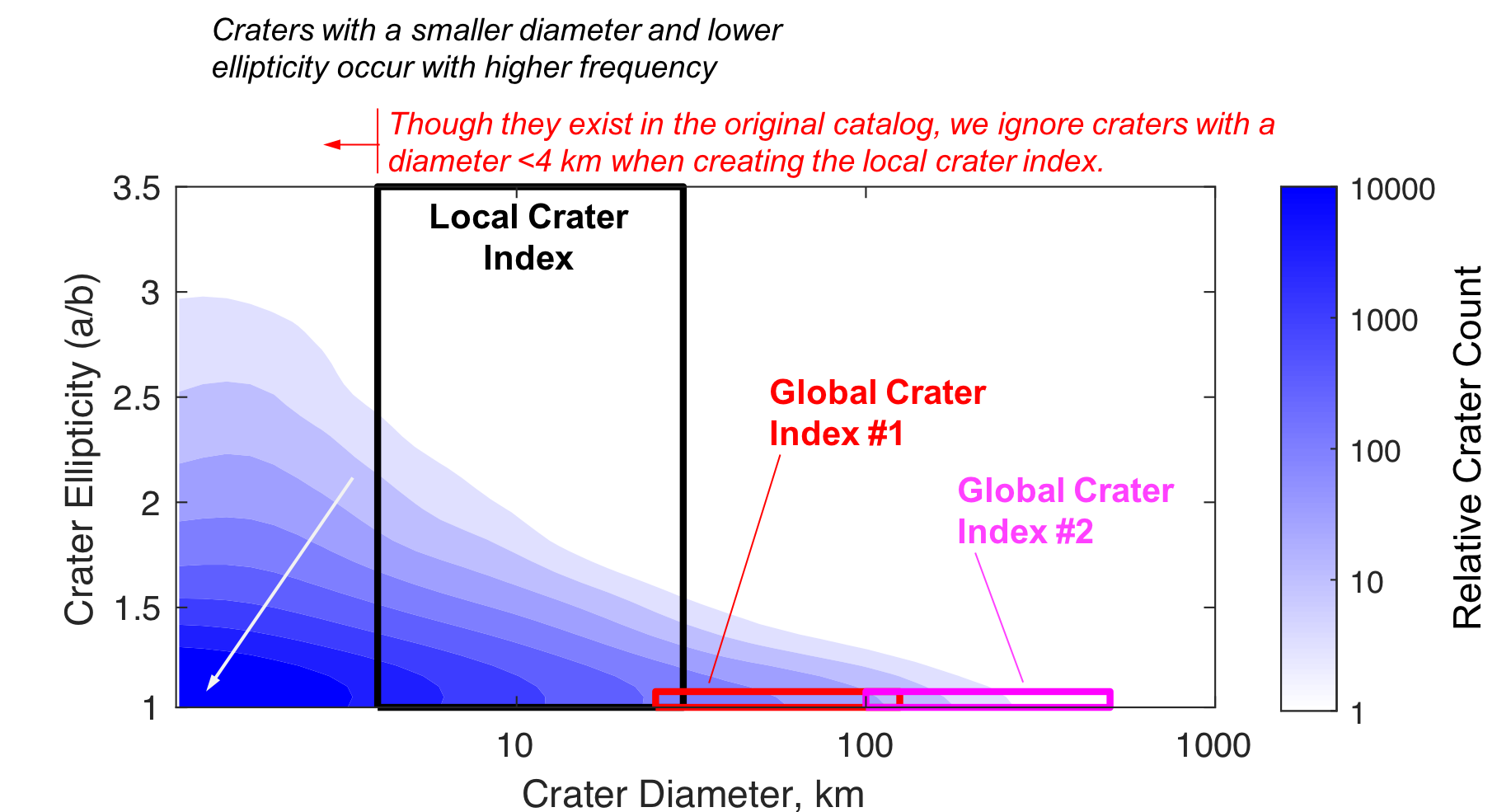}
	\caption{Relative lunar crater density as a function of diameter and ellipticity. Areas of darker blue indicate more craters per unit area. Overlayed boxes show regions from which craters are drawn to for the local and global crater indexes.}
	\label{fig:CraterIndexSizeEllip}
\end{figure}

\subsection{Indexing Local Crater Patterns with Coplanar Invariants}
To construct a global index of local-scale (small) crater patterns, we choose a HEALPix resolution  of $k=5$ for craters of diameter 4--30 km. This yields $N_{pix} = 12 (2^{10}) = 12,288$ surface pixels, with each surface pixel having a surface area of approximately $3,086 \text{ km}^2$.  There are 20,737 craters with catalog parameters supported by $>90\%$ of the crater rim circumference and having a diameter between 4 km and 30 km. Rather than producing $\binom{20,737}{3}=1.49\times10^{12}$ combinations, the HEALPix grouping strategy only produces the 4.8M crater triads that are nearly coplanar.

The local crater patterns are assumed coplanar and craters are allowed to follow an arbitrary elliptical shape. Thus, we construct a seven-element descriptor using an appropriate method from Section~\ref{Sec:PatternDescriptor}.

\subsection{Indexing Global Crater Patterns with Non-Coplanar Invariants}
For crater patterns that are regional or global in extent, we find it helpful to construct at least two indexes---though specific mission needs may require more or less. The extension to a single global index or many global indexes is trivial.

The Moon is very nearly a sphere at the global level. Therefore, because the non-coplanar invariants require each pair of ellipses to lie on a common quadric surface, the craters must be nearly circular. The global lunar crater databases  constructed here are built only from craters having an ellipticity of $a/b \leq 1.1$. Since the global patterns are assumed to lie on a nondegenerate quadric surface, we construct a three-element descriptor using an appropriate method from Section~\ref{Sec:PatternDescriptor}.

The first global crater index has a HEALPix resolution of $k=3$ for craters of diameter 25--125 km. This yields $N_{pix} = 12 (2^{6}) = 768$ surface pixels, with each surface pixel having a surface area of approximately $49,390 \text{ km}^2$. The index is good for matching crater patterns at the regional level. There are 904 nearly circular craters ($a/b \leq 1.1$) with catalog parameters supported by $>90\%$ of the crater rim circumference and having a diameter between 25 km and 125 km. Rather than producing $\binom{904}{3}=1.23\times10^{8}$ combinations, the HEALPix grouping strategy only produces the $140,929$ crater triads. The example crater triad in Fig.~\ref{fig:InvariantProjectionPANGU} (see Section~\ref{Sec:NonCoplanarInvariants}) was recognized using this index.

The second global crater index has a HEALPix resolution of $k=1$ for craters having a diameter over $100$ km.  This yields $N_{pix} = 12 (2^{2}) = 48$ surface pixels, with each surface pixel having a surface area of approximately $790,300 \text{ km}^2$. The index is good for matching crater patterns at the global level, when nearly an entire hemisphere is visible in an image. There are 31 nearly circular craters ($a/b \leq 1.1$) with catalog parameters supported by $>90\%$ of the crater rim circumference and having a diameter over $100$ km.  Rather than producing $\binom{31}{3}=4,495$ combinations, the HEALPix grouping strategy only produces the $707$ crater triads.

\subsection{Remarks on the Utility of Crater Index Hierarchies}

The primary purpose of having more than one index is to apply coplanar invariants at the local scale (allowing for arbitrary elliptical crater shape) and non-coplanar invariants at the regional/global scale (where curvature of the Moon makes craters lie in substantially different planes).

In practice, the orbital regime is often known ahead of time (e.g., LLO, cislunar), allowing us to only search crater patterns in the index of appropriate scale. If, however, we are truly ``lost-in-space'' and don't know the orbital regime \emph{a priori}, we have found that all indexes may be queried and only the correct one will produce a match. Thus, the hierarchy of indexes allows us to exploit mission-specific information when it is available, but does not necessarily require this information.

While this work builds a truly global index for both local crater patterns and global crater patterns, it may sometimes  be desirable to only include craters visible along a reference trajectory.

\section{Matching Observed Crater Patterns to a Pre-Built Index}
\label{Sec:IndexMatchingTop}
When supplied an image of the lunar surface, the objective is to recognize a pattern of observed craters using the image conics corresponding to the projection of the crater rims.

This is accomplished using a straightforward matching and and verification procedure. First, a crater detection algorithm (CDA) produces an ellipse fit to a crater rims in an image. Second, we consider triads of these image conics, cycling through combinations using the Enhanced Pattern Shifting (EPS) method \cite{Arnas:2017}. For each crater triad, we compute the appropriate descriptor (see Section~\ref{Sec:PatternDescriptor}) and query the corresponding pre-built index to find potential matches (e.g., best $N\geq1$ matches). For each crater match hypothesis, we compute the unknown camera position (see Section~\ref{Sec:ConicPose}). After first checking to ensure the position estimate is not inside the Moon (as would happen if we accidentally match to a pattern on the wrong side of the Moon), we then reproject the expected crater rim locations in the image using Eq.~\ref{eq:QuadricProjection}. The match is verified by comparing the observed and expected crater rims using a new distance metric (see Section~\ref{Sec:EllipseComp}). Previous methods only compared crater centers and thus required 4--5 correct correspondences to verify a match. Comparing the craters rims is more informative, allowing for pattern match verification with only three craters (thus allowing crater identification above more sparsely cratered terrain).

If a pattern match hypothesis is verified by this procedure, we declare this the solution and terminate the search. If a pattern match hypothesis is not verified we consider another hypothesis for that triad or move to the next triad in the EPS sequence. If we reach the end of the EPS sequence with no hypothesis being verified, we declare that no match is possible and terminate the search. 

The subsections that follow describe the mathematics for pose estimation from corresponding non-coplanar conics and a crater rim distance metric. Both of these algorithms are novel.

\subsection{Pose from Non-Coplanar Conics in Correspondence}
\label{Sec:ConicPose}

The usual approach for computing pose from matched craters only uses the coordinates of the crater center. Here, it is essential to remember that the center of an image ellipse does not generally produce a line-of-sight vector to the center of a 3D ellipse (or circle). Some past crater identification pipelines consider this effect (e.g., \cite{Park:2019}), but many do not. Regardless, there are a variety of algorithms one may use to compute pose from corresponding 2D image and 3D model points \cite{Ansar:2003,Lepetit:2009}. Unfortunately, however, using only the crater center points neglects the substantial amount of valuable navigation information contained within the shape of the projected crater rims. We aim to address this problem here.

Were the 3D conics strictly coplanar, there are a number of algorithms one could use to estimate pose from the projected image conics \cite{Sugimoto:2000,Kannala:2006,Wang:2017}. However, while we assume local crater patterns are coplanar for index building/matching, there is no reason to assume coplanarity for pose estimation or match verification. Furthermore, we require the ability to solve this problem for both local and global patterns---thus requiring a method for pose estimation from the projection of non-coplanar conics.

The literature discussing pose estimation from non-coplanar conics is limited. The most common approach is to look at a pair of ellipses in an image and construct a single-parameter family of possible poses from one of these two ellipses. This one parameter family is then numerically searched to find the best possible agreement with the second ellipse. This procedure was suggested in \cite{Ma:1993}, which is generally cited as the solution to this problem.

The approach of \cite{Ma:1993}, however, is not appropriate  for spacecraft navigation as it (1) does not fully use the information content of every conic and (2) it only deals with a conic pair. Instead, we search for a solution that uses all the information from a $d$-tuple (usually a triad) of image conics to solve for camera pose in the least squares sense.

Since the application at hand is lunar crater identification, we may simplify the problem by assuming the relative attitude is known. We consider this to be a reasonable assumption in practice since we have excellent knowledge of both the Moon's attitude (from ephemeris data; e.g., SPICE kernels \cite{Acton:1996,Acton:2018}) and the spacecraft attitude (from star trackers \cite{Liebe:1995,Liebe:2002}). There are two observations that can now be made. First, we can think of no plausible failure mode (where recovery is still possible) in which the spacecraft has no knowledge of time (necessary for finding lunar attitude from SPICE kernels) or of inertial attitude. Second, we observe that this still formally qualifies as a lost-in-space problem since it assumes no knowledge of the spacecraft translational states (position  or velocity). Moreover, the ``known'' attitude comes from the star tracker and there are many well-established lost-in-space star identification algorithms \cite{Spratling:2009,Rijlaarsdam:2020}. Regardless, we presume the attitude transformation matrix $\bT^M_C$ is known, such that the pose estimation problem only requires a solution for the selenographic camera location $\br_M$.

Therefore, let us begin with consideration of the projection of a single conic as described by Eq.~\ref{eq:HomographyPointConicWithScale}. Now, substitute for $\bH_{C_i}$ from Eq.~\ref{eq:DefHCi} and expand to find
\begin{equation}
    \label{eq:ConicPoseExpansion1}
    \bH_{M_i}^T \bB_i \bH_{M_i} - \bH_{M_i}^T \bB_i \br_M \bk^T - \bk \br_M^T \bB_i  \bH_{M_i} + \br_M^T \bB_i \br_M \bk \bk^T = s_i \bC_i 
\end{equation}
where
\begin{equation}
    \bB_i = \bT^C_M \bK^T \bA_i \bK \bT^M_C
\end{equation}
Recall here that $\bA_i$ describes the measured image conic, $\bC_i$ describes the 3D crater conic in the $i$th crater's ENU frame, $\bT^M_C$ is presumed known, and $\bk^T = [0 \; 0 \; 1]$. Thus,  the only unknown in Eq.~\ref{eq:ConicPoseExpansion1} is $\br_M$.

Now, recalling Eq.~\ref{eq:DefHMi} for $\bH_{M_i}$, rearrange the left-hand side into matrix form
\begin{equation}
    \left[ \begin{array}{c c c}
          \bS^T \bT_{E_i}^M  \bB_i \bT^{E_i}_M \bS & \bS^T \bT_{E_i}^M \bB_i \left(\bp_{M_i}-\br_M\right)  \\
          \left(\bp_{M_i}-\br_M\right)^T \bB_i^T \bT^{E_i}_M \bS  \ & \left(\bp_{M_i}-\br_M\right)^T \bB_i \left(\bp_{M_i}-\br_M\right)
   \end{array} \right] = s_i \bC_i
\end{equation}
where $\bS$ is from Eq.~\ref{eq:Smat}. It is immediately evident that the upper-left $2 \times 2$ submatrix is independent of $\br_M$, the lower-right element is quadratic in $\br_M$, and the remaining terms are linear in $\br_M$. Therefore, we first find the unknown scalar $s_i$ using only the upper-left $2 \times 2$ submatrix. Using the Frobenius norm, compute $\hat{s}_i$  as
\begin{equation}
    \hat{s}_i = \arg \min_{s_i} \left\| \bS^T \bT_{E_i}^M  \bB_i \bT^{E_i}_M \bS - s_i \bS^T \bC_i \bS \right\|^2_F
\end{equation}
which has the least squares solution
\begin{equation}
    \hat{s}_i =  \frac{\text{vec}(\bS^T \bC_i \bS)^T \text{vec}(\bS^T \bT_{E_i}^M  \bB_i \bT^{E_i}_M \bS) }{\text{vec}(\bS^T \bC_i \bS)^T  \text{vec}(\bS^T \bC_i \bS)}
\end{equation}
With the scale $s_i$ known, we can now take the upper-right $2\times1$ submatrix to form the linear system for all of the observed conics
\begin{equation}
    \left[ \begin{array}{c c c}
         \bS^T \bT_{E_1}^M \bB_1 \\
         \vdots \\
         \bS^T \bT_{E_n}^M \bB_n
   \end{array}
   \right] \br_M = 
   \left[ \begin{array}{c c c}
         \bS^T \bT_{E_1}^M \bB_i \bp_{M_1} - \hat{s}_i \bS^T \bC_1 \bk \\
         \vdots \\
         \bS^T \bT_{E_n}^M \bB_n \bp_{M_n} - \hat{s}_n \bS^T \bC_n \bk
   \end{array} \right]
\end{equation}
This may be solved for the selenographic camera location $\br_M$ in the least squares sense.

\subsection{Comparing Two Image Conics}
\label{Sec:EllipseComp}
In order to verify a crater match hypothesis, we require a measure of the distance between two conics. For crater pattern verification, this distance of interest is usually between the image ellipse we expect from the projection of a crater's rim ($\mathcal{A}_i$) and the image ellipse we measure of the same crater's rim ($\tilde{\mathcal{A}}_i$). For the moment, however, we will briefly discuss how to compute the distance between two arbitrary ellipses: $\mathcal{A}_i$ and $\mathcal{A}_j$.

Given two ellipses in an image, $\mathcal{A}_i$ and $\mathcal{A}_j$, we seek a scalar distance metric $d(\mathcal{A}_i,\mathcal{A}_j)$ that satisfies the three usual axioms for a distance metric \cite{Cullinane:2011}
\begin{enumerate*}
\item Minimality: $d(\mathcal{A}_i,\mathcal{A}_j)=0 \text{ iff } \mathcal{A}_i = \mathcal{A}_j$. That is, the distance between an ellipse and itself is zero.
\item Symmetry:  $d(\mathcal{A}_i,\mathcal{A}_j) = d(\mathcal{A}_j,\mathcal{A}_i)$. That is, the distance from $\mathcal{A}_i$ to $\mathcal{A}_j$ is  the same as the distance from $\mathcal{A}_j$ to $\mathcal{A}_i$.
\item Triangle Inequality: $d(\mathcal{A}_i,\mathcal{A}_j) \leq d(\mathcal{A}_i,\mathcal{A}_k)  + d(\mathcal{A}_k,\mathcal{A}_j) $
\end{enumerate*}
as well as a fourth axiom unique to this application
\begin{enumerate*}
\setcounter{enumi}{3}
\item Similarity Invariance: $d(\mathcal{A}_i,\mathcal{A}_j) = d(S[\mathcal{A}_i],S[\mathcal{A}_j])$, where $S[\cdot]$ is a similarity transformation. That is, the distance between $\mathcal{A}_i$ to $\mathcal{A}_j$ should not change if the two ellipses undergo a common translation, rotation, or scaling in the image (i.e., undergo a common similarity transformation). See Fig.~\ref{fig:JaccardSimInv}.
\end{enumerate*}
After reviewing the literature, we found numerous approaches for comparing ellipses, but all fail to meet one (or more) of the above four axioms. Thus, after a brief review of existing techniques, a novel method is proposed.

\begin{figure}[b!]
\centering
\includegraphics[width=0.6\columnwidth,trim=0in 0in 0in 0in,clip]{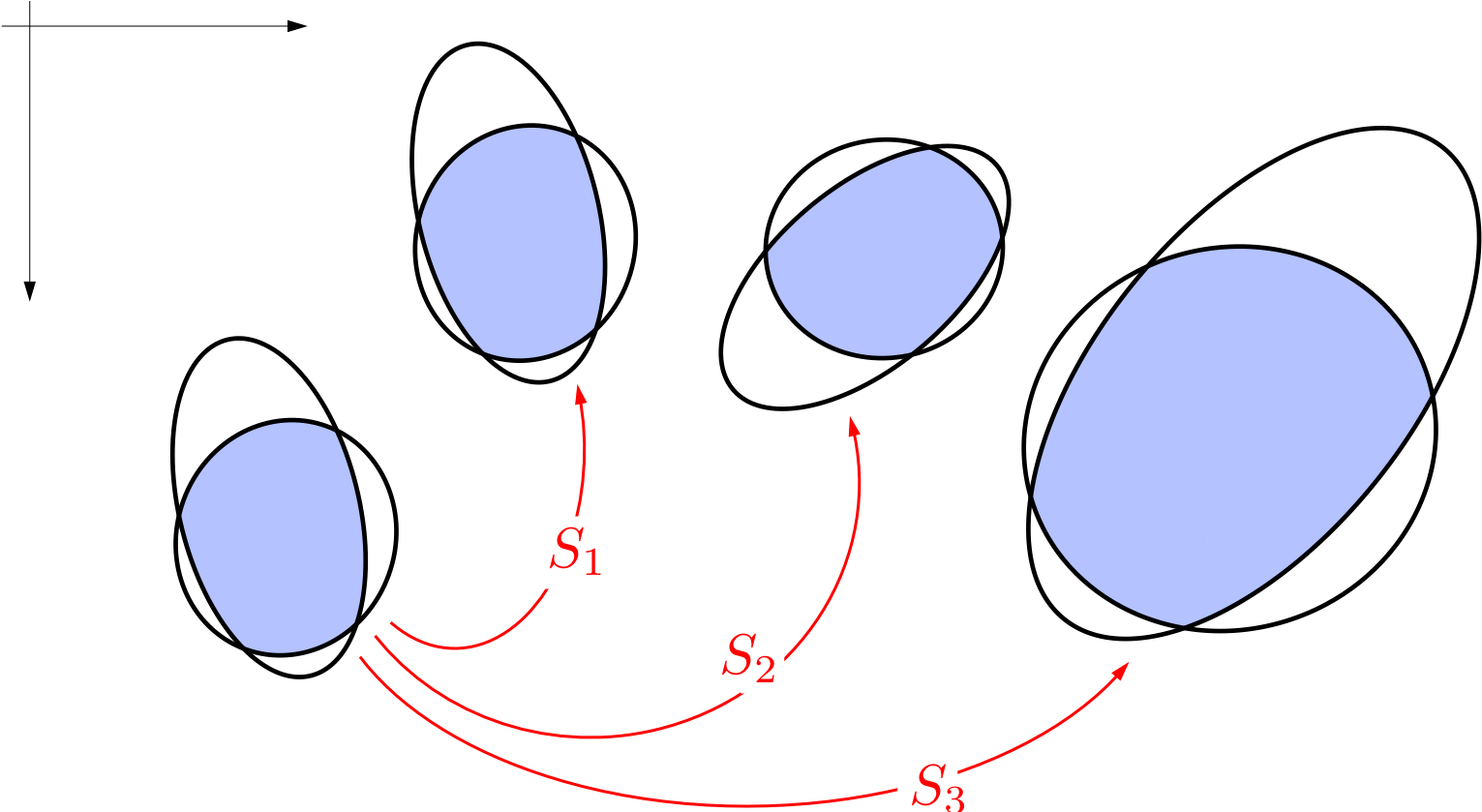}
	\caption{Examples of an ellipse pair undergoing a simiarity transformation. The left-most ellipse pair may  be  translated ($S_1$); translated and rotated ($S_2$); or translated, rotated, and scaled ($S_3$). Under the similarity invariance axiom for an ellipse pair distance metric, all four of these ellipse pairs should produce the same value for distance since they all have the same relative geometry. Observe that all four of these ellipse pairs have the same Jaccard distance and Gaussian  angle. The area of ellipse intersection (i.e., ellipse overlap) is shaded in light blue.}
	\label{fig:JaccardSimInv}
\end{figure}

There exist various \emph{ad hoc} methods based on the explicit ellipse parameters ($a$, $b$, $x_c$, $y_c$, and $\psi$) \cite{Prasad:2010,Cuevas:2014} or by measuring the distance only at one specific point (e.g., via the Hausdorff distance \cite{Swirski:2012,Solarna:2020}), though none of these are similarity invariants. Moreover, these methods have changing geometric meaning as the two ellipses change shape and relative orientation. While not a violation of one of the axioms, it is not a desirable attribute.

Another approach is to apply a common normalization to the implicit representation of both conics (e.g., $\text{det}(\bA_i) = 1$) and then compute the Frobenius norm of their difference \cite{Hosseinzadeh:2018}. There are various normalizations one could choose for the matrix $\bA$, as summarized in \cite{Kanatani:2011}. Nevertheless, such a comparison metric is given by
\begin{equation}
    d_F(\bA_{i},\bA_{j}) = \| \bA_{i} - \bA_{j} \|^2_{F}
\end{equation}
though the specific geometric meaning of this is not clear. This metric also fails to meet the similarity invariance axiom.

We also considered computing the distance as the line integral about one ellipse of the Euclidean distance to the other ellipse. While this admits an elegant and efficient solution with many interesting properties, it fails to meet both the symmetry and similarity invariance axioms. To meet the symmetry requirement, it would be possible to compute this metric in both directions and sum the result, though this is expensive. Regardless, a distance computed in this fashion will still fail to meet the similarity invariance axiom. The method was thus abandoned for the present application.

The deficiencies of the above methods motivate the need for another metric that satisfies all four distance axioms, has a geometrically consistent meaning, and possesses well-understood statistics. There are at least two such distance metrics: one using the Jaccard distance and another by interpreting the ellipse as a multivariate Gaussian.

\subsubsection{Ellipse Pair Distance Metric with the Jaccard Index}
The Jaccard index is the ratio of an intersection to a union. In this case, we take it to be the ratio of the areas of the ellipse intersection to the ellipse union,
\begin{equation}
    J(\mathcal{A}_i,\mathcal{A}_j) = \frac{ \text{Area}( \mathcal{A}_i \cap  \mathcal{A}_j )  }{ \text{Area}( \mathcal{A}_i \cup  \mathcal{A}_j ) }
\end{equation}
The Jaccard distance, which is well-known to satisfy all three of the classical distance metric axioms \cite{Levandowski:1971}, is computed as
\begin{equation}
    d_J(\mathcal{A}_i,\mathcal{A}_j) = 1 - J(\mathcal{A}_i,\mathcal{A}_j) =  1 - \frac{ \text{Area}( \mathcal{A}_i \cap  \mathcal{A}_j )  }{ \text{Area}( \mathcal{A}_i \cup  \mathcal{A}_j ) }
\end{equation}
It is straightforward to see that this metric is also satisfies our similarity invariance axiom.

The only difficulty with the Jaccard distance is the quick computation of the intersection area. Although it is possible to analytically compute the area of ellipse overlap \cite{Hughes:2012}, doing so requires the separate consideration of nine different relative ellipse configurations. Such algorithm  branching is generally undesirable. Fortunately, we recall that the ellipses in question represent regions of substantial extent in a digital image. Thus, the Jaccard distance may be approximated as the ratio of the number of pixels inside both ellipses to the number of pixels inside either ellipse. We may quickly count the pixels inside $\mathcal{A}_i$ by finding those pixels that satisfy the inequality $\bar{\bu}^T \bA_i \bar{\bu} < 0$,  which is a simple and  parallelizable computation. Once we count the pixels in this way, computing the Jaccard distance is trivial.

\subsubsection{Ellipse Pair Distance Metric with the Gaussian Angle}
It is possible to interpret the 2D ellipse $\mathcal{A}_i$ as a bivariate Gaussian probability density function. If that interpretation is used for both  $\mathcal{A}_i$ and $\mathcal{A}_j$,  the distance between two ellipses may be computed using any number of measures for comparing Gaussian PDFs. The authors of \cite{Yang:2016} follow this interpretation and consider the use of the Kullback-Leibler Divergence (KLD) and variants of the Wasserstein distance. The KLD, however, does not satisfy either the symmetry or triangle inequality axioms. The Wasserstein distance meets the first three axioms, but is not invariant under scaling (hence, fails the similarity invariance axiom---though modifications to the classical metrics may alleviate this shortcoming).

We introduce a different distance metric also based on the interpretation of $\mathcal{A}_i$ as a bivariate Gaussian. Therefore, begin by subdividing the conic locus matrix $\bA_i$ describing the image ellipse $\mathcal{A}_i$ as
\begin{equation}
   \bA_i \propto 
   \left[ \begin{array}{c c c}
         \bY_i & -\bY_i  \by_i \\
         -\by_i^T \bY_i^T & (\by_i^T \bY_i \by_i - 1)
   \end{array} \right]
\end{equation}
and it follows that that $\bbu^T_i \bA_i \bbu_i =  0$ and that
\begin{equation}
   (\bu_i - \by_i)^T \bY_i (\bu_i - \by_i) = 1
\end{equation}
With this particular scaling of $\bA_i$, we immediately recognize the $2\times2$ submatrix $\bY_i$ as the shape of the ellipse 
\begin{equation}
   \bY_i =
   \left[ \begin{array}{c c c}
         \cos \psi_i & -\sin \psi_i\\
         \sin \psi_i & \cos \psi_i
   \end{array} \right]
   \left[ \begin{array}{c c c}
         1/a^2_i & 0\\
         0 & 1/b_i^2
   \end{array} \right]
   \left[ \begin{array}{c c c}
         \cos \psi_i & \sin \psi_i\\
         -\sin \psi_i & \cos \psi_i
   \end{array} \right]
\end{equation}
and $\by^T_i=[u_{c_i}, v_{c_i}]$ as the image coordinates of the ellipse center. If we interpret the ellipse as the $1\sigma$ isofrequency contour of a bivariate Gaussian distribution (i.e., the so-called $1\sigma$ error ellipse), the ellipse $\mathcal{A}_i$ corresponds to  the distribution $N(\by_i,\bY^{-1}_i)$ with probability density
\begin{equation}
p_{\mathcal{A}_i}(\bu) = \frac{|\bY_i|}{2 \pi} \exp\left[ -\frac{1}{2}(\bu-\by_i)^T \bY_i (\bu-\by_i) \right]
\end{equation}

Therefore, given two ellipses $\mathcal{A}_i$ and $\mathcal{A}_j$, we may compute the distance between these two ellipses as the angle between $p_{\mathcal{A}_i}(\bu)$ and $p_{\mathcal{A}_j}(\bu)$. Such an angle is computed as the inner product
\begin{equation}
     \cos \theta = \frac{\int p_{\mathcal{A}_i}(\bu) \,  p_{\mathcal{A}_j}(\bu) }{ \sqrt{ \int p^2_{\mathcal{A}_i}(\bu)  \int p^2_{\mathcal{A}_j}(\bu)  }  }
\end{equation}
and we compute a distance metric as
\begin{equation}
    \label{eq:DefDistanceGA}
    d_{GA}(\mathcal{A}_i,\mathcal{A}_j) = \theta = \arccos \left[
    \frac{\int p_{\mathcal{A}_i}(\bu) \, p_{\mathcal{A}_j}(\bu) }{ \sqrt{ \int p^2_{\mathcal{A}_i}(\bu)  \int p^2_{\mathcal{A}_j}(\bu)  }  }
    \right]
\end{equation}
From inspection, this clearly satisfies the minimality and symmetry conditions. Satisfaction of the triangle inequality follows from the Cauchy–Schwarz inequality.

We compute
\begin{equation}
    \label{eq:intpi2}
    \int p^2_{\mathcal{A}_i}(\bu) = \frac{|\bY_i|^2}{4 \pi^2} \int \exp\left[ -\frac{1}{2}(\bu-\by_i)^T \bY_i (\bu-\by_i) \right] = \frac{|\bY_i|^2}{4 \pi^2} \cdot \frac{2 \pi}{|2 \bY_i|} = \frac{|\bY_i|}{8 \pi}
\end{equation}
To find the numerator, we first write
\begin{equation}
    \int p_{\mathcal{A}_i}(\bu) \, p_{\mathcal{A}_j}(\bu) = \frac{|\bY_i||\bY_j|}{4 \pi^2} \int \exp\left[ -\frac{1}{2}\left(  (\bu-\by_i)^T \bY_i (\bu-\by_i) + (\bu-\by_j)^T \bY_j (\bu-\by_j) \right)\right]
\end{equation}
Now, considering the term inside the exponential,
\begin{equation}
     (\bu-\by_i)^T \bY_i (\bu-\by_i) + (\bu-\by_j)^T \bY_j (\bu-\by_j) =  (\bu-\bz)^T ( \bY_i + \bY_j ) (\bu-\bz) + c 
\end{equation}
where
\begin{equation}
     \bz = (\bY_i + \bY_j)^{-1} (\bY_i \by_i + \bY_j \by_j )
\end{equation}
\begin{equation}
     c = (\by_i - \by_j)^T \bY_i (\bY_i + \bY_j)^{-1} \bY_j (\by_i - \by_j)
\end{equation}
Thus, noting that $c$ is a constant, we rewrite the integral as
\begin{align}
    & \int \exp\left[ -\frac{1}{2}\left(  (\bu-\by_i)^T \bY_i (\bu-\by_i) + (\bu-\by_j)^T \bY_j (\bu-\by_j) \right)\right] \\
    & = \int \exp\left[ -\frac{1}{2} (\bu-\bz)^T ( \bY_i + \bY_j ) (\bu-\bz) + c \right] \\
    & = \exp\left[-\frac{c}{2}\right]  \int  \exp\left[ -\frac{1}{2} (\bu-\bz)^T ( \bY_i + \bY_j ) (\bu-\bz) \right] \\
    & = \exp\left[-\frac{c}{2}\right]   \frac{2 \pi}{|\bY_i + \bY_j|}
\end{align}
Consequently,
\begin{align}
    \label{eq:intpipj}
    \int p_{\mathcal{A}_i}(\bu) \, p_{\mathcal{A}_j}(\bu) & = \frac{|\bY_i||\bY_j|}{2 \pi |\bY_i + \bY_j| } \exp\left[-\frac{c}{2}\right]  \\
    & = \frac{|\bY_i||\bY_j|}{2 \pi |\bY_i + \bY_j| } \exp\left[-\frac{1}{2} (\by_i - \by_j)^T \bY_i (\bY_i + \bY_j)^{-1} \bY_j (\by_i - \by_j)  \right] \nonumber
\end{align}

We may therefore compute the distance metric $d_{GA}$ by substitution of Eq.~\ref{eq:intpi2} and \ref{eq:intpipj} into Eq.~\ref{eq:DefDistanceGA},
\begin{equation}
    d_{GA}(\mathcal{A}_i,\mathcal{A}_j) = \arccos \left\{ \frac{4\sqrt{|\bY_i||\bY_j|}}{|\bY_i + \bY_j|} \exp\left[-\frac{1}{2} (\by_i - \by_j)^T \bY_i (\bY_i + \bY_j)^{-1} \bY_j (\by_i - \by_j)  \right]  \right\}
\end{equation}
In addition to satisfying the four required axioms, one of the great advantages of this particular distance metric is that it may be analytically computed from the parameters of the two image ellipses $\mathcal{A}_i$ and $\mathcal{A}_j$.

\subsubsection{A Fast Method for Assessing Ellipse Correspondence}
We propose to use the Gaussian angle distance metric as the criteria for evaluating a potential crater correspondence. In the case where $\mathcal{A}_j$ is a perturbed version of $\mathcal{A}_i$, we find that, to an excellent approximation, 
\begin{equation}
    \frac{d^2_{GA}}{\sigma^2} \sim \chi^2_4 
\end{equation}
where 
\begin{equation}
    \sigma \approx \frac{0.85}{\sqrt{a_i b_i}}\sigma_{img}
\end{equation}
The validity of this approximation is illustrated in Fig.~\ref{fig:EllipseFitChiSquare} for both a nearly circular crater and a very elliptical crater.
\begin{figure}[b!]
\centering
\includegraphics[width=1\columnwidth,trim=0in 0in 0in 0in,clip]{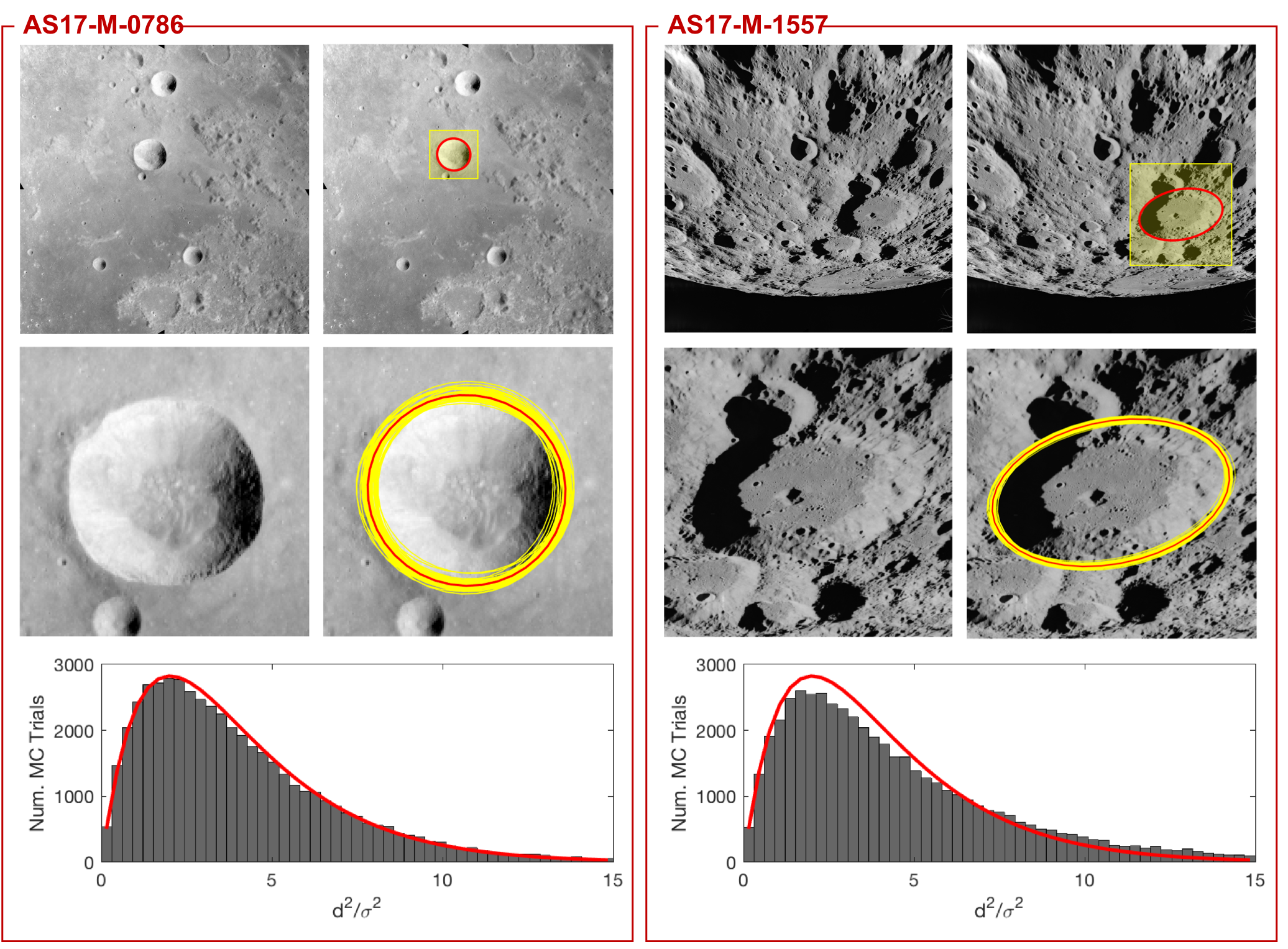}
	\caption{Example Monte Carlo (MC) results (50,000 runs) for crater contour correspondence check using Metric Camera images from Apollo 17. The left frame shows results on image AS17-M-0786 collected on 12 December 1972 (17:31:56 UTC), while the right frame shows results on image AS17-M-1557 collected on 13 December 1972 (16:32:57 UTC). The top pair of images show the original Apollo image (left) and with the overlay of a red elliptical crater rim and yellow region of interest (right). The bottom pair of images is the zoomed-in region of interest by itself (left) and with an overlay of the red circular crater rim and 50 yellow example crater rims from the Monte Carlo (right). The histograms at the bottom show the statistic $d^2/\sigma^2 \sim \chi^2_4$. The red overlay on the histogram is the analytic $\chi^2_4$ probability density function (PDF). (Credit for raw scans of Apollo flight film images: NASA/JSC/ASU. See \cite{Robinson:2008,Lawrence:2008}.)}
	\label{fig:EllipseFitChiSquare}
\end{figure}

Since the distance metric is observed to follow a $\chi^2_4$ distribution, we may construct a statistically-informed criterion for accepting a crater match hypothesis. This is accomplished by evaluating the quantile function for the $\chi^2_4$ distribution. For example, the 99th percentile for the  $\chi^2_4$ distribution occurs at 13.277, and would lead to an acceptance criteria of
\begin{equation}
\label{eq:CraterMatchHypothesis}
\text{Hypothesis}(\mathcal{A}_i = \tilde{\mathcal{A}}_i)  \leftarrow
\left\{ \begin{array}{l l c}
          \text{Accept} & d^2_{{GA}_i} /  \sigma^2 \leq 13.276  \\
          \text{Reject} & \text{otherwise}
   \end{array} \right. 
\end{equation}
The appropriate percentile for accepting or rejecting craters is application dependent and left as a design variable left to the analyst.

\section{Numerical Results}
\label{Sec:NumResults}
\subsection{Sensitivity to Crater Rim Fit and Viewing Geometry}
We performed a Monte Carlo analysis to evaluate the sensitivity of the proposed crater matching methodology to differing levels of measurement noise and under different viewing geometries.

Therefore, as an example, consider a spacecraft placed randomly (with a uniform distribution) around the lunar sphere, thus creating a situation where we randomly image with equal probability any part of the lunar globe. We assume a camera with a full field-of-view (FOV) of about $73.7 \times 73.7$ deg (same  as the Apollo metric camera  \cite{Edmundson:2016,Wu:1980}) and with a $2,200 \times 2,200$ pixel focal plane array (creating a 4.8 MPixel image). 

We now  consider two parametric scenarios: (1) a nadir pointing camera with varying errors in crater rim localization and (2) a fixed level of error in crater rim localization and increasing off-nadir  pointing angle. For the first scenario (nadir pointing images), let the error in the ellipse parameters $a$,  $b$,  $u_c$, $v_c$ all vary by $\sigma_{img}$ and range from 0 to 3 pixel. Subpixel estimation of these parameters is not uncommon since we generally fit an ellipse to hundreds of pixel measurements around the crater circumference. Thus, for the second scenario (off-nadir pointing images), we  assume an error of 0.5 pixel in the ellipse parameters.

\subsubsection{Local Crater Matching Results}
Suppose the randomly placed spacecraft has an altitude of 150 km above the lunar surface. At this altitude, the camera views primarily local crater patterns. In this case, results for the two Monte Carlo scenarios are shown in Table~\ref{tab:MCLocalEllipseFit} and Table~\ref{tab:MCLocalPointing}.

\begin{table}[b!]
	\caption{Monte Carlo results (100 trials) for local crater matching using seven-element descriptor. A nadir-pointing camera is placed randomly around the lunar globe at an altitude of 150 km, with ellipse fit  error varying from 0 to 3 pixel.}
	\centering{}%
	\label{tab:MCLocalEllipseFit}
	\begin{tabular}{ l r r r rr }
	\hline 
	\hline 
	Ellipse Fit & Correct & Incorrect & No Match  & Less Than &  RSS  Camera\\
	Error & Matches & Matches & Found & Three Craters & Position Error \\
	\hline 
	 0.0 pixel & 100  &   0  &   0  &   0 & 2.5$\times10^{-6}$ m \\
     0.5 pixel & 96  &   0  &   1  &   3 & 116 m\\
     1.0 pixel & 96  &   0  &   4  &   0 & 285 m\\
     1.5 pixel & 94  &   0  &   4  &   2 & 428 m \\
     2.0 pixel & 91  &   0  &   7  &   2 & 620 m\\
     2.5 pixel & 93  &   0  &   7  &   0 & 696 m \\
     3.0 pixel & 83  &  0   & 15  &   2 & 923 m\\
	\hline 
	\hline 
\end{tabular}
\end{table}

\begin{table}[b!]
	\caption{Monte Carlo results (100 trials) for local crater matching using seven-element descriptor. A camera is placed randomly around the lunar globe at an altitude of 150 km, with an off-nadir pointing varying from 0 to 30 deg.}
	\centering{}%
	\label{tab:MCLocalPointing}
	\begin{tabular}{ l r r r r r}
	\hline 
	\hline 
	Off-Nadir & Correct & Incorrect & No Match  & Less Than & RSS  Camera \\
	Angle & Matches & Matches & Found & Three Craters & Position Error \\
	\hline 
	 {~0} deg &98  &   0  &   1  &   1 & 140 m\\
     10 deg   &97  &   0  &   3  &   0 & 147 m\\
     20 deg   &99  &   0  &   1  &   0 & 134 m\\
     30 deg   &96  &   0  &   4  &   0 & 178 m\\
	\hline 
	\hline 
\end{tabular}
\end{table}

From Table~\ref{tab:MCLocalEllipseFit} we observe that matching performance is not appreciably affected by ellipse localization error until these errors reach about 2--3 pixels. Even then, the degradation in performance is gradual. The dropping percentage of correct matches occurs because the closest (nearest neighbor) index match to the computed invariants is no longer correct. We could delay the onset of this effect by considering additional neighbors beyond just the single best match, but this would increase run-time. The reader should also be warned that these results do not suggest that crater localization error always needs to be better than 2--3 pixels. What matters is not the absolute rim localization error, but rather the rim localization error as compared to the size of the image craters (i.e., a 2 pixel error on a crater with a 10 pixel diameter is significant, while a 2 pixel error on a crater with a 500 pixel diameter is not significant).

Test cases are counted in the ``No Match Found'' column when we see at least three craters but the algorithm produces no matches. This can happen for a number of reasons. First, the acceptance threshold for a crater match was set at the  99th percentile (as in Eq.~\ref{eq:CraterMatchHypothesis}), which would suggest that $1-0.99^3\approx3\%$ of the triads should fail to match due to measurement noise  (this acceptance threshold can be adjusted based on the analyst's preferences). When there are more than three craters in an image, the likelihood of this happening becomes rather small (but is never zero). The second reason for no match is that no combination of observed craters corresponds to an entry in the index. Given the FOV of this particular camera, it is possible to see more than nine HEALPix surface pixels at one time---thus creating the possibility of observing three craters that do not correspond to a triad within the index.

From Table~\ref{tab:MCLocalPointing} we observe that off-nadir pointing has no meaningful effect (at least up to 30 deg) on our ability to recognize a crater pattern.

\subsubsection{Global Crater Matching Results}
Suppose the randomly placed spacecraft has an altitude of 600 km above the lunar surface. At this altitude, the camera views primarily regional/global crater patterns. In this case, results for the two Monte Carlo scenarios are shown in Table~\ref{tab:MCGlobalEllipseFit} and Table~\ref{tab:MCGlobalPointing}.

\begin{table}[h!]
	\caption{Monte Carlo results (100 trials) for global crater matching  using three-element descriptor. A nadir-pointing camera is placed randomly around the lunar globe at an altitude of 600 km, with ellipse fit  error varying from 0 to 3 pixel.}
	\centering{}%
	\label{tab:MCGlobalEllipseFit}
	\begin{tabular}{ l r r r rr }
	\hline 
	\hline 
	Ellipse Fit & Correct & Incorrect & No Match  & Less Than &  RSS  Camera\\
	Error & Matches & Matches & Found & Three Craters & Position Error \\
	\hline 
	 0.0 pixel & 100  &   0  &   0  &   0 & 1.1$\times10^{-6}$ m \\
     0.5 pixel & 100  &   0  &   0  &   0 & 485 m\\
     1.0 pixel & 100 &   0  &   0  &   0 & 1,294 m\\
     1.5 pixel & 99  &   0  &   1  &   0 & 1,790 m \\
     2.0 pixel & 99  &   0  &   1  &   0 & 2,605 m\\
     2.5 pixel & 96 &   0  &    4  &   0 & 3,549 m\\
     3.0 pixel & 95 &   0  &   5  &   0 & 4,500 m\\
	\hline 
	\hline 
\end{tabular}
\end{table}

\begin{table}[h!]
	\caption{Monte Carlo results (100 trials) for global crater matching using three-element descriptor. A camera is placed randomly around the lunar globe at an altitude of 600 km, with an off-nadir pointing varying from 0 to 30 deg.}
	\centering{}%
	\label{tab:MCGlobalPointing}
	\begin{tabular}{ l r r r r r}
	\hline 
	\hline 
	Off-Nadir & Correct & Incorrect & No Match  & Less Than & RSS  Camera \\
	Angle & Matches & Matches & Found & Three Craters & Position Error \\
	\hline 
	 {~0} deg &100  &   0   &  0   &  0 & 549 m\\
     10 deg   & 99  &   0  &   1&     0 & 539 m\\
     20 deg   &99   &  0    & 1  &   0 & 652 m\\
     30 deg   &  97  &   0   &  2 &    1 & 866 m\\
	\hline 
	\hline 
\end{tabular}
\end{table}

We observe the same basic trend for the global pattern as for the local patterns. Matching performance degrades as we approach localization error of around 2--3 pixels, which occurs at the roughly same level of noise only because the 600 km altitude choice. We once again observe very little affect on matching performance due to off-nadir pointing.

Note that the camera location is computed in exactly the same way for the local and global matching examples (both use the algorithm from Section~\ref{Sec:ConicPose}). The global camera location errors are larger because the images are taken from a higher altitude.  Regardless, errors are still less than 1 km for the local crater patterns and less than 5 km for the global crater patterns---even with rim localization errors much larger than should be expected in practice.

\subsection{Local Matching on Clementine Images}
We demonstrate the seven-element local matching descriptor on real images collected by the Clementine spacecraft's UV/Visible (UV/Vis) camera. During it's 71 days in lunar orbit in 1994 \cite{Nozette:1994}, the Clementine mission collected millions of images of the lunar surface with a variety of instruments. One of these instruments was the UV/Visible camera, a $4.2 \times 5.6$ deg FOV camera with a $288 \times 384 $ pixel CCD focal plane array \cite{Kordas:1995,Hillier:1999}. Images of the lunar surface are available at a variety of ranges and viewpoints and serve as an excellent source of data for testing our proposed crater identification algorithms.

Some example crater identification results are shown in Fig.~\ref{fig:ClementineOverlay}. These results show good matching performance, both for nearly nadir pointing images (top two rows) and for images from oblique viewpoints (bottom row). For each example image, the best fit ellipse parameters (as obtained from the image) were perturbed by a Gaussian distribution with a standard deviation of $\sigma_{img} = 0.5$ pixel. The noise causes a different crater pattern to be matched first for each run, resulting in different crater triads shown in Fig.~\ref{fig:ClementineOverlay}. In each case, the craters with the white outline (and with centers connected by the thin yellow line) were correctly matched to the  Robbins catalog \cite{Robbins:2018} using the local crater index (HEALPix with $k=5$) described in Section~\ref{Sec:BuildCrateIndex}.

\begin{figure}[b!]
\centering
\includegraphics[width=1\columnwidth,trim=0in 0in 0in 0in,clip]{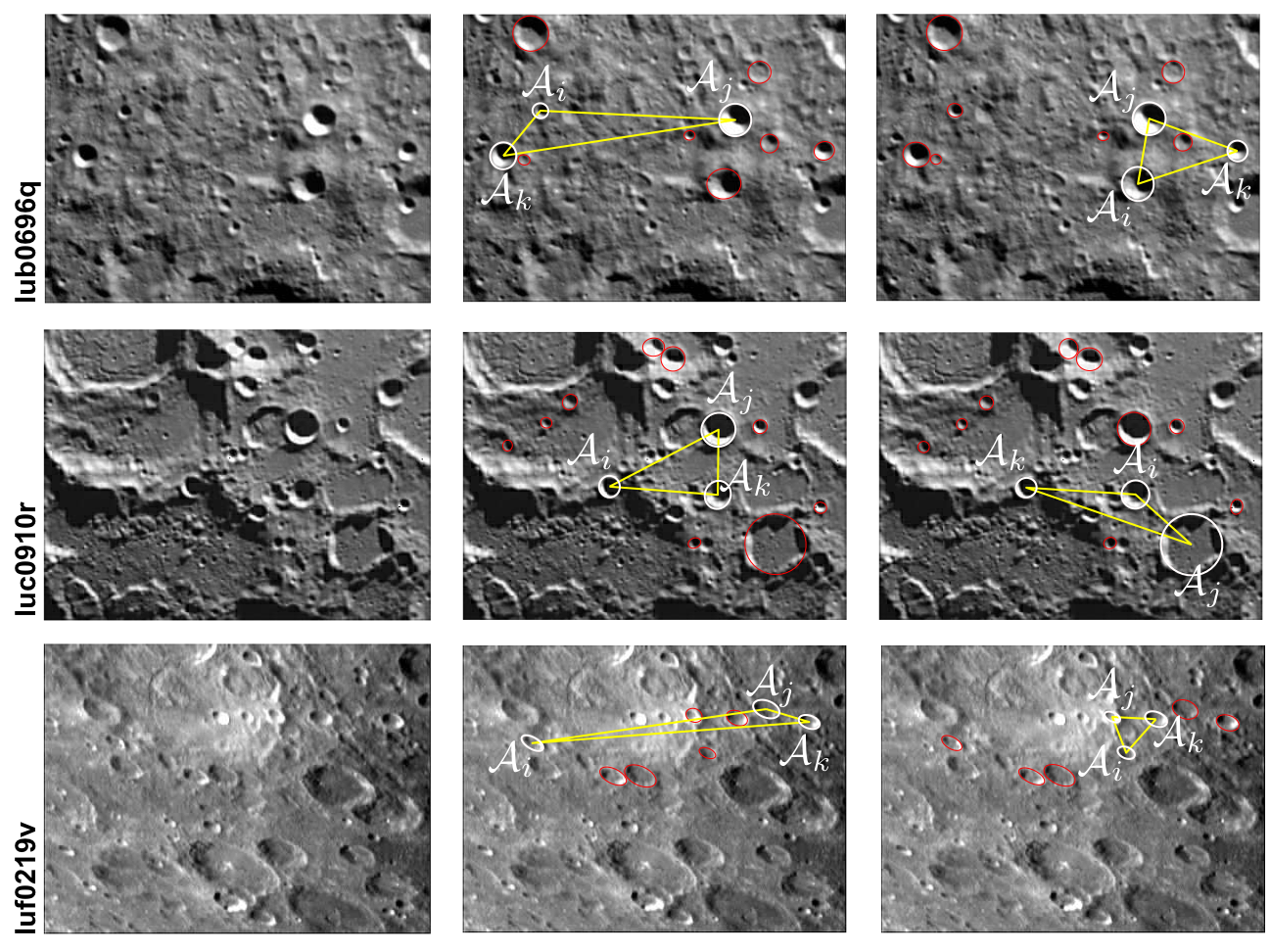}
	\caption{Example crater identification on a sampling of images from the Clementine UV/Vis Camera. Each row shows results for a different Clementine image. Left frame is the raw image from \cite{NRL:1995}. Center and right frame show examples of successful (correct) crater triad matches in white, with centers connected by a yellow triangle to facilitate pattern visualization. Other craters (observed but not contributing to the matched triad) are shown in red. Matches were performed using the seven-element descriptor for local patterns.}
	\label{fig:ClementineOverlay}
\end{figure}

\subsection{Global Matching on Synthetic Images on Reference Lunar Trajectory }
We demonstrate the three-element global matching descriptor on synthetic images along a reference lunar trajectory.

The synthetic images were generated using PANGU, and have a resolution of  $2,048 \times 2,592$ pixels with a 30 deg horizontal FOV. Three example images are shown in Fig.~\ref{fig:InvariantsNRHO}. For each image, the parameters of reference crater ellipse ($a$, $b$,  $u_c$, and $v_c$) were then perturbed by with uncorrelated Gaussian noise of $\sigma_{img} = 1.0$ pixel. This was repeated 1,000 times for each image (with each trial having different noise applied to the crater rim fits) to understand  the sensitivity  of the invariants to measurement noise. Histograms of the three invariants are also shown in Fig.~\ref{fig:InvariantsNRHO}. Since the measurement noise is fixed ($\sigma_{img} = 1.0$ pixel on $a$,  $b$, $u_c$, and $v_c$), we see the histograms widen as the spacecraft gets farther away from the Moon (top to bottom in figure). This occurs because the size of the pattern shrinks relative to the fixed magnitude of ellipse localization error---hence the fixed error causes a greater perturbation of the crater rim geometry.

\begin{figure}[b!]
\centering
\includegraphics[width=1\columnwidth,trim=0in 0in 0in 0in,clip]{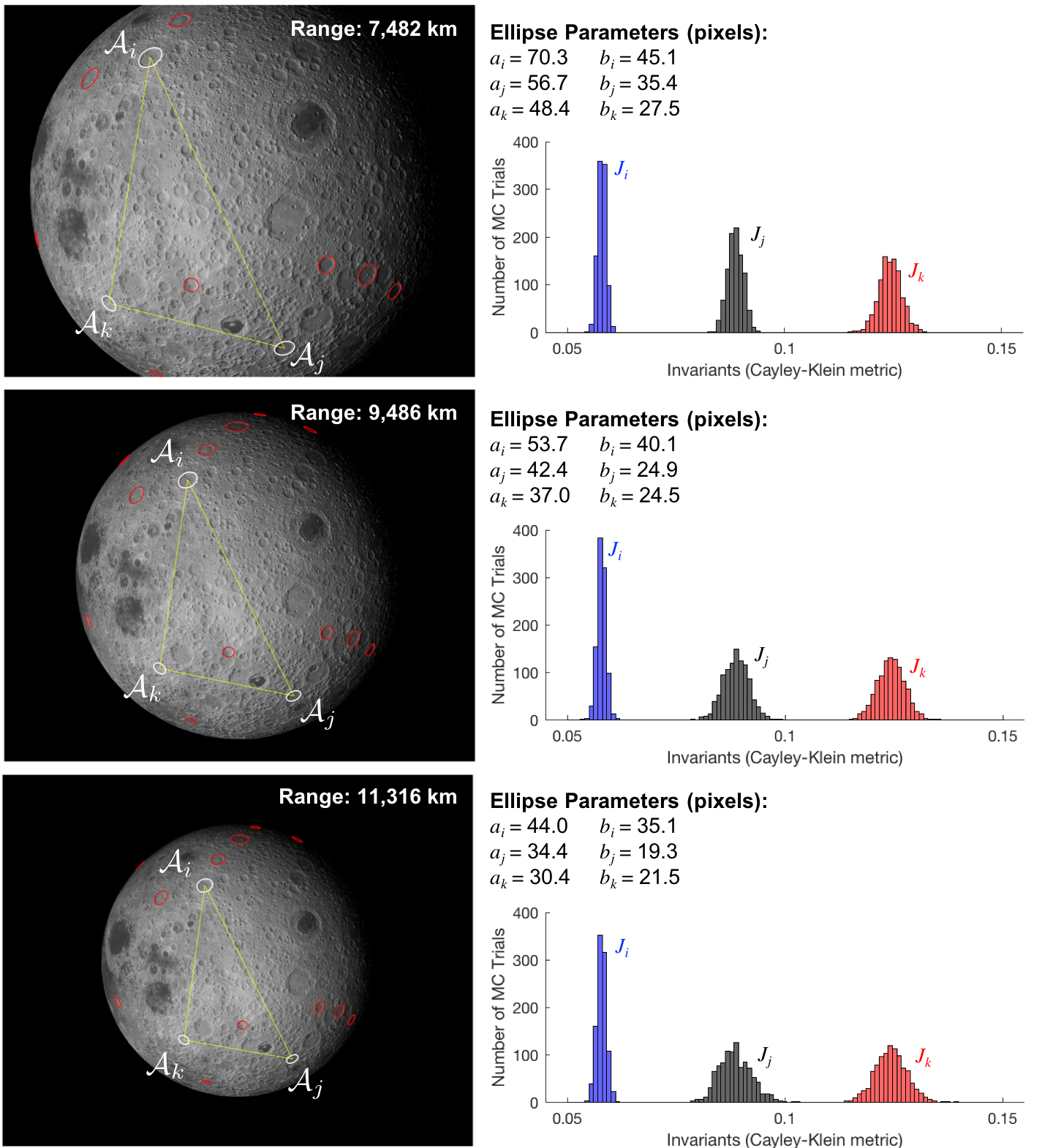}
	\caption{Example synthetic images along a reference trajectory as the spacecraft moves away from the Moon (top to bottom). The craters used to compute invariants are outlined in white and a sampling of other craters (which are not used) are outlined in red. Monte Carlo results (1,000 trials) assume a 1-$\sigma$ ellipse localization error of 1 pixel, leading to widening of invariant histograms as craters become smaller relative to the fixed error. Synthetic images of the Moon were produced using PANGU Planet Surface Simulation Software developed by the Space Technology Centre at the University of Dundee, Scotland \cite{Parkes:2009}.}
	\label{fig:InvariantsNRHO}
\end{figure}

\section{Conclusion}
The identification of craters in a digital image of the lunar surface is a critical capability, both for terrain relative navigation (TRN) and for the registration of scientific images. In this work, we provide the first comprehensive and mathematically rigorous approach for image-based crater identification using concepts from invariant theory. We first show that crater rims tend to be elliptical in shape, suggesting that they are well modeled as a conic (e.g., circle, ellipse) lying on the lunar surface. We then show that there are no projective invariants for conics lying on the surface of an arbitrarily shaped celestial body, but there are invariants for regularly shaped bodies like the Moon. Specifically, for a $d$-tuple of conics, we show that there are $3d-6$ algebraically independent projective invariants for conics lying on a nondegenerate quadric surface and there are $5d-8$ algebraically independent projective invariants for conics lying on a plane.

With the number of algebraically independent projective invariants known, we then develop a practical means of computing these invariants from just the apparent crater rims in a digital image (i.e., from the perspective projection of the 3D conics). These invariants may be used to form a pattern descriptor that is always the same, regardless of camera viewpoint. Thus, these descriptors may be precomputed for known crater patterns and stored in a searchable index. We find that descriptors for a triad of craters provide good matching results. 

When given an image with observed crater rims, a match hypothesis for the crater  rim pattern may be found by a simple nearest neighbor search of the index. With the resulting image-to-map crater correspondence hypothesis, we show how to estimate the location of the camera in the least-squares sense using the three crater rims directly (instead of just the crater center coordinates). Finally, to verify the crater match and pose hypothesis, we reproject the expected crater rim contours into the images and compare these with the measured crater rims. This is accomplished with a new distance metric that compares the entire ellipse fit (instead of just the crater center coordinates). By verifying with a distance metric considering the full crater rim contours, we are able to verify a pattern using fewer craters than methods using only center coordinates.

These techniques are demonstrated in a variety of numerical experiments---both on synthetic and real images. For a camera placed randomly around the Moon with no \emph{a priori} position knowledge, successful matches generally exceed 90\% of the test cases. The remaining cases return ``no match,'' either because there were fewer than three craters in a particular image or because measurement noise corrupted the crater rim fits too much for unambiguous recognition. In no case was an incorrect match returned, a sign that the new distance metric can indeed verify match hypotheses with only three craters (instead of the usual $\geq5$ craters). 

While this work represents a significant advance in methods for crater-based TRN, a great deal of work remains to be done. What follows are a few thoughts regarding the next steps:
\begin{enumerate*}
\item While the repurposing of scientific crater catalogs for TRN is convenient, these catalogs were built with different objectives and are generally suboptimal for TRN applications. Thus, we suggest that the lunar exploration community would benefit greatly from a purpose-built TRN crater catalog. 
\item We provide a complete set of independent projective invariants for crater patterns on the surface of the Moon in Section~\ref{Sec:ComputingInvariantsTop} (three invariants for a triad of conics on a quadric surface, seven invariants for a triad of conics on a plane). While there are no additional algebraically independent invariants, it is possible to construct other algebraically \emph{dependent} invariants as functions of the ones provided in this work (e.g., the $p^2$ invariants in Section~\ref{Sec:PatternDescriptor} are algebraic functions of the  projective invariants from Section~\ref{Sec:ComputingInvariantsTop}). There may exist other forms of these invariants that make lunar crater patterns more distinct or that have superior stability in the presence of measurement noise. Searching for such alternative formulations of the invariants from  Section~\ref{Sec:ComputingInvariantsTop} is an interesting problem.
\item This work develops a global index  of crater descriptors at three different scales using HEALPix. Most real missions likely don't need a global index---but, instead, just an index of craters along a reference trajectory. Although HEALPix may still be used to manage varying scales, there is a great deal of forward work in customizing the proof-of-concept outlined here to a mission-specific application. The best design for an actual in-flight index is not self evident and is the topic of ongoing work.
\item Not addressed in this work are crater detection algorithms (CDAs). Though many different CDAs exist, it is not clear which specific algorithms are best paired with the crater identification framework outlined here. Our crater identification algorithm is critically dependent on the localization of the best-fit ellipse to the crater rim, which is often not the metric used to design or evaluate CDAs. Thus, finding (or developing) an appropriate CDA is an obvious topic of follow-on work.
\end{enumerate*}

In conclusion, we find that crater-based TRN holds great promise for lost-in-space navigation near the Moon. We introduce a robust method for crater identification in this manuscript, though we note that additional work is required to achieve a flight-capable system. Crater-based TRN is a rich field for future study.

\section{Acknowledgements}
The authors thank Peter Sturm for insightful discussions related to the projection of conics that led to the technique presented in Section~\ref{Sec:ConicPose}.  The authors also thank Paul McKee and the NASA Johnson Space Center for assistance with production of the PANGU synthetic images. The following individuals provided valuable discussion that greatly improved the quality of this manuscript: Stuart Robbins, Joseph Mundy, Richard Hartley, Jeffrey Banks, James McCabe, Kristen Schell, Jacob Hinkle, and Christopher D'Souza. This work was partially supported by NASA award 80NSSC17M0027 (J.A.C.), NSF grant IIS 1837985 (H.D.), NSF grant DMS 2001460 (H.D.), and  NASA Lunar Data Analysis Program grant 80NSSC17K0343 (R.W.).


\bibliography{mybibfile}
\bibliographystyle{spmpsci}

\end{document}